\documentclass[a4paper,12pt]{article}
\usepackage{latexsym,amssymb,enumerate,amsmath,epsfig,amsthm,authblk,dsfont,bm}
\usepackage[margin=1in]{geometry}
\usepackage[dvipsnames]{xcolor}
\usepackage{setspace,color}
\usepackage{tikz}
\usepackage{floatrow}
\usepackage{graphicx,subfig}
\usepackage[ruled]{algorithm2e}
\usepackage{algorithmic}
\usepackage{epstopdf}
\usepackage{grffile}
\usepackage{hyperref}
\hypersetup{
	colorlinks=true,
	linkcolor=blue,
	citecolor=blue,
	filecolor=magenta,
	urlcolor=cyan,
}

\newcommand{\bp}{\mathbf{p}}
\newcommand{\bff}{\mathbf{f}}
\newcommand{\RR}{\mathds{R}}

\newcommand{\bu}{\mathbf{u}}
\newcommand{\bv}{\mathbf{v}}
\newcommand{\bx}{\mathbf{x}}
\newcommand{\bq}{\mathbf{q}}
\newcommand{\bn}{\mathbf{n}}
\newcommand{\bD}{\mathbf{D}}
\newcommand{\bmu}{\bm{\mu}}
\newcommand{\bnu}{\bm{\nu}}
\newcommand{\blambda}{\bm{\lambda}}
\newcommand{\bbf}{\mathbf{f}}
\newcommand{\bG}{\mathbf{G}}
\newcommand{\bM}{\mathbf{M}}
\newcommand{\TV}{\mathrm{TV}}
\newcommand{\cF}{\mathcal{F}}
\newcommand{\cS}{\mathcal{S}}
\newcommand{\cI}{\mathcal{I}}
\newcommand{\cH}{\mathcal{H}}
\newcommand{\cL}{\mathcal{L}}

\newcommand{\cA}{\mathcal{A}}
\newcommand{\cE}{\mathcal{E}}
\newcommand{\cR}{\mathcal{R}}

\newcommand{\cof}{\mathrm{cof}}

\newcommand{\diver}{\mathrm{div}}
\newcommand{\Real}{\mathrm{Real}}

\DeclareMathOperator*{\argmin}{arg\,min}
\newtheorem{thm}{Theorem}[section]
\newtheorem{remark}{Remark}[section]


\begin{document}
	\title{Elastica Models for Color Image Regularization}
	\date{	\normalsize{\emph{ In memory of Roland Glowinski--a dear friend, mentor, colleague and great leader. }}}
	\author{Hao Liu \thanks{Department of Mathematics, Hong Kong Baptist University, Kowloon Tong, Hong Kong. Email: haoliu@hkbu.edu.hk.}
		, Xue-Cheng Tai \thanks{Hong Kong Center for Cerebro-Cardiovascular Health Engineering (COCHE), Hong Kong Science Park, Hong Kong. Email: xtai@hkcoche.org, xuechengtai@gmail.com.}
		, Ron Kimmel \thanks{Computer Science Department, and Electrical and Computer Engineering Department, Technion, Haifa, Israel. Email: ron@cs.technion.ac.il.}
		, Roland Glowinski \thanks{The author is deceased.  Former address: Department of Mathematics, University of Houston, Honston, TX 77204, USA, and Department of Mathematics, Hong Kong Baptist University, Kowloon Tong, Hong Kong. }
	}
	\maketitle

	\begin{abstract} 
			The choice of a proper regularization measure plays an important role in the field of image processing. 
		One classical approach treats color images as two dimensional surfaces embedded in a five dimensional spatial-chromatic space. 
		In this case, a natural regularization term arises as the image surface area.
		Choosing the chromatic coordinates as dominating over  the spatial ones, the image spatial coordinates could be thought of as a paramterization of the image surface manifold in a three dimensional color space. 
		Minimizing the area of the image manifold leads to the Beltrami flow or mean curvature flow of the image surface in the 3D color space, while minimizing the elastica of the image surface yields an additional interesting regularization. 
		Recently, the authors proposed a color elastica model, which minimizes both the surface area and elastica of the image manifold. 
		In this paper, we propose to modify the color elastica and introduce two new  models for color image regularization. 
		The revised measures are motivated by the relations between the color elastica model, Euler's elastica model and the total variation model for gray level images. 
		Compared to our previous color elastica model, the new models are direct extensions of Euler's elastica model to color images. 
		The proposed models are nonlinear and challenging to minimize. 
		To overcome this difficulty, two operator-splitting methods are suggested. 
		Specifically, nonlinearities are decoupled by introducing new vector- and matrix-valued variables. 
		Then, the minimization problems are converted to solving initial value problems which are time-discretized by operator splitting. 
		Each subproblem, after splitting, either has a closed-form solution or can be solved efficiently. 
		The effectiveness and advantages of the proposed models are demonstrated by comprehensive experiments. 
		The benefits of incorporating the elastica of the image surface as regularization terms compared to common alternatives are empirically validated.
	\end{abstract}

\section{Introduction}
Image regularization is a fundamental topic in image processing, and appears in almost every tasks in this field. 
In past decades, tremendous efforts have been devoted to looking for good image regularizers, while most of which focus on gray-scale images. As one important way to describe the magnificence of the world is through colors, good regularizers for color images is of high demand. 

In the literature, image regularization for gray-scale images has been extensively studied. 
Given a gray-scale image $v$, a well-known regularizer is the total variation (TV) \cite{rudin1992nonlinear}, given by
$
\int_{\Omega} |\nabla v| d\bx,
$
where $\Omega\subset \RR^2$ is a bounded domain and $d\bx=dx_1dx_2$ with $x_1,x_2$ being the coordinates of a point $\bx$ in $\Omega$. 
TV is a first order regularizer since it only depends on the first order partial derivatives of $v$. 
It is known in preserving sharp changes in the gradient of the image. Fast algorithms for TV based image regularization models can be found in \cite{chambolle2011first,chambolle2004algorithm,goldstein2009split}. 
One drawback of the TV model is that it suffers from the staircase effects. 
To overcome this obstacle, high order regularizers were explored. 
One of the most popular regularizers is realized by minimizing Euler's elastica energy
\begin{align}
	\int_{\Omega} \left(a+b\left(\nabla\cdot \frac{\nabla v}{|\nabla v|}\right)^2\right) |\nabla v| d\bx ,
	\label{eq.EL}
\end{align}
where $a,b\geq 0$ are weight parameters. 
For gray-level images, Euler's elastica treats the image as a function and minimizes the length and curvature of each of its level curves. 
Another perspective of (\ref{eq.EL}) is that it penalizes both the TV and the variation of the TV. This observation provides a new perspective on defining the `Sobolev space' for images. From this point of view, TV is a first order model and Euler's elastica is a second order model, since the former and the latter one penalize the first and second variation of gray images, respectively.
Due to the superior performance of (\ref{eq.EL}) in various image processing models, designing efficient algorithms for Euler's elastica based models has been a popular topic. 
Efforts in that direction include augmented Lagrangian multipliers based methods \cite{duan2013fast,tai2011fast,zhang2017fast,liu2018proximal,yashtini2015alternating,duan2014two}, and split Bregmann method \cite{yashtini2016fast}. 
Recently, an operator splitting method was proposed in \cite{deng2019new}. 
Unlike previous methods, this method is insensitive to the choice of parameters and is almost parameters free. 
We refer the readers to \cite{kang2019survey} for a survey on fast algorithms for Euler's elastica based models in image inpainting.

Color images can be thought of as vector-valued signals with $m$ {\it chromatic} channels. 
One simple way to process a color image is to apply gray-scale image regularizers channel by channel. 
However, this way, the interactions between channels are ignored. 
In literature, regularizer models and fast algorithms that treat to color images are limited. 
As a generalization of the scalar TV, \cite{blomgren1998color} proposed the color TV which is the square root of the sum of squared TV of each channel. 
The authors of \cite{tan2018color} proposed a total curvature model in which the color TV in \cite{blomgren1998color} was replaced by the sum of squared level set curvature of each channel. 
In the geometric point of view, inspired by the discussion of tensor gradient for vector-valued images in \cite{di1986note} in which a color image is considered as a two dimensional manifold in $\RR^m$, \cite{sapiro1996anisotropic} proposed an anisotropic diffusion framework and \cite{weickert1999coherence} suggested an edge-enhancing diffusion method. 
Based on the framework of \cite{di1986note}, another generalization of the scalar TV, known as the vectorial TV (VTV), was proposed by \cite{goldluecke2012natural}. 
Relating to the framework proposed in \cite{sapiro1996vector}, VTV penalizes the largest singular value of the Jacobian of the color image (a $m\times 2$ matrix) on each pixel in its domain.
Efficient algorithms for this family of regularizers are studied in \cite{bresson2008fast,duval2009projected}. Algorithms dedicated to color image enchancement are studied in \cite{batard2018geometric,batard2020variational,nikolova2014fast,naik2003hue,pierre2017variational,trahanias1992color}.

Another geometry based regularizer for color images is the Beltrami framework, which was proposed in  \cite{kimmel1998natural,sochen1998general} and further investigated in \cite{kimmel2000images,spira2007short,wang2013fidelity,roussos2010tensor,wetzler2011efficient}. In this framework, a color image is considered as a two-dimensional manifold embedded in the $m+2$ dimensional space-feature space. 
The Beltrami framework minimizes the Polyakov action \cite{polyakov1981quantum}, which is a functional that measures the surface area of the surface. 
Its first variation gradient flow leads to a Beltrami flow. 
It was shown that at its limit, the Polyakov action reduces to TV model for gray-scale images. 
Fast algorithms for the Beltrami framework have been developed in \cite{bar2007deblurring, rosman2009efficient,rosman2011semi, rosman2010polyakov,zosso2014primal}.

Most of the aforementioned color image regularizers are first order, which may not be rich enough to capture image properties. 
Recently, based on the Beltrami framework, the authors have proposed a second order regularizer, the color elastica model \cite{liu2021color}. 
The color elastica model is an extension of the Beltrami framework and penalizes both the Polyakov action and the Beltrami flow, the second term being the norm of the Laplace-Beltrami operator acting on the image coordinates. 
In the limit, it would be nothing but the square of mean curvature of the color image surface embedded in the $\mathbb{R}^3$ chromatic space.
In our setting, the color elastica model involves a parameter $\alpha$, which controls the weight between spatial coordinates and feature (chromatic or color) coordinates. 
The color elastica model generalizes (\ref{eq.EL}) to color images, as for gray-scale images it reduces to a weighted Euler's elastica model by letting $\alpha$ go to zero. 
However, the weight in the reduced model is $1/|\nabla v|$, which may be challenging to handle when $\nabla v$ vanishes. 
The second drawback is that the algorithm proposed in \cite{liu2021color} converges very slowly when $\alpha$ is small.

In this article, we propose two modified color elastica models, which are more natural extensions of (\ref{eq.EL}) for color images. In the first model, we add a weight to the Laplace-Beltrami term so that the model exactly reduces to (\ref{eq.EL}) as $\alpha\rightarrow0$. 
The second model is based on the first one and the relation between TV and surface area of gray-scale images. 
The second model boils down directly to (\ref{eq.EL}) for gray-scale images. 
The proposed models contain nonlinear functionals that are difficult to minimize. 
Two operator-splitting methods are designed which solve the proposed models efficiently. 
Operator-splitting methods are known for decomposing complicated problems into several easy-to-solve sub-problems and have been applied in numerical PDEs \cite{glowinski2019finite,buttazzo2020numerical,liu2019finite}, inverse problems \cite{glowinski2015penalization}, obstacle problem \cite{liu2022fast}, fluid-structure interactions \cite{bukavc2013fluid} and recently in image processing \cite{deng2019new,liu2021color,he2020curvature,duan2022fast}. 
In our proposed algorithms, we decouple the nonlinearity by introducing new vector- and matrix-valued variables. 
Then, minimizing the functionals is converted to solving initial-value problems until a steady state is reached. 
The initial-value problems are time-discretized by operator splitting methods such that each sub-problem either has a closed form solution or can be solved efficiently.

This article is structured as follows: We provide motivation and formulations of the proposed models in Section \ref{sec.formulation}. 
In Section \ref{sec.splitting}, we present our operator-splitting schemes and discuss the solution to each subproblem. 
The proposed operator-splitting methods are space discretized in Section \ref{sec.dis}. 
We empirically justify the proposed models in Section \ref{sec.justification}.
We demonstrate the efficiency and performance of the proposed algorithms and models in Section \ref{sec.experiments} by comprehensive numerical experiments. 
This article is concluded in Section \ref{sec.conclusion}.

\section{Problem formulation}
\label{sec.formulation}
\subsection{Motivation towards the proposed models}
In image processing, the image surface area \cite{sochen1998general,yezzi1998modified,liu2019surface} and total variation \cite{rudin1992nonlinear} are two popular regularizers for image restoration. 
We motivate our construction by reviewing the links between these models.
Let $\Omega$ be a rectangular domain with coordinates $x_1,x_2$. 
Any gray-scale image $f$ can be considered as a two-dimensional surface embedded in the three-dimensional space-feature space, $F(x_1,x_2)=(\sqrt{\alpha}x_1,\sqrt{\alpha}x_2,f(x_1,x_2))$, where $\alpha$ is a parameter controlling the weight of spatial coordinates. 
Under such a parameterization, the metric on the image manifold is $g=\det(\bG)$, where
\begin{align}
	\bG\,=\,\begin{pmatrix}
		\alpha +(\partial_1f)^2 & \partial_1f\partial_2f\cr
		\partial_1f\partial_2f & \alpha+ (\partial_2f)^2
	\end{pmatrix}
	\label{eq.G}
\end{align}
with  $\partial_1 f=\partial f/\partial{x_1},\,  \partial_2 f=\partial f/\partial{x_2}$.
From this metric, the surface area of $f$ can be computed as
\begin{align}
	S(f)\,=\,\int_{\Omega} \sqrt{g}d\bx=\int_{\Omega} \sqrt{\alpha^2+\alpha[(\partial_1f)^2+(\partial_2f)^2]}d\bx.
	\label{eq.Polykav.gray}
\end{align}
At the other end, the total variation of $f$ is given by
\begin{align}
	\TV(f)=	\int_{\Omega} \sqrt{(\partial_1f)^2+(\partial_2f)^2}d\bx.
	\label{eq.TV}
\end{align}
Comparing the right-hand side of  (\ref{eq.Polykav.gray}) and (\ref{eq.TV}), we observe that $\TV(f)$ can be recovered by replacing $g$ in $S(f)$ by $g-\alpha^2$,
\begin{align}
	\int_{\Omega} \sqrt{g-\alpha^2}d\bx=\sqrt{\alpha} \TV(f).
	\label{eq.TVRelation}
\end{align}
This observation would guide us in our exploration of a modified color elastica model. 

Recall that the color ealstica model proposed in \cite{liu2021color} takes an RGB image as a two dimensional surface embedded in the five dimensional space-feature space. 
The model is given as
\begin{equation}
	\min_{\bv\in (\cH^2(\Omega))^3}\int_{\Omega}\left[ 1+\beta\sum_{k=1}^3 |\Delta_g v_k|^2\right] \sqrt{g} d\bx +\frac{1}{2\eta} \sum_{k=1}^3 \int_{\Omega} |v_k-f_k|^2d\bx,
	\label{eq.model.old}
\end{equation}
where $\beta>0, \eta>0$ are weight parameters, and $\cH^2$ is the Sobolev space defined as
$$
\cH^2(\Omega)=\left\{ v| v\in \cL^2(\Omega), \nabla v\in (\cL^2(\Omega))^{ 2}, \bD v\in  (\cL^2(\Omega))^{2\times 2} \right\},
$$
where $\bD$ denotes the Hessian and the derivatives being in the weak sense.
In (\ref{eq.model.old}), $g$ is the determinant of the manifold metric defined by, $g=\det \bG$, where $\bG=(g_{ij})_{1\leq i,j \leq 2}$ and
$$
g_{11}\,=\,\alpha+\sum_{k=1}^3 \left| \frac{\partial v_k}{\partial x_1}\right|^2,\ \,\, g_{12}\,=\,g_{21}\,=\,\sum_{k=1}^3 \frac{\partial v_k}{\partial x_1} \frac{\partial v_k}{\partial x_2},\ \,\,g_{22}\,=\,\alpha+\sum_{k=1}^3 \left| \frac{\partial v_k}{\partial x_2}\right|^2,
$$
with $\alpha>0$. 
Whenever $\alpha>0$, $\bG$ is positive definite. 
In (\ref{eq.model.old}), $\Delta_g$ is the Laplace-Beltrami operator associated with matrix $\bG$,
\begin{equation}
	\Delta_g \phi=\frac{1}{\sqrt{g}} \nabla \cdot (\sqrt{g}\bG^{-1}\nabla\phi),\,\,\,\, \forall \phi\in V.
\end{equation}
The second term in (\ref{eq.model.old}) is called the color elastica term, which captures the elastica of images in the chromatic space. The color elastica term is the variation of the Polyakov action \cite{kimmel1998natural,sochen1998general} given by
$$
\int_{\Omega}\sqrt{g}d\bx.
$$
From the perspective of `Sobolev space' of images as discussed in the introduction, the Polyakov action is a first order model and the color elastica model (\ref{eq.model.old}) is a second order model for color images.

We have shown in \cite[Remark 3.2]{liu2021color} that for the single-channel case, when $\alpha\rightarrow 0$, the color elastica model (\ref{eq.model.old}) reduces to a variant of Euler's elastica model: the term $\left(\nabla\cdot \frac{\nabla u}{|\nabla u|}\right)^2$ is weighted by $1/|\nabla v|^2$, see Appendix \ref{sec.colorElastica.Euler} for details.
Since Euler's elastica model have demonstrated impressive performance in processing single-channel images, we would like to derive multi-channel image models that are more direct extensions of Euler's elastica model. 
As mentioned above, there are two drawbacks of (\ref{eq.model.old}) when connecting it with Euler's elastica model:
(i) we need $\alpha\rightarrow 0$, and (ii) the resulting model differs from Euler's elastica model by a factor. 
Next, we introduce our new models by making two modifications that provide a remedy to each drawback.

The first modification targets drawback (ii). 
We weight the second term, the color elastica term, in (\ref{eq.model.old}) by $g$ to get,
\begin{equation}
	\int_{\Omega}\left[ 1+\beta\sum_{k=1}^3 g|\Delta_g v_k|^2\right] \sqrt{g} d\bx +\frac{1}{2\eta} \sum_{k=1}^3 \int_{\Omega} |v_k-f_k|^2d\bx.
	\label{eq.model1}
\end{equation}
This modification modulates the color elastica term by the metric $g$.
Due to this additional factor, the new regularizer maybe not coordinate invariant as $g$ depends on the parametrization of the manifold. Compared to (\ref{eq.model.old}), the gradient flow of (\ref{eq.model1}) has a preconditioning that amplifies the action along {\em edges} in the image.
With this modification, model (\ref{eq.model1}) reduces to Euler's elastica model as $\alpha\rightarrow 0$ for one-channel images.

The second modification targets drawback (i) and utilizes relation (\ref{eq.TVRelation}).
First note that an alternative expression of $\Delta_g \phi$ is
\begin{equation}
	\Delta_g \phi=\frac{1}{\sqrt{g}} \nabla \cdot \left(\frac{1}{\sqrt{g}}\cof(\bG)\nabla\phi\right), \forall \phi\in V,
	\label{eq.Beltrami.2}
\end{equation}
where $\cof(\bG)$ is the cofactor matrix of $\bG$.
Taking advantage of (\ref{eq.TVRelation}), we replace $g$ in (\ref{eq.model1}) and (\ref{eq.Beltrami.2}) by $g-\alpha^2$ to get the modified model:
\begin{equation}
	\int_{\Omega}\left[ 1+\beta\sum_{k=1}^3 (g-\alpha^2)|\widetilde{\Delta}_g v_k|^2\right] \sqrt{g-\alpha^2} d\bx +\frac{1}{2\eta} \sum_{k=1}^3 \int_{\Omega} |v_k-f_k|^2d\bx,
	\label{eq.model2}
\end{equation}
where
\begin{equation}
	\widetilde{\Delta}_g \phi=\frac{1}{\sqrt{g-\alpha^2}} \nabla \cdot \left(\frac{1}{\sqrt{g-\alpha^2}}\cof(\bG)\nabla\phi\right)\quad  \forall \phi\in V.
	\label{eq.Beltrami2.2}
\end{equation}
In the single-channel case, the functional (\ref{eq.model2}) reduces to
\begin{equation}
	\int_{\Omega}\sqrt{\alpha}\left[ 1+\beta\left|\nabla \cdot \frac{\nabla v}{|\nabla v|}\right|\right] |\nabla v| d\bx +\frac{1}{2\eta}  \int_{\Omega} |v-f|^2d\bx
\end{equation}
which exactly is the Euler's elastica model without any condition.

In Figure \ref{fig.RelativeEnergyPen}, we use a simple example to demonstrate the improvement of the regualrizers in model (\ref{eq.model1}) and (\ref{eq.model2}) over that in the original color elastica model (\ref{eq.model.old}). We use $\cF_0,\cF_1$ and $\cF_2$ to denote the regularizers in model (\ref{eq.model.old}), (\ref{eq.model1}) and (\ref{eq.model2}), respectively. For any noisy image $\bbf$ and its clean version $\bbf_0$, we define the relative energy of regularizer $\cF$ by $\cR\cF(\bbf)=\cF(\bbf)/\cF(\bbf_0)$. As the noise level of $\bbf$ increases, the regularizer is more effective if its relative energy inscreases faster. In Figure \ref{fig.RelativeEnergyPen}, we take the clean image 'Pens' in Figure \ref{fig.benchmark} as an example and compare the plots of relative energies of all three regualrizers versus noise level. The new regularizers in model (\ref{eq.model1}) and (\ref{eq.model2}) have faster increasing rates than that in model (\ref{eq.model.old}), demonstrating that they are more effective and can better characterize natural images than the regularizer in model (\ref{eq.model.old}). We refer readers to Section \ref{sec.justification} for more details and comprehensive comparisons with other models.

In the rest of this article, we focus on efficient algorithms to minimize (\ref{eq.model1}) and (\ref{eq.model2}).



\begin{figure}[t!]
	\centering
	\includegraphics[width=0.4\textwidth]{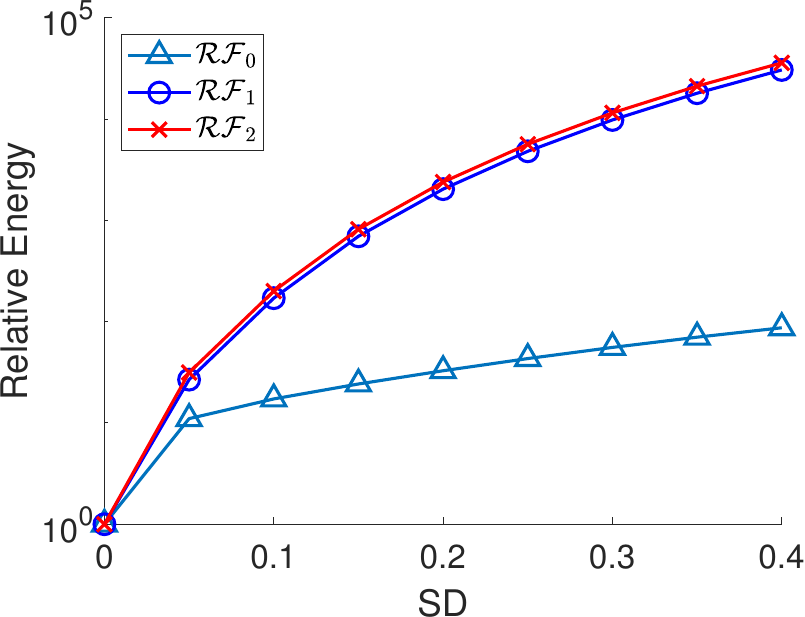}
	\caption{Comparison of the relative energy of the regularizers in the modified color elastica models (\ref{eq.model1}) and (\ref{eq.model2}), and the original color elastica model (\ref{eq.model.old}). The noisy images are generated by adding Gaussian noise with SD varying from 0.01 to 0.1. This test takes the image 'Pens' in Figure \ref{fig.justify} as an example. The considered regularizers are $\cF_0$ (in model (\ref{eq.model.old})), $\cF_1$ (in model (\ref{eq.model1})) and $\cF_2$ (in model (\ref{eq.model2})). Each relative energy is averaged over 10 experiments.}
	\label{fig.RelativeEnergyPen}
\end{figure}


\begin{remark}
	Given a color image $\bv$, one can show that the corresponding $g$ has the following expression
	
	\begin{align}
		g=&\alpha^2+\alpha\left( |\nabla v_1|^2+ |\nabla v_2|^2 + |\nabla v_3|^2\right) \nonumber\\
		&+ (\det(\nabla v_1; \nabla v_2))^2+(\det(\nabla v_1; \nabla v_3))^2+(\det(\nabla v_2; \nabla v_3))^2.
		\label{eq.g}
	\end{align}
	In (\ref{eq.g}), the first two terms have $\alpha$ as a factor and the first term is a lower bound of $g$. The second term is the sum of squared total variation of each channel, which is called color TV in \cite{blomgren1998color}. The third term is independent of $\alpha$ and describes the interactions between channels. Therefore, larger $\alpha$ gives more weights to the color TV term and less weight to the inter-channel interaction term. Compared to (\ref{eq.model1}), the measure used in (\ref{eq.model2}) shifts $g$ so that its lower bound is zero.
\end{remark}

\subsection{Reformulation of (\ref{eq.model1}) and (\ref{eq.model2})}
The modified models (\ref{eq.model1}) and (\ref{eq.model2}) are nonlinear functionals and are difficult to minimize. To develop efficient algorithms to find their minimizers, we decouple the nonlinearities in the Laplace-Beltrami operator by introducing several vector-valued and matrix-valued variables.
For $k=1,2,3$ and $r=1,2$, let us denote by $q_{kr}$ the real valued function $\frac{\partial v_k}{\partial x_r}$ and by $\bq$ the $3\times 2$ matrix
\begin{align*}
	\bq=\begin{pmatrix}
		q_{11} & q_{12}\\
		q_{21} & q_{22}\\
		q_{31} & q_{32}
	\end{pmatrix} =\nabla v,
	\mbox{ with } v=
	\begin{pmatrix}
		v_1\\ v_2\\ v_3
	\end{pmatrix}.
\end{align*}
Denote $\bq_k=\begin{pmatrix}
	q_{k1} & q_{k2}
\end{pmatrix}, k=1,2,3$, we introduce the $3\times 2$ matrix $\bmu=\sqrt{g}\bq\bG^{-1}$ with
\begin{equation}
	\bmu_k=\sqrt{g}\bq_k\bG^{-1}
	\label{eq.mu}
\end{equation}
and
$\bnu=\frac{1}{\sqrt{g-\alpha^2}}\bq\cof(\bG)$
with
\begin{equation}
	\bnu_k=\frac{1}{\sqrt{g-\alpha^2}}\bq_k\cof(\bG).
	\label{eq.tmu}
\end{equation}
Here $\bmu$ and $\bnu$ are proxies of the parts inside the divergence in (\ref{eq.Beltrami.2}) and (\ref{eq.Beltrami2.2}), respectively.
Expressions (\ref{eq.mu}) and (\ref{eq.tmu}) imply 
$$
\bq_k=\frac{1}{\sqrt{g}}\bmu_k\bG \quad \mbox{and} \quad \sqrt{g-\alpha^2}\bnu_k=\bq_k\cof(\bG).
$$
We denote by $\bM(\bq)$ the matrix-valued function defined by
\begin{align}
	\begin{pmatrix}
		\alpha+q_{11}^2+q_{21}^2+q_{31}^2 & q_{11}q_{12}+q_{21}q_{22}+q_{31}q_{32}\\
		q_{11}q_{12}+q_{21}q_{22}+q_{31}q_{32} & \alpha+q_{12}^2+q_{22}^2+q_{32}^2
	\end{pmatrix}
\end{align}
and denote $\det \bM(\bq)$ by $m(\bq)$.  
Define the sets $\Sigma_f,S_{\bG},\widetilde{S}_{\bG}$ as
\begin{align*}
	&\Sigma_f=\left\{\bq\in (\cL^2(\Omega))^{3\times 2}, \exists \bv\in \cH^1(\Omega) \mbox{ such that } \bq=\nabla v \mbox{ and } \int_{\Omega} v_k-f_k dx=0 \mbox{ for }k=1,2,3\right\},\\
	&S=\left\{(\bq,\bmu)\in \left((\cL^2(\Omega))^{3\times 2},(\cL^2(\Omega))^{3\times 2}\right), \bmu_k=\sqrt{\det \bM(\bq)}\bq_k\left(\bM(\bq)\right)^{-1}, k=1,2,3\right\},\\
	&\widetilde{S}=\left\{(\bq,\bnu)\in \left((\cL^2(\Omega))^{3\times 2},(\cL^2(\Omega))^{3\times 2}\right), \sqrt{\det \bM(\bq)-\alpha^2}\bnu_k=\bq_k\cof(\bM(\bq)), k=1,2,3\right\},
\end{align*}
and their indicator functions as
\begin{align*}
	&I_{\Sigma_f}(\bq)=
	\begin{cases}
		0 & \mbox{if } \bq\in \Sigma_f,\\
		+ \infty & \mbox{otherwise},
	\end{cases}\\
	&I_{S}(\bq,\bmu)=
	\begin{cases}
		0 & \mbox{if } (\bq,\bmu)\in S,\\
		+ \infty & \mbox{otherwise},
	\end{cases}
	\quad
	I_{\widetilde{S}}(\bq,\bnu)=
	\begin{cases}
		0 & \mbox{if } (\bq,\bnu)\in \widetilde{S},\\
		+ \infty & \mbox{otherwise}.
	\end{cases}
\end{align*}

Let $(\bp,\blambda)$ be the minimizer of
\begin{align}
	\min\limits_{\substack{\bq\in (\cH^1(\Omega))^{3\times 2},\\ \bmu\in (\cH^1(\Omega))^{3\times 2}}} \displaystyle\int_{\Omega} &\left(1+ \sum_{k=1}^3 |\nabla\cdot \bmu_k|^2 \right)\sqrt{m(\bq)} d\bx +\frac{1}{2\eta} \sum_{k=1}^3 \displaystyle\int_{\Omega} |(v_{\bq})_k-f_k|^2 d\bx \nonumber\\
	& +I_{\Sigma_f}(\bq)+I_S(\bq,\bmu),
	\label{eq.model1.var}
\end{align}
where $\bv_{\bq}=\{(v_{\bq})_k\}_{k=1}^3$ is the solution of
\begin{equation}
	\begin{cases}
		\nabla^2(v_{\bq})_k=\nabla \cdot \bq_k,\\
		(\nabla(v_{\bq})_k-\bq_k)\cdot \bn=0,\\
		\displaystyle\int_{\Omega} (v_{\bq})_k dx=\displaystyle\int_{\Omega} f_k d\bx,\\
		k=1,2,3.
	\end{cases}
	\label{eq.vq}
\end{equation}

Then, $\bu_{\bp}$ solving (\ref{eq.vq}) is the minimizer of (\ref{eq.model1}). In (\ref{eq.model1.var}), $\bv_{\bq}$ can be uniquely determined by $\bq$. The complicated Laplace-Beltrami operator is represented using the divergence of $\bmu$. The resulting formulation is an unconstrained optimization problem of $\bq$ and $\bmu$ only.

Similarly, if $(\bp,\blambda)$ be the minimizer of
\begin{align}
	\min\limits_{\substack{\bq\in (\cH^1(\Omega))^{3\times 2},\\ \bnu\in (\cH^1(\Omega))^{3\times 2}}}
	\displaystyle\int_{\Omega} &\left(1+ \sum_{k=1}^3 |\nabla\cdot \bnu_k|^2 \right)\sqrt{m(\bq)-\alpha^2} d\bx +\frac{1}{2\eta} \sum_{k=1}^3 \displaystyle\int_{\Omega} |(v_{\bq})_k-f_k|^2 d\bx\\
	& +I_{\Sigma_f}(\bq)+I_{\widetilde{S}}(\bq,\bnu).
	\label{eq.model2.var}
\end{align}
Then $\bu_{\bp}$ solving (\ref{eq.vq}) is the minimizer of (\ref{eq.model2}).

\section{Operator splitting methods}
\label{sec.splitting}
In this section, we derive operator-splitting schemes to solve (\ref{eq.model1.var}) and (\ref{eq.model2.var}). We will first derive the Euler-Lagrange equations of both functional and associate them with initial-value problems, which are suitable to be solved by operator-splitting methods.
\subsection{Optimal conditions of (\ref{eq.model1.var}) and (\ref{eq.model2.var})}
Define
\begin{align}
	&J_1(\bq,\bmu)=\int_{\Omega} \left(1+ \beta\sum_{k=1}^3 |\nabla\cdot \bmu_k|^2 \right)\sqrt{m(\bq)} d\bx,\\
	&\widetilde{J}_1(\bq,\bnu)=\int_{\Omega} \left(1+ \beta\sum_{k=1}^3 |\nabla\cdot \bnu_k|^2 \right)\sqrt{m(\bq)-\alpha^2} d\bx,\\
	&J_2(\bq)=\frac{1}{2\eta} \sum_{k=1}^3 \int_{\Omega} |(v_{\bq})_k-f_k|^2 d\bx.
\end{align}
%

If $(\bp,\blambda)$ is the minimizer of (\ref{eq.model1.var}), it satisfies 
\begin{equation}
	\begin{cases}
		\partial_{\bq} J_1(\bp,\blambda) +D_{\bq}J_2(\bp)+\partial_{\bq} I_{\Sigma_f}(\bp)+\partial_{\bq} I_{S}(\bp,\blambda) \ni 0,\\
		D_{\bmu} J_1(\bp,\blambda) +\partial_{\bmu} I_{S}(\bp,\blambda) \ni0,\\
	\end{cases}
	\label{eq.optimal1}
\end{equation}
where $D_{\bq}$ (resp. $\partial_{\bq}$) denotes the partial derivative (resp. subdifferential) of a differentiable function (resp. non-differentiable function) with respect to $\bq$. 

\begin{remark}\label{remark.S}
	Although the set $S$ may be nonconvex, which makes $I_S$ nonconvex, $I_S$ is an indicator function with special properties.
	Actually, the subdifferential of $I_{S}(\bq,\bmu)$ exists for any $(\bq,\bmu)\in S$, and $(\mathbf{0},\mathbf{0})$ is an element of it.
	Let $(\bq^*,\bmu^*)$ be any element of $S$, we have $I_S(\bq^*,\bmu^*)=0$ and
	\begin{align*}
		I_S(\bq,\bmu)\geq 0=I_S(\bq^*,\bmu^*)+ \langle \mathbf{0}, \bq-\bq^*\rangle + \langle \mathbf{0}, \bmu-\bmu^*\rangle,
	\end{align*}
	for any $(\bq,\bmu)$ (not necessarily in $S$). 
	By definition, the subdifferential of $I_S$ at $(\bq^*,\bmu^*)$ exists, and $(\mathbf{0},\mathbf{0})$ is an element of it. 
	
	In (\ref{eq.optimal1}), $(\bp,\blambda)$ is a minimizer of (\ref{eq.model1.var}). 
	We have $(\bp,\blambda)\in S$ and $(\mathbf{0},\mathbf{0})$ an element of the subdifferential of $I_S$ at $(\bp,\blambda)$,  which proves 
	(\ref{eq.optimal1}).
\end{remark}

We then introduce an artificial time and associate (\ref{eq.optimal1}) with the following initial-value problem
\begin{equation}
	\begin{cases}
		\frac{\partial \bp}{\partial t}+\partial_{\bq} J_1(\bp,\blambda) +D_{\bq}J_2(\bp)+\partial_{\bq} I_{\Sigma_f}(\bp)+\partial_{\bq} I_{S}(\bp,\blambda) \ni 0,\\
		\gamma_1\frac{\partial \blambda}{\partial t}+D_{\bmu}  J_1(\bp,\blambda) +\partial_{\bmu} I_{S}(\bp,\blambda) \ni0,\\
		(\bp(0),\blambda(0))=(\bp^0,\blambda^0),
	\end{cases}
	\label{eq.ivp1}
\end{equation}
where $\gamma_1$ is a positive constant controlling the evolution speed of $\blambda$, respectively. In (\ref{eq.ivp1}), $(\bp^0,\blambda^0)$ is the initial condition whose choice will be discussed in Section \ref{sec.initial}. Note that the steady state solution of (\ref{eq.ivp1}) solves (\ref{eq.model1.var}).

Similarly, for problem (\ref{eq.model2.var}), we solve for the steady state solution of the following initial-value problem
\begin{equation}
	\begin{cases}
		\frac{\partial \bp}{\partial t}+\partial_{\bq} \widetilde{J}_1(\bp,\blambda) +D_{\bq}J_2(\bp)+\partial I_{\Sigma_f}(\bp)+\partial_{\bq} I_{\widetilde{S}}(\bp,\blambda) \ni 0,\\
		\gamma_1\frac{\partial \blambda}{\partial t}+D_{\bnu} \widetilde{J}_1(\bp,\blambda) +\partial_{\bnu} I_{\widetilde{S}}(\bp,\blambda) \ni0,\\
		(\bp(0),\blambda(0))=(\bp^0,\blambda^0).
	\end{cases}
	\label{eq.ivp2}
\end{equation}
The argument in Remark \ref{remark.S} also applies to $I_{\widetilde{S}}$, 
which validates
(\ref{eq.ivp2}).

\subsection{Operator-splitting schemes}

The expression of (\ref{eq.ivp1}) and (\ref{eq.ivp2}) are well suited to be time discretized by operator-splitting methods, as what has been done in \cite{deng2019new,liu2021color,liu2021operator}. We refer the readers to \cite{glowinski2017splitting} for a complete discussion of operator-splitting methods.
A simple choice is the Lie scheme \cite{glowinski2016some}. In the following, we use $\tau$ to denote the time step. 
For (\ref{eq.ivp1}), we update $\bp,\blambda$ as follows:\\
\emph{\underline{Initialization}}
\begin{equation}
	\mbox{Initialize }\bp^0, \blambda^0.
	\label{eq.split.0}
\end{equation}
\emph{\underline{Fractional step 1}}\\
Solve
\begin{equation}
	\begin{cases}
		\begin{cases}
			\frac{\partial \bp}{\partial t}+\partial_{\bq} J_1(\bp,\blambda) \ni 0,\\
			\gamma_1 \frac{\partial \blambda}{\partial t}+ D_{\bmu} J_1(\bp,\blambda)\ni 0,
		\end{cases}
		\mbox{ in } \Omega\times (t^n,t^{n+1}),\\
		(\bp(t^n),\blambda(t^n))=(\bp^n,\blambda^n),
	\end{cases}
	\label{eq.split.1}
\end{equation}
and set $\bp^{n+1/3}=\bp(t^{n+1}),\blambda^{n+1/3}=\blambda(t^{n+1})$.

\noindent\emph{\underline{Fractional step 2}}\\
Solve
\begin{equation}
	\begin{cases}
		\begin{cases}
			\frac{\partial \bp}{\partial t}+\partial_{\bq} I_{S}(\bp,\blambda) \ni 0,\\
			\gamma_1 \frac{\partial \blambda}{\partial t}+ \partial_{\bmu}I_{S}(\bp,\blambda)\ni 0,
		\end{cases}
		\mbox{ in } \Omega\times (t^n,t^{n+1}),\\
		(\bp(t^n),\blambda(t^n))=(\bp^{n+1/3},\blambda^{n+1/3}),
	\end{cases}
	\label{eq.split.2}
\end{equation}
and set $\bp^{n+2/3}=\bp(t^{n+1}),\blambda^{n+2/3}=\blambda(t^{n+1})$.

\noindent\emph{\underline{Fractional step 3}}\\
Solve
\begin{equation}
	\begin{cases}
		\begin{cases}
			\frac{\partial \bp}{\partial t}+D_{\bq} J_2(\bp)+ \partial_{\bq} I_{\Sigma_f}(\bp) \ni 0,\\
			\gamma_1 \frac{\partial \blambda}{\partial t}= 0,
		\end{cases}
		\mbox{ in } \Omega\times (t^n,t^{n+1}),\\
		(\bp(t^n),\blambda(t^n))=(\bp^{n+2/3},\blambda^{n+2/3}),
	\end{cases}
	\label{eq.split.3}
\end{equation}
and set $\bp^{n+1}=\bp(t^{n+1}),\blambda^{n+1}=\blambda(t^{n+1})$.

For $\bp$ and $\blambda$ in (\ref{eq.split.1})-(\ref{eq.split.3}), we update them by implicit schemes. Specifically, we use the Marchuk-Yanenko type scheme to time discretize (\ref{eq.split.1})-(\ref{eq.split.3}):
Initialize $(\bp^0, \blambda^0)$.

For $n\geq0$, we update $(\bp^n,\blambda^n)\rightarrow (\bp^{n+1/3},\blambda^{n+1/3}) \rightarrow (\bp^{n+2/3},\blambda^{n+2/3}) \rightarrow (\bp^{n+1},\blambda^{n+1})$ as:
\begin{align}
	&\begin{cases}
		\frac{\bp^{n+1/3}-\bp^n}{\tau} + \partial_{\bq}  J_1(\bp^{n+1/3},\blambda^n) \ni 0,\\
		\gamma_1 \frac{\blambda^{n+1/3}-\blambda^n}{\tau} + D_{\bmu}  J_1(\bp^{n+1/3},\blambda^{n+1/3}) \ni 0,
	\end{cases} \label{eq.split1.1}\\
	&\begin{cases}
		\frac{\bp^{n+2/3}-\bp^{n+1/3}}{\tau} + \partial_{\bq} I_{S}(\bp^{n+2/3},\blambda^{n+2/3}) \ni 0,\\
		\gamma_1\frac{\blambda^{n+2/3}-\blambda^{n+1/3}}{\tau} + \partial_{\bmu} I_{S}(\bp^{n+2/3},\blambda^{n+2/3}) \ni 0,\\
	\end{cases}\label{eq.split1.2}\\
	&\begin{cases}
		\frac{\bp^{n+1}-\bp^{n+2/3}}{\tau} + D_{\bq} J_2(\bp^{n+1}) + \partial_{\bq} I_{\Sigma_f}(\bp^{n+1})\ni 0,\\
		\blambda^{n+1}=\blambda^{n+2/3},\\
	\end{cases}\label{eq.split1.3}
\end{align}

Analogously, the operator splitting scheme for (\ref{eq.ivp2}) is
\begin{align}
	&\begin{cases}
		\frac{\bp^{n+1/3}-\bp^n}{\tau} + \partial_{\bq}  \widetilde{J}_1(\bp^{n+1/3},\blambda^n) \ni 0,\\
		\gamma_1\frac{\blambda^{n+1/3}-\blambda^n}{\tau} + D_{\bnu}  \widetilde{J}_1(\bp^{n+1/3},\blambda^{n+1/3}) \ni 0,
	\end{cases} \label{eq.split2.1}\\
	&\begin{cases}
		\frac{\bp^{n+2/3}-\bp^{n+1/3}}{\tau} + \partial_{\bq} I_{\widetilde{S}}(\bp^{n+2/3},\blambda^{n+2/3}) \ni 0,\\
		\gamma_1\frac{\blambda^{n+2/3}-\blambda^{n+1/3}}{\tau} + \partial_{\bnu} I_{\widetilde{S}}(\bp^{n+2/3},\blambda^{n+2/3}) \ni 0.
	\end{cases}\label{eq.split2.2}\\
	&\begin{cases}
		\frac{\bp^{n+1}-\bp^{n+2/3}}{\tau} + D_{\bq} J_2(\bp^{n+1}) + \partial_{\bq} I_{\Sigma_f}(\bp^{n+1})\ni 0,\\
		\blambda^{n+1}=\blambda^{n+2/3}.
	\end{cases}\label{eq.split2.3}
\end{align}

In the rest of this section, we discuss solutions $(\bp,\blambda)$ to each of the subprobelm in scheme (\ref{eq.split1.1})--(\ref{eq.split1.3}) and (\ref{eq.split2.1})--(\ref{eq.split2.3}).


\subsection{On the solution of (\ref{eq.split1.1}) }
\label{sec.p1}
In (\ref{eq.split1.1}), $\bp^{n+1/3}$ is the minimizer of
\begin{align}
	\bp^{n+1/3}=\argmin_{\bq\in (L^2(\Omega))^{3\times 2}} \left[ \frac{1}{2\tau} \int_{\Omega} |\bq-\bp^n|^2dx +\int_{\Omega} \left( 1 +\beta \sum_{k=1}^3 |\nabla\cdot \blambda^n_k|^2\right)\sqrt{m(\bq)} d\bx \right].
	\label{eq.frac1.p}
\end{align}

We suggest to use the fixed point method to solve it.
The functional in (\ref{eq.frac1.p}) is in the form of
\begin{equation}
	E_1(\bq)=\frac{1}{2\tau} \int_{\Omega} |\bq-\bp|^2d\bx +\int_{\Omega} s\sqrt{m(\bq)} d\bx
\end{equation}
with some $s>0$ and $\bp\in (L^2(\Omega))^{3\times 2}$.

The first variation of $E_1$ with respect to $q_{kr},k=1,2,3, r=1,2,$ is
\begin{align}
	&\frac{\partial E_1}{\partial q_{kr}}=\frac{1}{\tau}(q_{kr}-p_{kr}) +\frac{s}{2\sqrt{m(\bq)}} \frac{\partial m(\bq)}{\partial q_{kr}},
\end{align}
with
\begin{align}
	&\frac{\partial m(\bq)}{\partial q_{k1}}=2g_{22}q_{k1}-2g_{12}q_{k2}, & &\frac{\partial m(\bq)}{\partial q_{k2}}=2g_{11}q_{k2}-2g_{12}q_{k1}
	\label{eq.updatep1}
\end{align}
for $k=1,2,3$. In (\ref{eq.updatep1}), the notation $\bM(\bq)=\begin{pmatrix}
	g_{11}& g_{12}\\ g_{21} & g_{22}
\end{pmatrix}$ is used. Given an initial guess $\bq^0$, in the $(\omega+1)$-th iteration, we freeze the denominator and update $\bq_{k}$ by solving for $\frac{\partial E_1}{\partial q_{kr}}=0$:
\begin{align}
	q_{k1}^{\omega+1}=\frac{\sqrt{m(\bq^{\omega})}p_{k1}+s\tau g_{12}^{\omega}q_{k2}^{\omega}}{\sqrt{m(\bq^{\omega})}+s\tau g_{22}^{\omega}},\quad
	q_{k2}^{\omega+1}=\frac{\sqrt{m(\bq^{\omega})}p_{k2}+s\tau g_{12}^{\omega}q_{k1}^{\omega}}{\sqrt{m(\bq^{\omega})}+s\tau g_{11}^{\omega}}.
	\label{eq.freezq}
\end{align}
We continue updating until $\|\bq^{\omega+1}-\bq^{\omega}\|_{\infty}< \xi_1$ for some small $\xi_1>0$, where we define $\|\bq\|_{\infty}=\max_{k,r} |q_{kr}|$. Then we set $\bp^{n+1/3}=\bq^*$ where $\bq^*$ is the converged variable.

For $\blambda^{n+1/3}$, it is the unique solution to
\begin{align}
	\begin{cases}
		\blambda^{n+1/3}=(\blambda^{n+1/3}_k)_{k=1}^3 \in (\cH^1(\Omega))^{3\times 2},\\
		\gamma_1 \displaystyle\int_{\Omega} \blambda_k^{n+1/3}\cdot \bmu_k d\bx+ 2\beta\tau \int_{\Omega} \sqrt{ m(\bp^{n+1/3})}(\nabla\cdot \blambda_k^{n+1/3})(\nabla \cdot \bmu_i)d\bx=\gamma_1 \displaystyle\int_{\Omega} \blambda_k^n \cdot \bmu_kd\bx,\\
		\forall \bmu_k\in (\cH^1(\Omega))^2, k=1,2,3.
	\end{cases}
	\label{eq.frac1.pde.lambda1.var.0}
\end{align}
Here $\blambda^{n+1/3}$ is also the weak solution to the linear elliptic Neumann problem
\begin{equation}
	\begin{cases}
		\gamma_1\blambda_k^{n+1/3}-2\beta\tau\nabla\left( \sqrt{ m(\bp^{n+1/3})}(\nabla\cdot \blambda_k^{n+1/3})\right)=\gamma_1\blambda_k^n &\mbox{ in } \Omega,\\
		\blambda_k^{n+1/3}\cdot \bn=0 & \mbox{ on } \partial\Omega, \\
		k=1,2,3,
	\end{cases}
	\label{eq.frac1.pde.lambda1.0}
\end{equation}
where $\bn$ denotes the outward normal direction.

In Section \ref{sec.G}, we introduce a new auxiliary variable $\bG$ approximating $\bM(\bp)$. See Remark \ref{remark.1} for another choice to update $\blambda^{n+1/3}$ in (\ref{eq.split1.1}).


\subsection{On the solution of (\ref{eq.split2.1})}
In (\ref{eq.split2.1}), we can compute $\bp^{n+1/3}$ in the same way as that in (\ref{eq.split1.1}), except replacing $m(\bq)$ in Section \ref{sec.p1} by $m(\bq)-\alpha^2$. The updating formula analogous to (\ref{eq.freezq}) is 
\begin{align}
	q_{k1}^{\omega+1}=\frac{p_{k1}+\frac{s\tau}{\sqrt{m(\bq^{\omega})-\alpha^2}+\varepsilon} g_{12}^{\omega}q_{k2}^{\omega}}{1+\frac{s\tau}{\sqrt{m(\bq^{\omega})-\alpha^2}+\varepsilon} g_{22}^{\omega}},\quad
	q_{k2}^{\omega+1}=\frac{p_{k2}+\frac{s\tau}{\sqrt{m(\bq^{\omega})-\alpha^2}+\varepsilon} g_{12}^{\omega}q_{k1}^{\omega}}{1+\frac{s\tau}{\sqrt{m(\bq^{\omega})-\alpha^2}+\varepsilon} g_{11}^{\omega}}.
	\label{eq.freezq2}
\end{align}
In (\ref{eq.freezq2}), $\varepsilon>0$ is a small number to avoid division by 0. In our experiments, $\varepsilon=10^{-3}$ gives fast convergence rate while providing good results.

For $\blambda^{n+1/3}$, follow the derivation Section \ref{sec.p1}, it is the weak solution of 
\begin{equation}
	\begin{cases}
		\gamma_1\blambda_k^{n+1/3}-2\beta\tau\nabla\left( \sqrt{m(\bp^{n+1/3})-\alpha^2}(\nabla\cdot \blambda_k^{n+1/3})\right)=\gamma_1\blambda_k^n &\mbox{ in } \Omega,\\
		\blambda_k^{n+1/3}\cdot \bn=0 & \mbox{ on } \partial\Omega, \\
		k=1,2,3.
	\end{cases}
	\label{eq.frac1.pde.lambda2.0}
\end{equation}
See Remark \ref{remark.2} for another option to update $\blambda_k^{n+1/3}$ in (\ref{eq.split2.1}).


\subsection{On the solution of (\ref{eq.split1.2})}
\label{sec.G}
The solution $(\bp^{n+2/3},\blambda^{n+2/3})$ in (\ref{eq.split1.2}) is the minimizer of
\begin{equation}
	(\bp^{n+2/3},\blambda^{n+2/3})=\argmin_{(\bq,\bmu)\in S} \frac{1}{2}\int_{\Omega} \left( \left| \bq-\bp^{n+1/3}\right|^2 +\gamma_1 \left| \bmu-\blambda^{n+1/3}\right|^2 \right)d\bx.
	\label{eq.frac2.lambda1.0}
\end{equation}
The constraint in the set $S$ is nonlinear in $\bq$, making (\ref{eq.frac2.lambda1.0}) difficult to solve. In this paper, instead of directly solving (\ref{eq.frac2.lambda1.0}), we borrow the idea of sequential quadratic programming (SQP) \cite{nocedal1999numerical} and replace $S$ by $S_G$  defined by
\begin{align}
	S_{\bG}=\left\{(\bq,\bmu)\in \left((\cL^2(\Omega))^{3\times 2},(\cL^2(\Omega))^{3\times 2}\right), \bmu_k=\sqrt{\det \bG}\bq_k\bG^{-1}, k=1,2,3\right\}
	\label{eq.S_bG}
\end{align}
for some $\bG$ close to $\bM(\bq^{n+2/3})$. The constraint in $S_{\bG}$ is a `linearization' of the constraint in $S$ (in the flavor of the fixed point method). Such a strategy is a variant of the first step of the first order SQP studied in \cite{bai2018analysis} with $x_0=[\bp^{n+1/3},\blambda^{n+1/3}]$. 

For a small $\tau$, we expect that $\bp^{n+1/3}$ is close to $\bp^{n+2/3}$. Thus $\bG=\bM(\bp^{n+1/3})$ is a natural choice. However, such a choice makes our algorithm unstable. To improve the robustness, we take $\bG$ as a variable and update it with $\bp$ during iterations (with damping) while keeping it being close to $\bM(\bp)$. Specifically, given an initial value $\bG^0$, for $k=1,2,3$, every time $\bp^{n+k/3}$ is computed, we update $\bG^{n+k/3}$ with damping as
\begin{align}
	\bG^{n+k/3}=e^{-\gamma_2\tau}\bG^{n+(k-1)/3}+(1-e^{-\gamma_2\tau})\bM(\bp^{n+k/3})
	\label{eq.bG}
\end{align}
for some $\gamma_2>0$. The updating formula (\ref{eq.bG}) is the solution $\bG(t^{n+1})$ to the differential equation
\begin{align}
	\frac{\partial \bG}{\partial t}+\gamma_2(\bG-\bM(\bp^{n+k/3}))=0
\end{align}
given $\bG(t^n)=\bG^{n+(k-1)/3}$. In (\ref{eq.S_bG}), we set $\bG=\bG^{n+1/3}$, i.e., the latest $\bG$ computed using $\bp^{n+1/3}$. Consequently, (\ref{eq.frac2.lambda1.0}) is approximated by
\begin{equation}
	(\bp^{n+2/3},\blambda^{n+2/3})=\argmin_{(\bq,\bmu)\in S_{\bG^{n+1/3}}} \frac{1}{2}\int_{\Omega} \left( \left| \bq-\bp^{n+1/3}\right|^2 +\gamma_1 \left| \bmu-\blambda^{n+1/3}\right|^2 \right)d\bx.
	\label{eq.frac2.lambda1}
\end{equation}

We then focus on the solution of (\ref{eq.frac2.lambda1}). For $(\bq,\bmu)\in S_{\bG^{n+1/3}}$, we have $\bmu_i=\sqrt{g^{n+1/3}}\bq_i(\bG^{n+1/3})^{-1}$ with $g^{n+1/3}=\det \bG^{n+1/3}$. While when $\alpha$ is very close to 0, the matrix $\bG^{n+1/3}$ maybe singular and $g^{n+1/3}$ may equal to 0 at some locations. To avoid computing $(\bG^{n+1/3})^{-1}$ and $1/\sqrt{g^{n+1/3}}$, we rewrite the relation as
\begin{align}
	\sqrt{g^{n+1/3}}\bq=\bmu\bG^{n+1/3},
	\label{eq.frac2.constraint1}
\end{align}
and consider the following constrained optimization problem
\begin{equation}
	(\bp^{n+2/3},\blambda^{n+2/3})=\argmin_{(\bq,\bmu):\sqrt{g^{n+1/3}}\bq=\bmu\bG^{n+1/3}} \int_{\Omega} \left( \left| \bq-\bp^{n+1/3}\right|^2 +\gamma_1 \left| \bmu-\blambda^{n+1/3}\right|^2 \right)dx.
	\label{eq.frac2.lambda1.1}
\end{equation}
For simplicity, we temporally use $g, g_{ij},p_{kr},\lambda_{kr}$ to denote $g^{n+1/3}, g_{ij}^{n+1/3},p_{kr}^{n+1/3},\lambda_{kr}^{n+1/3}$ in this subsection.
The explicit solution to (\ref{eq.frac2.lambda1.1}) is given in the following theorem
\begin{thm}\label{thm.frac2}
	Let
	\begin{align*}
		&a_{k1}=-\frac{g_{11}^2}{\gamma_1}-\frac{g_{21}^2}{\gamma_1}-g,\ a_{i2}=-\frac{g_{11}g_{12}}{\gamma_1}-\frac{g_{21}g_{22}}{\gamma_1},\ a_{k3}=\lambda_{k1}g_{11}+\lambda_{k2}g_{21}-\sqrt{g}p_{k1},\\
		&b_{k1}=-\frac{g_{11}g_{12}}{\gamma_1}-\frac{g_{21}g_{22}}{\gamma_1},\ b_{i2}=-\frac{g_{12}^2}{\gamma_1}-\frac{g_{22}^2}{\gamma_1}-g,\ b_{k3}=\lambda_{k1}g_{12}+\lambda_{k2}g_{22}-\sqrt{g}p_{k2}.
	\end{align*}
	The solution to (\ref{eq.frac2.lambda1.1}) is given as
	\begin{align*}
		&\lambda_{k1}^{n+2/3}=\lambda_{k1}-\frac{g_{11}}{\gamma_1}\frac{a_{k2}b_{k3}-a_{k3}b_{k2}}{a_{k1}b_{k2}-a_{k2}b_{k1}}- \frac{g_{12}}{\gamma_1}\frac{a_{k1}b_{k3}-a_{k3}b_{k1}}{a_{k2}b_{k1}-a_{k1}b_{k2}},\\
		&\lambda_{k2}^{n+2/3}=\lambda_{k2}-\frac{g_{21}}{\gamma_1}\frac{a_{k2}b_{k3}-a_{k3}b_{k2}}{a_{k1}b_{k2}-a_{k2}b_{k1}}- \frac{g_{22}}{\gamma_1}\frac{a_{k1}b_{k3}-a_{k3}b_{k1}}{a_{k2}b_{k1}-a_{k1}b_{k2}},\\
		&p_{k1}^{n+2/3}=p_{k1}+\sqrt{g}\frac{a_{k2}b_{k3}-a_{k3}b_{k2}}{a_{k1}b_{k2}-a_{k2}b_{k1}},\ p_{k2}^{n+2/3}=p_{k2}+\sqrt{g}\frac{a_{k1}b_{k3}-a_{k3}b_{k1}}{a_{k2}b_{k1}-a_{k1}b_{k2}}.
	\end{align*}
	for $k=1,2,3$. 
\end{thm}
\begin{proof}
	We derive the formulas using the method of Lagrange multipliers. The formula can be derived component-wisely. Let $k\in\{1,2,3\}.$ The constraint (\ref{eq.frac2.constraint1}) implies 
	\begin{align}
		\begin{cases}
			F_1(\bmu_k,\bq_k)=0,\\
			F_2(\bmu_k,\bq_k)=0,
		\end{cases}
		\mbox{ with }
		\begin{cases}
			F_1(\bmu_k,\bq_k)=g_{11}\mu_{k1}+g_{21}\mu_{k2}-\sqrt{g}q_{k1},\\
			F_2(\bmu_k,\bq_k)=g_{12}\mu_{k1}+g_{22}\mu_{k2}-\sqrt{g}q_{k2}.
		\end{cases}
		\label{eq.frac2.proof.1}
	\end{align}
	To derive the formula for $\bmu_k,\bq_k$, consider the following Lagrangian functional
	\begin{align}
		E_{k2}=\int_{\Omega} \left( \left| \bq_k-\bp_k\right|^2 +\gamma_1 \left| \bmu_k-\blambda_k\right|^2 \right)d\bx+s_{k1}F_1(\bmu_k,\bq_k)+ s_{k2}F_2(\bmu_k,\bq_k),
	\end{align}
	where $s_{k1},s_{k2}$ are Lagrange multipliers. Computing partial derivatives of $E_{k2}$ with respect to $\bmu_k,\bq_k$ and setting them to 0 gives rise to
	\begin{align}
		\begin{cases} 
			\gamma_1(\mu_{k1}-\lambda_{k1})+s_{k1}g_{11}+s_{k2}g_{12}=0,\\
			\gamma_1(\mu_{k2}-\lambda_{k2})+s_{k1}g_{21}+s_{k2}g_{22}=0,\\
			q_{k1}-p_{k1}-s_{k1}\sqrt{g}=0,\\
			q_{k2}-p_{k2}-s_{k2}\sqrt{g}=0.
		\end{cases}
		\label{eq.frac2.proof.2}
	\end{align}
	Reorganizing (\ref{eq.frac2.proof.2}), we have
	\begin{align}
		\begin{cases}
			\mu_{k1}=\lambda_{k1}-\frac{g_{11}}{\gamma_1}s_{k1}-\frac{g_{12}}{\gamma_1}s_{k2},\\
			\mu_{k2}=\lambda_{k2}-\frac{g_{21}}{\gamma_1}s_{k1}-\frac{g_{22}}{\gamma_1}s_{k2},\\
			q_{k1}=p_{k1}+s_{k1}\sqrt{g},\\
			q_{k2}=p_{k2}+s_{k2}\sqrt{g}.
		\end{cases}
		\label{eq.frac2.proof.3}
	\end{align}
	Substituting (\ref{eq.frac2.proof.3}) into (\ref{eq.frac2.proof.1}) gives rise to
	\begin{align}
		\begin{cases}
			a_{k1}s_{k1}+a_{k2}s_{k2}+a_{k3}=0,\\
			b_{k1}s_{k1}+b_{k2}s_{k2}+b_{k3}=0
		\end{cases}
		\label{eq.frac2.proof.4}
	\end{align}
	for $a_{k1},a_{k2},a_{k3},b_{k1},b_{k2},b_{k3}$ defined in Theorem \ref{thm.frac2}. Solving (\ref{eq.frac2.proof.4}) for $s_{k1},s_{k2}$ gives
	\begin{align}
		s_{k1}=\frac{a_{k2}b_{k3}-b_{k2}a_{k3}}{b_{k2}a_{k1}-a_{k2}b_{k1}}, \quad s_{k2}=\frac{a_{k1}b_{k3}-a_{k3}b_{k1}}{b_{k1}a_{k2}-b_{k2}a_{k1}}.
		\label{eq.frac2.proof.5}
	\end{align}
	Substituting (\ref{eq.frac2.proof.5}) into (\ref{eq.frac2.proof.3}) finishes the proof.
\end{proof}

\begin{remark}
	We remark that the approximation (\ref{eq.frac2.lambda1}) is not strictly the first step of the first order SQP \cite{bai2018analysis} to solve (\ref{eq.frac2.lambda1.0}), as we did not use the Jacobian of the constraint in $S$ when constructing $S_{\bG}$. Nevertheless, numerical experiments suggest that our proposed algorithm converges with this numerical approximation.
\end{remark}

\begin{remark}\label{remark.1}
	Using $\bG^{n+1/3}$ and $g^{n+1/3}$ defined above, another option to update $\blambda^{n+1/3}$ in (\ref{eq.split1.1}) is to replace $m(\bp^{n+1/3})$ in (\ref{eq.frac1.pde.lambda1.var.0}) and (\ref{eq.frac1.pde.lambda1.0}) by $g^{n+1/3}$ to get
	\begin{align}
		\begin{cases}
			\blambda^{n+1/3}=(\blambda^{n+1/3}_k)_{k=1}^3 \in (\cH^1(\Omega))^{3\times 2},\\
			\gamma_1 \displaystyle\int_{\Omega} \blambda_k^{n+1/3}\cdot \bmu_k d\bx+ 2\beta\tau \int_{\Omega} \sqrt{g^{n+1/3}}(\nabla\cdot \blambda_k^{n+1/3})(\nabla \cdot \bmu_i)d\bx=\gamma_1 \displaystyle\int_{\Omega} \blambda_k^n \cdot \bmu_kd\bx,\\
			\forall \bmu_k\in (\cH^1(\Omega))^2, k=1,2,3.
		\end{cases}
		\label{eq.frac1.pde.lambda1.var}
	\end{align}
	and
	\begin{equation}
		\begin{cases}
			\gamma_1\blambda_k^{n+1/3}-2\beta\tau\nabla\left( \sqrt{g^{n+1/3}}(\nabla\cdot \blambda_k^{n+1/3})\right)=\gamma_1\blambda_k^n &\mbox{ in } \Omega,\\
			\blambda_k^{n+1/3}\cdot \bn=0 & \mbox{ on } \partial\Omega, \\
			k=1,2,3.
		\end{cases}
		\label{eq.frac1.pde.lambda1}
	\end{equation}
	Here $g^{n+1/3}$ is a relaxed version of $m(\bp^{n+1/3})$. Such a treatment enhances the coupling between $(\bp,\blambda)$ and $\bG$. In our experiments, (\ref{eq.frac1.pde.lambda1.var})-(\ref{eq.frac1.pde.lambda1}) give similar results as (\ref{eq.frac1.pde.lambda1.var.0})-(\ref{eq.frac1.pde.lambda1.0}). While our algorithm is more stable with (\ref{eq.frac1.pde.lambda1.var})-(\ref{eq.frac1.pde.lambda1}). In the rest of the paper, we stick with (\ref{eq.frac1.pde.lambda1.var})-(\ref{eq.frac1.pde.lambda1}).
\end{remark}

\begin{remark}
	\label{remark.2}
	Similar to Remark \ref{remark.1}, another option to update $\blambda^{n+1/3}$ in (\ref{eq.split2.1}) is to replace $m(\bp^{n+1/3})$ in (\ref{eq.frac1.pde.lambda2.0}) by $g^{n+1/3}$ to get
	\begin{equation}
		\begin{cases}
			\gamma_1\blambda_k^{n+1/3}-2\beta\tau\nabla\left( \sqrt{g^{n+1/3}-\alpha^2}(\nabla\cdot \blambda_k^{n+1/3})\right)=\gamma_1\blambda_k^n &\mbox{ in } \Omega,\\
			\blambda_k^{n+1/3}\cdot \bn=0 & \mbox{ on } \partial\Omega, \\
			k=1,2,3.
		\end{cases}
		\label{eq.frac1.pde.lambda2}
	\end{equation}
	In the rest of the paper, we stick with (\ref{eq.frac1.pde.lambda2}).
\end{remark}
\subsection{On the solution of (\ref{eq.split2.2})}
\label{sec.G_tilde}

The solution $(\bp^{n+2/3},\blambda^{n+2/3})$ is the minimizer of
\begin{equation}
	(\bp^{n+2/3},\blambda^{n+2/3})=\argmin_{(\bq,\bnu)\in \widetilde{S}} \int_{\Omega} \left( \left| \bq-\bp^{n+1/3}\right|^2 +\gamma_1 \left| \bmu-\blambda^{n+1/3}\right|^2 \right)d\bx.
	\label{eq.frac2.lambda2.0}
\end{equation}
Similar to what has been done in Section \ref{sec.G}, we approximate (\ref{eq.frac2.lambda2.0}) by
\begin{equation}
	(\bp^{n+2/3},\blambda^{n+2/3})=\argmin_{(\bq,\bnu)\in \widetilde{S}_{\bG^{n+1/3}}} \int_{\Omega} \left( \left| \bq-\bp^{n+1/3}\right|^2 +\gamma_1 \left| \bmu-\blambda^{n+1/3}\right|^2 \right)d\bx
	\label{eq.frac2.lambda2}
\end{equation}
with
\begin{align}
	\widetilde{S}_{\bG}=\left\{(\bq,\bnu)\in \left((\cL^2(\Omega))^{3\times 2},(\cL^2(\Omega))^{3\times 2}\right), \sqrt{\det \bG-\alpha^2}\bnu_k=\bq_k\cof(\bG), k=1,2,3\right\}.
\end{align}

Denote $g^{n+1/3}=\det \bG^{n+1/3}$. In $\widetilde{S}_{\bG^{n+1/3}}$, we have $\sqrt{g^{n+1/3}-\alpha^2}\bnu=\bq\cof(\bG^{n+1/3})$. We consider the following constrained optimization problem
\begin{equation}
	(\bp^{n+2/3},\blambda^{n+2/3})=\argmin_{(\bq,\bnu):\sqrt{g^{n+1/3}-\alpha^2}\bnu=\bq\cof(\bG^{n+1/3})} \int_{\Omega} \left( \left| \bq-\bp^{n+1/3}\right|^2 +\gamma_1 \left| \bnu-\blambda^{n+1/3}\right|^2 \right)d\bx.
	\label{eq.frac2.lambda2.1}
\end{equation}
For simplicity, we temporally use $g, g_{ij},p_{kr},\lambda_{kr}$ to denote $g^{n+1/3}, g_{ij}^{n+1/3},p_{kr}^{n+1/3},\lambda_{kr}^{n+1/3}$ in this subsection. The explicit solution for (\ref{eq.frac2.lambda2.1}) is given in the following theorem
\begin{thm}\label{thm.frac2.2}
	Let 
	\begin{align*}
		&a_{k1}=-g_{22}^2-g_{12}^2-\frac{g-\alpha^2}{\gamma_1},\ a_{k2}=g_{12}g_{22}+g_{11}g_{12},\ a_{k3}=g_{22}p_{k1}-g_{12}p_{k2}-\sqrt{g-\alpha^2}\lambda_{k1},\\
		&b_{k1}=-g_{11}g_{12}+g_{12}g_{22}, \ b_{k2}=g_{11}^2-g_{12}^2-\frac{g-\alpha^2}{\gamma_1},\ b_{k3}=g_{11}p_{k1}-g_{12}p_{k2}-\sqrt{g-\alpha^2}\lambda_{k2}.
	\end{align*}
	The solution to (\ref{eq.frac2.lambda2.1}) is given as
	\begin{align*}
		&\lambda_{k1}^{n+2/3}=\lambda_{k1}+ \frac{\sqrt{g-\alpha^2}}{\gamma_1} \frac{a_{k2}b_{k3}-b_{k2}a_{k3}}{b_{k2}a_{k1}-a_{k2}b_{k1}},\\
		&\lambda_{k2}^{n+2/3}=\lambda_{k2}+ \frac{\sqrt{g-\alpha^2}}{\gamma_1} \frac{a_{k1}b_{k3}-b_{k1}a_{k3}}{b_{k1}a_{k2}-a_{k1}b_{k2}},\\
		&p_{k1}^{n+2/3}=p_{k1}-g_{22}\frac{a_{k2}b_{k3}-a_{k3}-b_{k2}}{a_{k1}b_{k2}-a_{k2}b_{k1}}+ g_{22}\frac{a_{k1}b_{k3}-a_{k3}b_{k1}}{a_{k2}b_{k1}-a_{k1}b_{k2}},\\
		&p_{k2}^{n+2/3}=p_{k2}+g_{12}\frac{a_{k2}b_{k3}-a_{k3}-b_{k2}}{a_{k1}b_{k2}-a_{k2}b_{k1}}- g_{11}\frac{a_{k1}b_{k3}-a_{k3}b_{k1}}{a_{k2}b_{k1}-a_{k1}b_{k2}},
	\end{align*}
	for $i=1,2,3$.
\end{thm}
Theorem \ref{thm.frac2.2} can be proved similarly to Theorem \ref{thm.frac2}. The proof is omitted here.

\subsection{On the solution of (\ref{eq.split1.3}) and (\ref{eq.split2.3})}
Problems (\ref{eq.split1.3}) and (\ref{eq.split2.3}) are the same, in which $\bp^{n+1}$ solves
\begin{align}
	\argmin_{\bq\in (\cH^1(\Omega))^{3\times 2}}\left[\frac{1}{2\tau}\int_{\Omega} |\bq-\bp^{n+2/3}|^2 d\bx+ \frac{1}{2\eta} \sum_{k=1}^3 \int_{\Omega} |(v_{\bq})_k-f_k|^2 d\bx + I_{\Sigma_f}(\bq)\right].
\end{align}
Since $\bp^{n+1}\in \Sigma_f$, there exists some $\bu^{n+1}\in (\cH^1(\Omega))^3$ such that $\bp^{n+1}=\nabla \bu^{n+1}$. Furthermore, $\bu^{n+1}$ solves
\begin{align}
	\argmin_{\bv\in (\cH^1(\Omega))^{3}}\left[\frac{1}{2\tau}\int_{\Omega} |\nabla \bv-\bp^{n+2/3}|^2 d\bx+ \frac{1}{2\eta} \sum_{k=1}^3 \int_{\Omega} |v_k-f_k|^2 d\bx\right].
	\label{eq.frac3.var}
\end{align}
Here $\bu^{n+1}$ is also the unique weak solution of the folloing linear elliptic problem
\begin{align}
	\begin{cases} 
		-\eta\nabla^2 u_k^{n+1} + \tau\frac{1}{\eta} u_k^{n+1}=-\eta\nabla\cdot \bp_k^{n+2/3}+\tau f_k \mbox{ in } \Omega,\\
		\nabla u_k^{n+1}\cdot \bn=0 \mbox{ on } \partial \Omega,\\
		k=1,2,3.
	\end{cases}
	\label{eq.frac3.pde}
\end{align}
After $\bu^{n+1}$ is solved, we set $\bp^{n+1}=\nabla \bu^{n+1}$.
\subsection{Initial condition} \label{sec.initial}
Both schemes (\ref{eq.split1.1})--(\ref{eq.split1.3}) and (\ref{eq.split2.1})--(\ref{eq.split2.3}) require an initial condition $(\bp^0,\blambda^0)$. We first initialize 
$$\bu^0=\bbf \quad \mbox{ or } \quad \bu^0=\mathbf{0}$$ 
and compute $\bp^0=\nabla \bu^0$.
Then we set $\blambda^0=\sqrt{\det \bp^0}\bp^0(\bM(\bp^0))^{-1}$ for problem (\ref{eq.model1.var}) or $\blambda^0=\frac{1}{\sqrt{\det \bp^0-\alpha^2}}\bp^0\cof(\bM(\bp^0))$ for problem (\ref{eq.model2.var}).

We also take $\bG$ as a variable which is updated during iterations, we initialize $\bG^0=\bM(\bp^0)$.

Our algorithms for problems (\ref{eq.model1.var}) and  problem (\ref{eq.model2.var}) are summarized in Algorithm \ref{alg.1} and \ref{alg.2}, respectively.
\begin{remark}
	The functionals in (\ref{eq.model1}) and (\ref{eq.model2}) are complicated and nonconvex. All we can expect is that our algorithms converge to a local minimizer. However, in our experiments, both algorithms are robust to initial conditions and noise. With either initial condition discussed above, every time we generate noisy images with random Gaussian noise, our algorithms always provide good results.
\end{remark}	

\begin{algorithm}[th!]
	\caption{\label{alg.1}An operator-splitting method for solving problem (\ref{eq.model1.var}).}
	\begin{algorithmic}
		\STATE {\bf Input:} The noisy image $\bbf$, parameters $\alpha,\beta,\eta,\tau, \gamma_1,\gamma_2$.
		\STATE {\bf Initialization:} Set  $n=0,$ $(\bp^0,\blambda^0,\bG^0)=(\bp_0,\blambda_0,\bG_0)$ as discussed in Section \ref{sec.initial}.
		\WHILE{not converge}
		\STATE 1. Solve (\ref{eq.split1.1}) using (\ref{eq.freezq}) and (\ref{eq.frac1.pde.lambda1}) (or (\ref{eq.frac1.pde.lambda1.0})) to obtain  $(\bp^{n+1/3}, \blambda^{n+1/3})$. \\
		\ \ \ \ Update $\bG^{n+1/3}$ using (\ref{eq.bG}).
		\STATE 2. Solve (\ref{eq.split1.2}) using Theorem \ref{thm.frac2} to obtain $(\bp^{n+2/3}, \blambda^{n+2/3})$.\\
		\ \ \ \ Update $\bG^{n+2/3}$ using (\ref{eq.bG}).
		\STATE 3. Solve (\ref{eq.split1.3}) using (\ref{eq.frac3.pde}) to obtain $(\bp^{n+1}, \blambda^{n+1})$.\\
		\ \ \ \ Update $\bG^{n+1}$ using (\ref{eq.bG}).
		\STATE 4. Set $n=n+1$.
		\ENDWHILE
		\STATE Solve (\ref{eq.vq}) using the converged function $\bp^*$ to obtain $u^*$.
		\STATE {\bf Output:} The function $u^*$.
	\end{algorithmic}
\end{algorithm}

\begin{algorithm}[th!]
	\caption{\label{alg.2}An operator-splitting method for solving problem (\ref{eq.model2.var}).}
	\begin{algorithmic}
		\STATE {\bf Input:} The noisy image $f$, parameters $\alpha,\beta,\eta,\tau,\gamma_1,\gamma_2$.
		\STATE {\bf Initialization:} Set $n=0,$ $(\bp^0,\blambda^0)=(\bp_0,\blambda_0)$.
		\WHILE{not converge}
		\STATE 1. Solve (\ref{eq.split2.1}) using (\ref{eq.freezq2}) and (\ref{eq.frac1.pde.lambda2}) (or (\ref{eq.frac1.pde.lambda2.0})) to obtain $(\bp^{n+1/3}, \blambda^{n+1/3})$.\\
		\ \ \ \ Update $\bG^{n+1/3}$ using (\ref{eq.bG}).
		\STATE 2. Solve (\ref{eq.split2.2}) using Theorem \ref{thm.frac2.2} to obtain $(\bp^{n+2/3}, \blambda^{n+2/3})$.\\
		\ \ \ \ Update $\bG^{n+2/3}$ using (\ref{eq.bG}).
		\STATE 3. Solve (\ref{eq.split2.3}) using (\ref{eq.frac3.pde}) to obtain $(\bp^{n+1}, \blambda^{n+1})$.\\
		\ \ \ \ Update $\bG^{n+1}$ using (\ref{eq.bG}).
		\STATE 4. Set $n=n+1$.
		\ENDWHILE
		\STATE Solve (\ref{eq.vq}) using the converged function $\bp^*$ to obtain $u^*$.
		\STATE {\bf Output:} The function $u^*$.
	\end{algorithmic}
\end{algorithm}

\subsection{On the periodic boundary condition}
The solutions to each subproblem discussed so far use Neumann boundary conditions. In image processing, the periodic boundary condition is a popular condition which allows one to use FFT. In this subsection, we discuss the minimal efforts necessary to modify the algorithms and solvers above to accommodate periodic boundary conditions.

Assume our computational domain is $\Omega=[0,L_1]\times[0,L_2]$. We first replace the function space $\cH^1(\Omega)$ by $\cH^1_P(\Omega)$ defined as
\begin{align*}
	\cH^1_P(\Omega)=\left\{v\in \cH^1(\Omega): v(0,:)=v(L_1,:), \ v(:,0)=v(:,L_2)\right\}.
\end{align*}
The set $\Sigma_f$ is replaced by
\begin{align*}
	\Sigma_f=\left\{\bq\in (\cL^2(\Omega))^{3\times 2}, \exists \bv\in \cH_P^1(\Omega) \mbox{ such that } \bq=\nabla v \mbox{ and } \int_{\Omega} v_k-f_k dx=0 \mbox{ for }k=1,2,3\right\}.
\end{align*}
Problems (\ref{eq.model1.var}), (\ref{eq.vq}) and (\ref{eq.model2.var}) are replaced by
\begin{align}
	&\min\limits_{\substack{\bq\in (\cH_P^1(\Omega))^{3\times 2},\\ \bmu\in (\cH_P^1(\Omega))^{3\times 2}}} \displaystyle\int_{\Omega} \left(1+ \sum_{k=1}^3 |\nabla\cdot \bmu_k|^2 \right)\sqrt{m(\bq)} d\bx  \nonumber\\
	&\hspace{4cm}+\frac{1}{2\eta} \displaystyle\sum\limits_{k=1}^3 \displaystyle\int_{\Omega} |(v_{\bq})_k-f_k|^2 d\bx +I_{\Sigma_f}(\bq)+I_{S}(\bq,\bmu),
	\label{eq.model1.var.periodic}
\end{align}
\begin{equation}
	\begin{cases}
		\nabla^2(v_{\bq})_k=\nabla \cdot \bq_k,\\
		(v_{\bq})_k(x_1,0)=(v_{\bq})_k(x_1,L_2),\ 0\leq x_1\leq L_1,\\
		(v_{\bq})_k(0,x_2)=(v_{\bq})_k(L_1,x_2),\ 0\leq x_2\leq L_2,\\
		\left(\frac{\partial (v_{\bq})_k}{\partial x_2}-q_{k2}\right)(x_1,0)=\left(\frac{\partial (v_{\bq})_k}{\partial x_2}-q_{k2}\right)(x_1,L_2),\ 0\leq x_1\leq L_1,\\
		\left(\frac{\partial (v_{\bq})_k}{\partial x_1}-q_{k1}\right)(0,x_2)=\left(\frac{\partial (v_{\bq})_k}{\partial x_1}-q_{k1}\right)(L_1,x_2),\ 0\leq x_2\leq L_2,\\
		\displaystyle\int_{\Omega} (v_{\bq})_k d\bx=\displaystyle\int_{\Omega} f_k d\bx,
	\end{cases}
	\label{eq.vq.periodic}
\end{equation}
and
\begin{align}
	&\min\limits_{\substack{\bq\in (\cH_P^1(\Omega))^{3\times 2},\\ \bnu\in (\cH_P^1(\Omega))^{3\times 2}}}
	\displaystyle\int_{\Omega} \left(1+ \sum_{k=1}^3 |\nabla\cdot \bnu_k|^2 \right)\sqrt{m(\bq)-\alpha^2} d\bx \nonumber\\
	&\hspace{4cm} +\frac{1}{2\eta} \displaystyle\sum_{k=1}^3 \displaystyle\int_{\Omega} |(v_{\bq})_k-f_k|^2 d\bx +I_{\Sigma_f}(\bq)+I_{\widetilde{S}}(\bq,\bnu),
	\label{eq.model2.var.periodic}
\end{align}
respectively. We modify subproblem solvers as follows

Replace(\ref{eq.frac1.pde.lambda1.var}), (\ref{eq.frac1.pde.lambda1}) and (\ref{eq.frac1.pde.lambda2}) by
\begin{align}
	\begin{cases}
		\blambda^{n+1/3}=(\blambda^{n+1/3}_k)_{k=1}^3 \in (\cH_P^1(\Omega))^{3\times 2},\\
		\gamma_1 \displaystyle\int_{\Omega} \blambda_k^{n+1/3}\cdot \bmu_k dx+ 2\beta\tau \int_{\Omega} \sqrt{g^{n+1/3}}(\nabla\cdot \blambda_k^{n+1/3})(\nabla \cdot \bmu_k)dx=\gamma_1 \displaystyle\int_{\Omega} \blambda_k^n \cdot \bmu_k d\bx,\\
		\forall \bmu_k\in (\cH_P^1(\Omega))^2, k=1,2,3,
	\end{cases}
	\label{eq.frac1.pde.lambda1.var.periodic}
\end{align}
\begin{equation}
	\begin{cases}
		\gamma_1\blambda_k^{n+1/3}-2\beta\tau\nabla\left( \sqrt{g^{n+1/3}}(\nabla\cdot \blambda_k^{n+1/3})\right)=\gamma_1\blambda_k^n &\mbox{ in } \Omega,\\
		\lambda_{k2}^{n+1/3}(x_1,0)=\lambda_{k2}^{n+1/3}(x_1,L_2),\ 0\leq x_1\leq L_1,\\
		\lambda_{k1}^{n+1/3}(0,x_2)=\lambda_{k1}^{n+1/3}(L_1,x_2),\ 0\leq x_2\leq L_2,\\
		k=1,2,3,
	\end{cases}
	\label{eq.frac1.pde.lambda1.periodic}
\end{equation}
and
\begin{equation}
	\begin{cases}
		\gamma_1\blambda_k^{n+1/3}-2\beta\tau\nabla\left( \sqrt{g^{n+1/3}-\alpha^2}(\nabla\cdot \blambda_k^{n+1/3})\right)=\gamma_1\blambda_k^n &\mbox{ in } \Omega,\\
		\lambda_{k2}^{n+1/3}(x_1,0)=\lambda_{k2}^{n+1/3}(x_1,L_2),\ 0\leq x_1\leq L_1,\\
		\lambda_{k1}^{n+1/3}(0,x_2)=\lambda_{k1}^{n+1/3}(L_1,x_2),\ 0\leq x_2\leq L_2,\\
		k=1,2,3.
	\end{cases}
	\label{eq.frac1.pde.lambda2.periodic}
\end{equation}
Equations (\ref{eq.frac1.pde.lambda1.0}) and (\ref{eq.frac1.pde.lambda2.0}) can be modified analogously.

Finally, we replace (\ref{eq.frac3.var}) and (\ref{eq.frac3.pde}) by
\begin{align}
	\argmin_{\bv\in (\cH_P^1(\Omega))^{3}}\left[\frac{1}{2\tau}\int_{\Omega} |\nabla \bv-\bp^{n+2/3}|^2 d\bx+ \frac{1}{2\eta} \sum_{k=1}^3 \int_{\Omega} |v_k-f_k|^2 d\bx\right],
	\label{eq.frac3.var.periodic}
\end{align}
and
\begin{align}
	\begin{cases} 
		-\eta\nabla^2 u_k^{n+1} + \tau\frac{1}{\eta} u_k^{n+1}=-\eta\nabla\cdot \bp_k^{n+2/3}+\tau f_k \mbox{ in } \Omega,\\
		\frac{\partial u_k^{n+1}}{\partial x_2}(x_1,0)=\frac{\partial u_k^{n+1}}{\partial x_2}(x_1,L_2),\ 0\leq x_1\leq L_1,\\
		\frac{\partial u_k^{n+1}}{\partial x_1}(0,x_2)=\frac{\partial u_k^{n+1}}{\partial x_1}(L_1,x_2),\ 0\leq x_2\leq L_2,\\
		k=1,2,3,
	\end{cases}
	\label{eq.frac3.pde.periodic}
\end{align}
respectively.

Problems (\ref{eq.vq.periodic}), (\ref{eq.frac1.pde.lambda2.periodic}) and (\ref{eq.frac3.pde.periodic}) are linear elliptic problems which can be solved efficiently by FFT. In the rest of this article, periodic boundary conditions are used.

\section{Numerical discretization}
\label{sec.dis}
In this section, we numerically discretize the scheme (\ref{eq.split1.1})-(\ref{eq.split1.3}) and (\ref{eq.split2.1})--(\ref{eq.split2.3}) with periodic boundary conditions. Let our computational domain be $\Omega=[0,L_1]\times [0,L_2]$. We discretize $\Omega$ by $M\times N$ grids with step $h=L_1/M=L_2/N$. For any function $v$ defined on $\Omega$, we denote $v(ih,jh)$ by $v(i,j)$ for $1\leq i\leq M, \ 1\leq j\leq N$. We assume all functions satisfy the periodic boundary condition.

We first define several difference operators. For $1\leq i\leq M, \ 1\leq j\leq N$ and a scalar-valued function $v$, define the forward $(+)$ and backward $(-)$ difference by
\begin{align*}
	\partial_1^+ v(i,j)=(v(i+1,j)-v(i,j))/h,\ \partial_1^- v(i,j)=(v(i,j)-v(i-1,j))/h,\\
	\partial_2^+ v(i,j)=(v(i,j+1)-v(i,j))/h,\ \partial_2^- v(i,j)=(v(i,j)-v(i,j-1))/h,
\end{align*}
where $v(M+1,j)=v(1,j),\ v(0,j)=v(M,j), \ v(i,N+1)=v(i,1),\ v(i,0)=v(i,N)$ are used. With the notations above, the forward $(+)$ and backward $(-)$ gradient and divergence for scalar-valued function $v$ and vector-valued function $\bq$ are defined as
\begin{align*}
	&\nabla^{\pm} v(i,j)=(\partial^{\pm}_1 v(i,j), \partial^{\pm}_2 v(i,j)),\\
	&\diver^{\pm} \bq(i,j)= \partial_1^{\pm}q_1(i,j)+\partial_2^{\pm} q_2(i,j),\ \nabla^{\pm}\bq(i,j)=
	\begin{pmatrix}
		\partial^{\pm}_1 q_1(i,j) & \partial^{\pm}_2 q_1(i,j) \\
		\partial^{\pm}_1 q_2(i,j) & \partial^{\pm}_2 q_2(i,j)
	\end{pmatrix}.
\end{align*}

Define the shifting operator and identity operator by
\begin{align*}
	\cS_1^{\pm} v(i,j)=v(i\pm1,j), \ \cS_2^{\pm} v(i,j)=v(i,j\pm1), \ \cI v(i,j)=v(i,j).
\end{align*}
Let $\cF$ and $\cF^{-1}$ denote the discrete Fourier transform and the inverse transform. We have
\begin{align*}
	&\cF(\cS_1^{\pm} v)(i,j)=(\cos z_i\pm \sqrt{-1} \sin z_i) \cF(v)(i,j),\\
	&\cF(\cS_2^{\pm} v)(i,j)=(\cos z_j\pm \sqrt{-1} \sin z_j) \cF(v)(i,j),
\end{align*}
where 
\begin{align}
	z_i=\frac{2\pi}{M}(i-1), \ z_j=\frac{2\pi}{N}(j-1)
	\label{eq.z}
\end{align}
for $i=1,...,M$ and $j=1,...,N$.

\subsection{Computing the discrete analogue of $\bp^{n+1/3},\blambda^{n+1/3}$ in (\ref{eq.split1.1}) and (\ref{eq.split2.1})}
\label{sec.split1.1.dis}
In (\ref{eq.split1.1}), $\bp^{n+1/3}$ can be computed pixel-wisely using (\ref{eq.freezq}). For $\blambda^{n+1/3}$, we solve (\ref{eq.frac1.pde.lambda1.periodic}) whose discrete analogue is 
\begin{equation}
	\gamma_1\blambda_k^{n+1/3}-2\beta\tau\nabla^+\left( \sqrt{g^{n+1/3}}(\diver^- \blambda_k^{n+1/3})\right)=\gamma_1\blambda_k^n 
	\label{eq.frac1.pde.lambda1.periodic.dis}
\end{equation}
for $k=1,2,3$.
Instead of solving (\ref{eq.frac1.pde.lambda1.periodic.dis}), we use the frozen coefficient approach \cite{tai2011fast,he2020curvature} to solve
\begin{equation}
	\gamma_1\blambda_k^{n+1/3}-c_1\nabla^+(\diver^- \blambda_k^{n+1/3}) =\gamma_1\blambda_k^n +\nabla^+\left( 2\beta\tau\left(\sqrt{g^{n+1/3}}-c_1\right)(\diver^- \blambda_k^{n})\right),
	\label{eq.frac1.pde.lambda1.periodic.dis.1}
\end{equation}
where $c_1>0$ is some constant. Problem (\ref{eq.frac1.pde.lambda1.periodic.dis.1}) can be solved efficiently by FFT. We first rewrite (\ref{eq.frac1.pde.lambda1.periodic.dis.1}) in matrix form as
\begin{align}
	\begin{pmatrix}
		\gamma_1- c_1 \partial_1^+\partial_1^-&-c_1\partial_1^+\partial_2^- \\
		-c_1\partial_2^+\partial_1^- & \gamma_1-c_1 \partial_2^+\partial_2^-
	\end{pmatrix}
	\begin{pmatrix}
		\lambda_{k1}^{n+1/3}\\ \lambda_{k2}^{n+1/3}
	\end{pmatrix}=
	\begin{pmatrix}
		w_1\\w_2
	\end{pmatrix}
	\label{eq.frac1.pde.lambda1.periodic.dis.2}
\end{align}
with 
\begin{align*}
	&w_1=\gamma_1\lambda_{k1}^n+\partial_1^+\left(2\beta\tau\left(\sqrt{g^{n+1/3}}-c_1\right) \left(\partial_1^-\lambda_{k1}^n+\partial_2^- \lambda_{k2}^n\right)\right),\\
	&w_2=\gamma_1\lambda_{k2}^n+\partial_2^+\left(2\beta\tau\left(\sqrt{g^{n+1/3}}-c_1\right) \left(\partial_1^-\lambda_{k1}^n+\partial_2^- \lambda_{k2}^n\right)\right).
\end{align*}
Problem (\ref{eq.frac1.pde.lambda1.periodic.dis.2}) is equivalent to
\begin{align}
	\begin{pmatrix}
		\gamma_1- c_1 (\cS_1^+-\cI)(\cI-\cS_1^-)/h^2&-c_1(\cS_1^+-\cI)(\cI-\cS_2^-)/h^2 \\
		-c_1(\cS_2^+-\cI)(\cI-\cS_1^-)/h^2 & \gamma_1-c_1 (\cS_2^+-\cI)(\cI-\cS_2^-)/h^2
	\end{pmatrix}
	\begin{pmatrix}
		\cF(\lambda_{k1}^{n+1/3})\\ \cF(\lambda_{k2}^{n+1/3})
	\end{pmatrix}=
	\begin{pmatrix}
		\cF(w_1)\\\cF(w_2)
	\end{pmatrix}.
	\label{eq.frac1.pde.lambda1.periodic.dis.3}
\end{align}
Taking discrete Fourier transform on both sides of (\ref{eq.frac1.pde.lambda1.periodic.dis.3}) gives rise to
\begin{align}
	\begin{pmatrix}
		a_{11}&a_{12} \\
		a_{21} & a_{22}
	\end{pmatrix}
	\begin{pmatrix}
		\cF(\lambda_{k1}^{n+1/3})\\ \cF(\lambda_{k2}^{n+1/3})
	\end{pmatrix}=
	\begin{pmatrix}
		\cF(w_1)\\\cF(w_2)
	\end{pmatrix}
	\label{eq.frac1.pde.lambda1.periodic.dis.4}
\end{align}
with
\begin{align*}
	&a_{11}=\gamma_1-2c_1(\cos z_i-1)/h^2,\ a_{22}=\gamma_1-2c_1(\cos z_j-1)/h^2,\\
	&a_{12}=-c_1(1-\cos z_i-\sqrt{-1}\sin z_i)(1-\cos z_j+\sqrt{-1}\sin z_j)/h^2,\\
	&a_{21}=-c_1(1-\cos z_j-\sqrt{-1}\sin z_j)(1-\cos z_i+\sqrt{-1}\sin z_i)/h^2,
\end{align*}
and $z_i,z_j$ defined in (\ref{eq.z}). We have
\begin{align}
	\begin{pmatrix}
		\lambda_{k1}^{n+1/3}\\ \lambda_{k2}^{n+1/3}
	\end{pmatrix}=\Real\left(\cF^{-1}\left[\frac{1}{a_{11}a_{22}-a_{12}a_{21}}
	\begin{pmatrix}
		a_{22}\cF(w_1)-a_{12}\cF(w_2)\\
		-a_{21}\cF(w_1)+a_{22}\cF(w_2)
	\end{pmatrix} \right]\right).
	\label{eq.frac1.pde.lambda1.periodic.dis.5}
\end{align}

In (\ref{eq.split2.1}), $\bp^{n+1/3}$ can be computed pixel-wisely using (\ref{eq.freezq2}). The computation of the discrete analogue of $\blambda^{n+1/3}$ can be conducted similarly to that of $\blambda^{n+1/3}$ in (\ref{eq.split1.1}) as is discussed above. One can use (\ref{eq.frac1.pde.lambda1.periodic.dis.5}) except replacing $w_1,w_2$ by
\begin{align*}
	&w_1=\gamma_1\lambda_{k1}^n+\partial_1^+\left(2\beta\tau\left(\sqrt{g^{n+1/3}-\alpha^2}-c_1\right) \left(\partial_1^-\lambda_{k1}^n+\partial_2^- \lambda_{k2}^n\right)\right),\\
	&w_2=\gamma_1\lambda_{k2}^n+\partial_2^+\left(2\beta\tau\left(\sqrt{g^{n+1/3}-\alpha^2}-c_1\right) \left(\partial_1^-\lambda_{k1}^n+\partial_2^- \lambda_{k2}^n\right)\right).
\end{align*}
\subsection{Computing the discrete analogue of $\bp^{n+2/3},\blambda^{n+2/3}$ in (\ref{eq.split1.2}) and (\ref{eq.split2.2})}
The computation of the discrete analogues of $\bp^{n+2/3},\blambda^{n+2/3}$ in (\ref{eq.split1.2}) and (\ref{eq.split2.2}) can be conducted pixel-wisely using Theorem \ref{thm.frac2} and \ref{thm.frac2.2}, respectively.

\subsection{Computing the discrete analogue of $\bp^{n+1}$ in (\ref{eq.split1.3}) and (\ref{eq.split2.3})}
For $\bp^{n+1}$ in (\ref{eq.split1.3}) and (\ref{eq.split2.3}), we solve (\ref{eq.frac3.pde.periodic}) whose discrete analogue is 
\begin{align}
	\begin{cases} 
		-\eta\diver^-(\nabla^+ u_k^{n+1}) + \tau u_k^{n+1}=-\eta\diver^- \bp_k^{n+2/3}+\tau f_k \mbox{ in } \Omega,\\
		k=1,2,3.
	\end{cases}
	\label{eq.frac3.pde.periodic.dis}
\end{align}
Problem (\ref{eq.frac3.pde.periodic.dis}) is equivalent to
\begin{align}
	\begin{cases} 
		\left(\tau h^2\cI-\eta (\cI-\cS_1^-)(\cS_1^+-\cI)-\eta (\cI-\cS_2^-)(\cS_2^+-\cI) \right) u_k^{n+1}= g_k,\\
		k=1,2,3,
	\end{cases}
	\label{eq.frac3.pde.periodic.dis1}
\end{align}
where $g_k=-\eta h^2\diver^- \bp_k^{n+2/3}+\tau h^2f_k$. Taking the discrete Fourier transform on both sides, we get
$$
b\cF(u_k^{n+1})=\cF(g_k)
$$
with $b=\tau h^2+4\eta-2\eta\cos z_i-2\eta \cos z_j,$ where $z_i,z_j$ are defined in (\ref{eq.z}).
We compute 
\begin{align}
	u_k^{n+1}=\Real\left(\cF^{-1}\left( \frac{\cF(g)}{b}\right)\right)
\end{align}
and $\bp^{n+1}=\nabla^+ \bu$.

\begin{figure}[t!]
	\centering
	\includegraphics[width=0.25\textwidth]{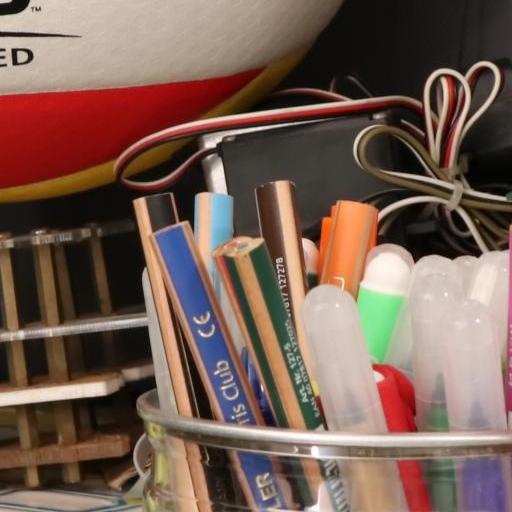}\,
	\includegraphics[width=0.25\textwidth]{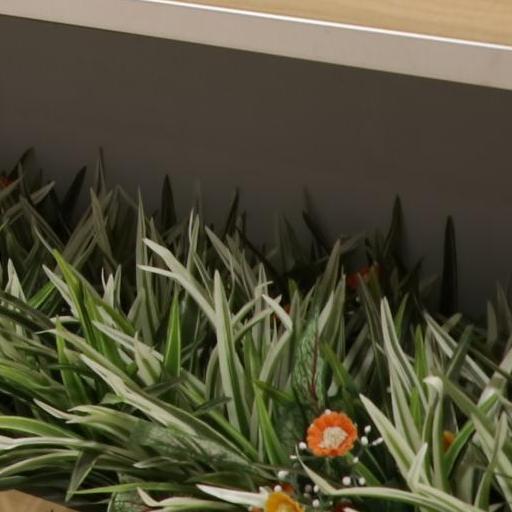}\,
	\includegraphics[width=0.25\textwidth]{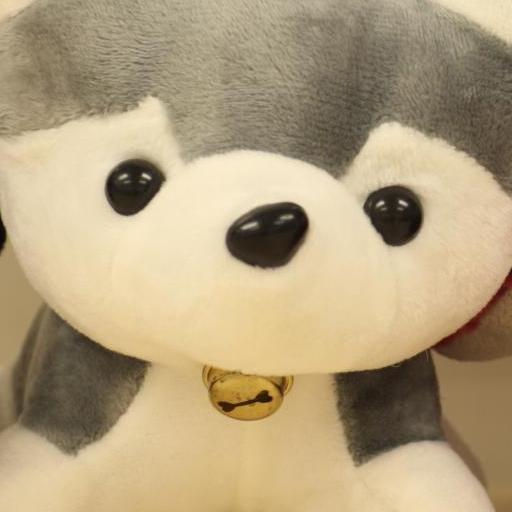}
	\caption{Empirical justification of proposed models on
		three images sampled from \cite{xu2018real}. 
		From left to right, `Pens', `Plant' and `Toy'.}
	\label{fig.justify}
\end{figure}

\begin{figure}[t!]
	\centering
	\begin{tabular}{ccc}
		(a) & (b) & (c)\\
		\includegraphics[width=0.3\textwidth]{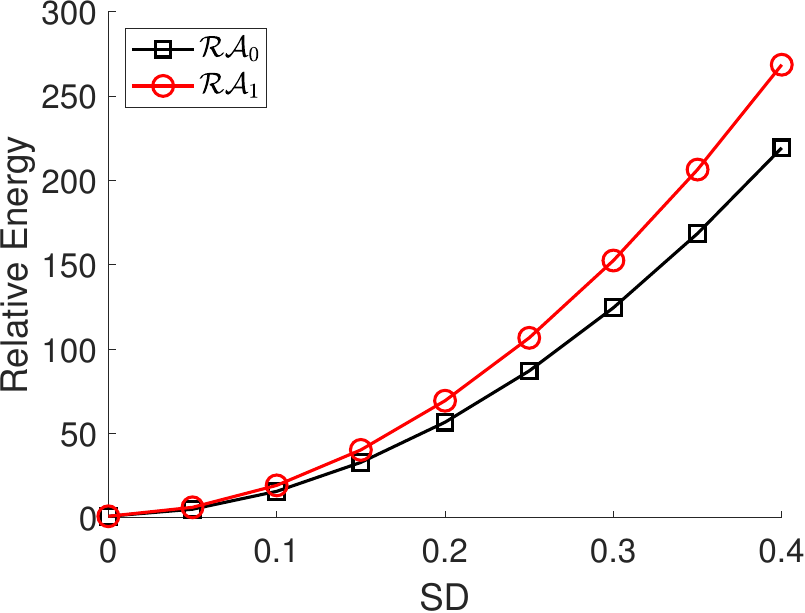}&
		\includegraphics[width=0.3\textwidth]{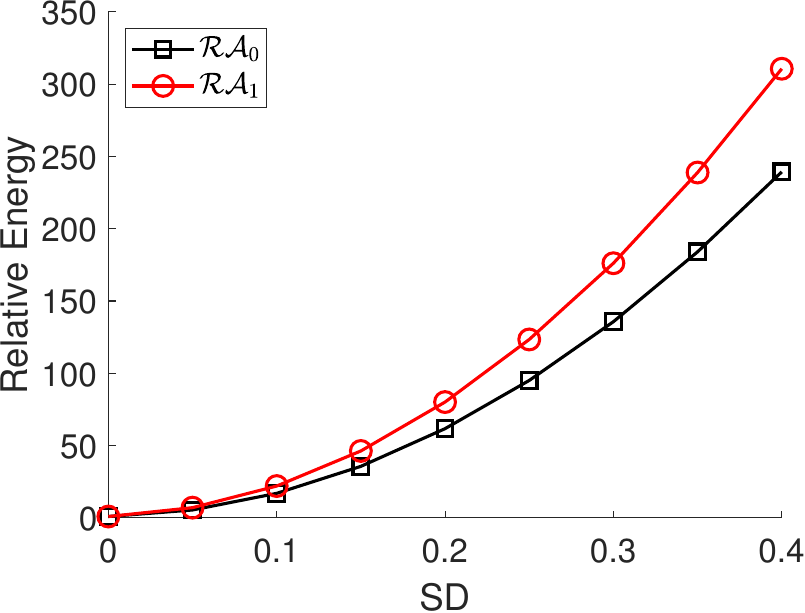}&
		\includegraphics[width=0.3\textwidth]{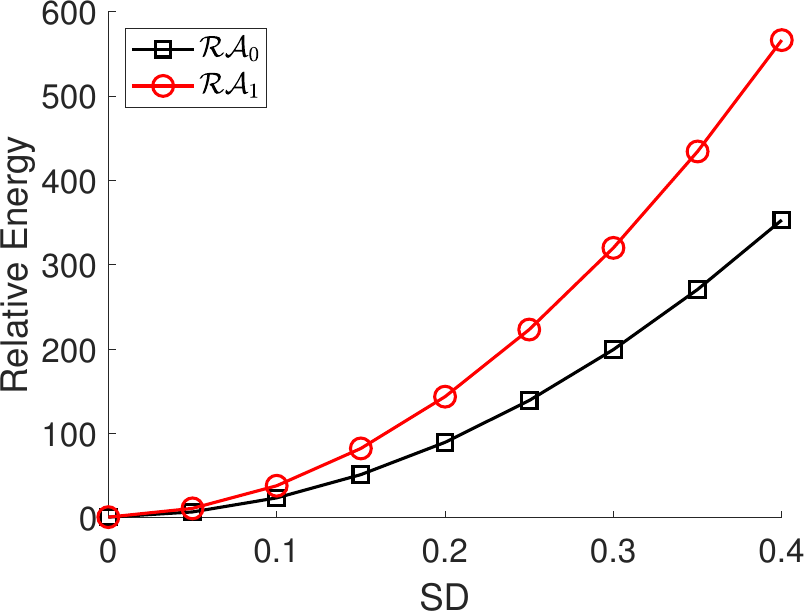}\\
		\includegraphics[width=0.3\textwidth]{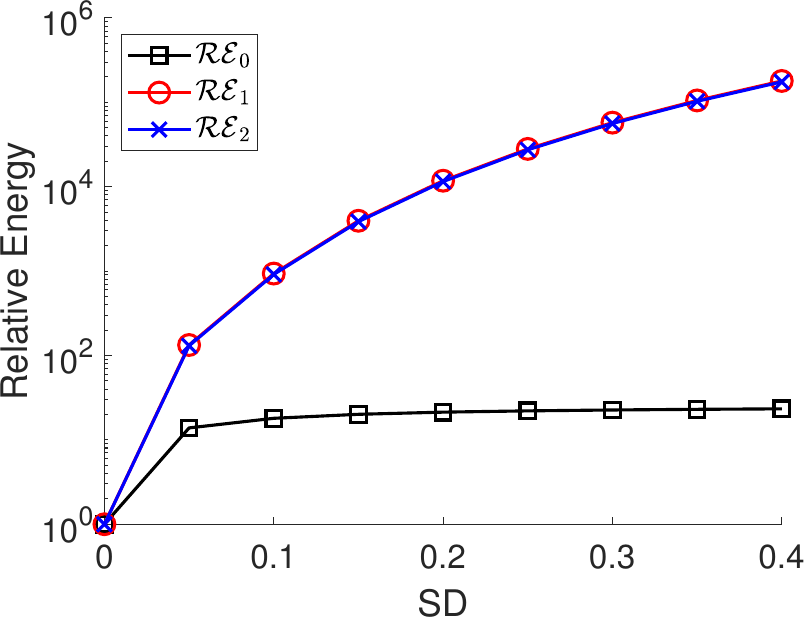}&
		\includegraphics[width=0.3\textwidth]{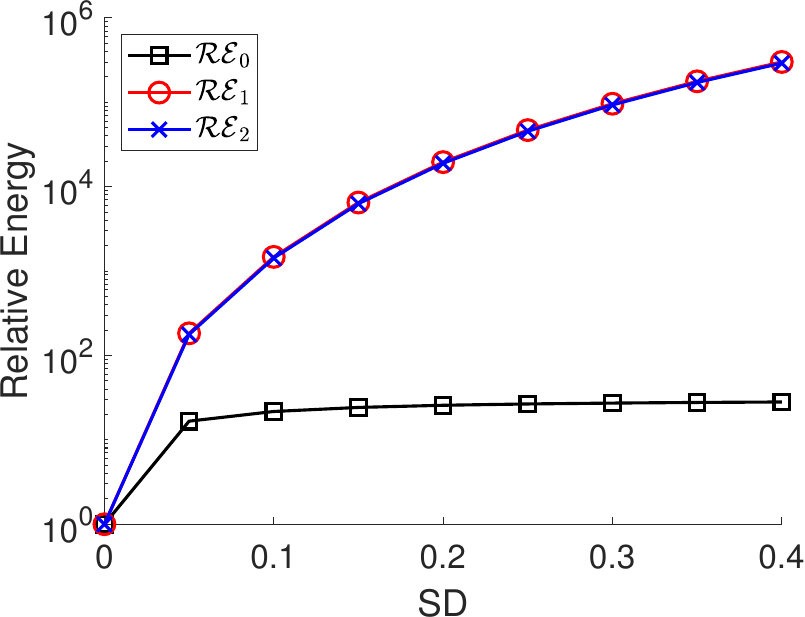}&
		\includegraphics[width=0.3\textwidth]{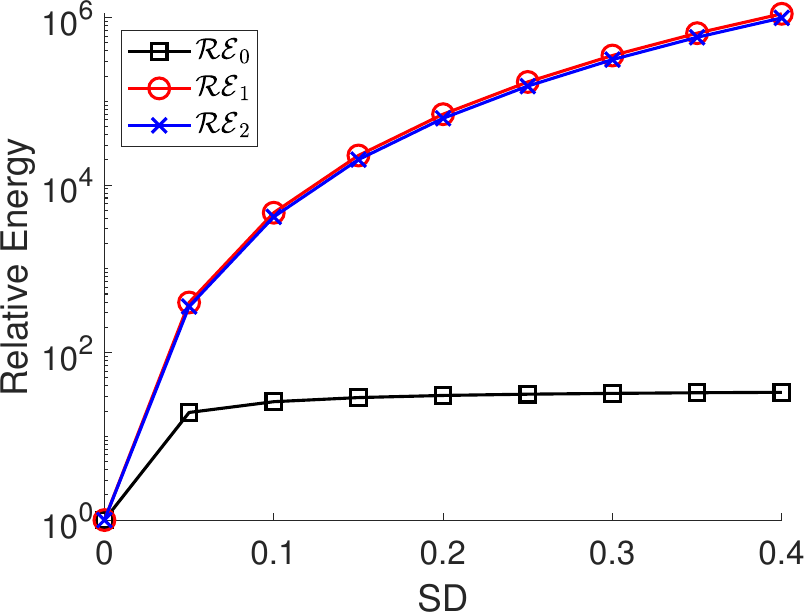}
	\end{tabular}
	\caption{Comparison of the relative energy of the surface area (top) and the elastica (bottom) as functions of the noise standard deviation (SD).  
		Noisy images are generated by adding Gaussian noise with SD varying from 0 to 0.4. 
		(a)--(c) correspond to the images `Pens', `Plant' and `Toy' in Figure \ref{fig.justify}, respectively.  
		Each relative energy is averaged over 10 experiments.}
	\label{fig.RelativeEnergy}
\end{figure}

\begin{figure}[t!]
	\begin{tabular}{ccc}
		(a) & (b) & (c)\\
		\includegraphics[width=0.3\textwidth]{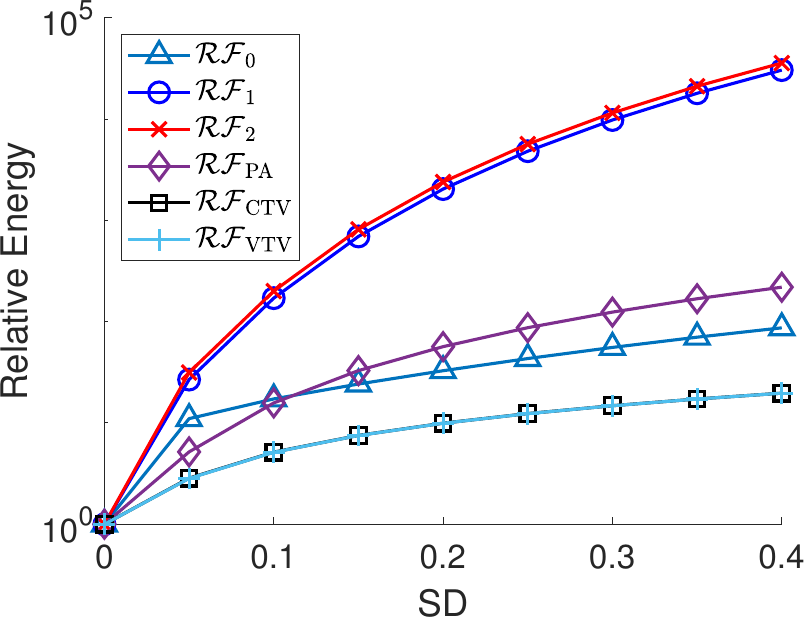}&
		\includegraphics[width=0.3\textwidth]{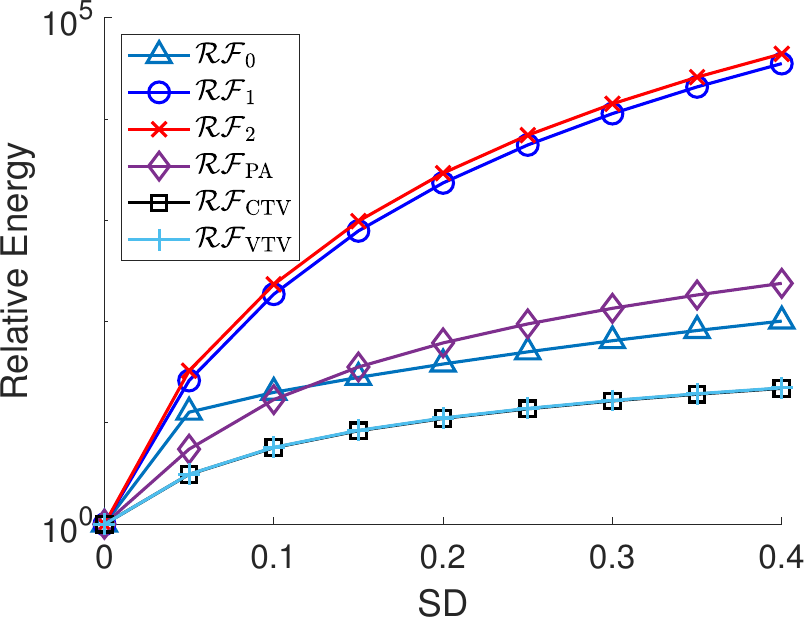}&
		\includegraphics[width=0.3\textwidth]{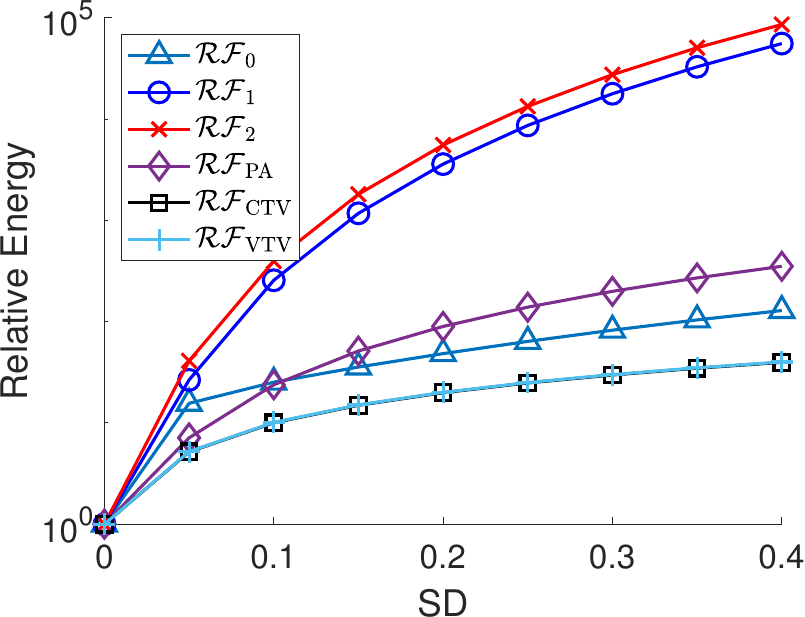}
	\end{tabular}
	\caption{Comparison of the relative energy of different regularizers. The noisy images are generated by adding Gaussian noise with SD varying from 0 to 0.4. (a)--(c) correspond to `Pens', `Plant' and `Toy' images in Figure \ref{fig.justify}. 
		The tested regularizers are $\cF_0$ (in CE), $\cF_1$ (in model (\ref{eq.model1})), $\cF_2$ (in model (\ref{eq.model2})), $\cF_{\rm PA}$ (in PA), $\cF_{\rm CTV}$ (in CTV) and $\cF_{\rm VTV}$ (in VTV). 
		Each relative energy is averaged over 10 experiments.}
	\label{fig.RelativeEnergy1}
\end{figure}


\begin{figure}[t!]
	\begin{tabular}{ccc}
		(a) & (b) & (c)\\
		\includegraphics[width=0.3\textwidth]{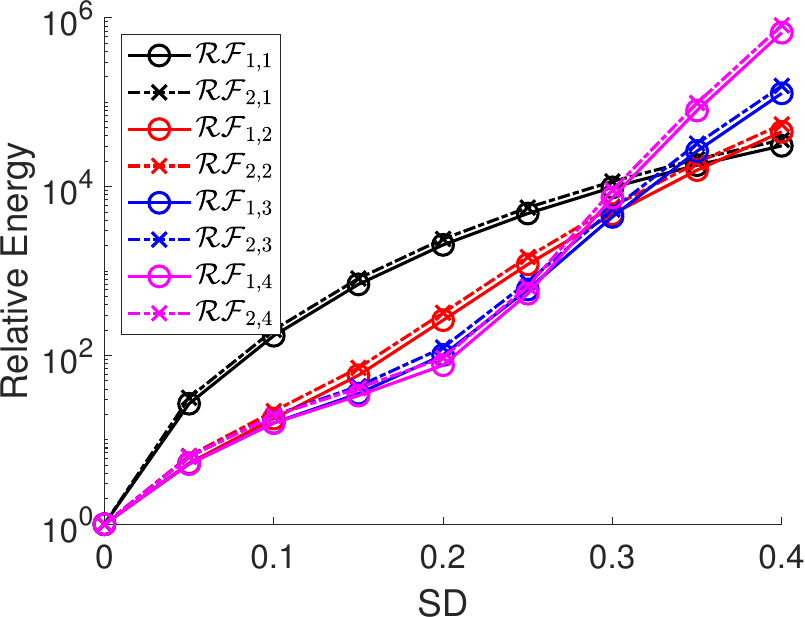} &
		\includegraphics[width=0.3\textwidth]{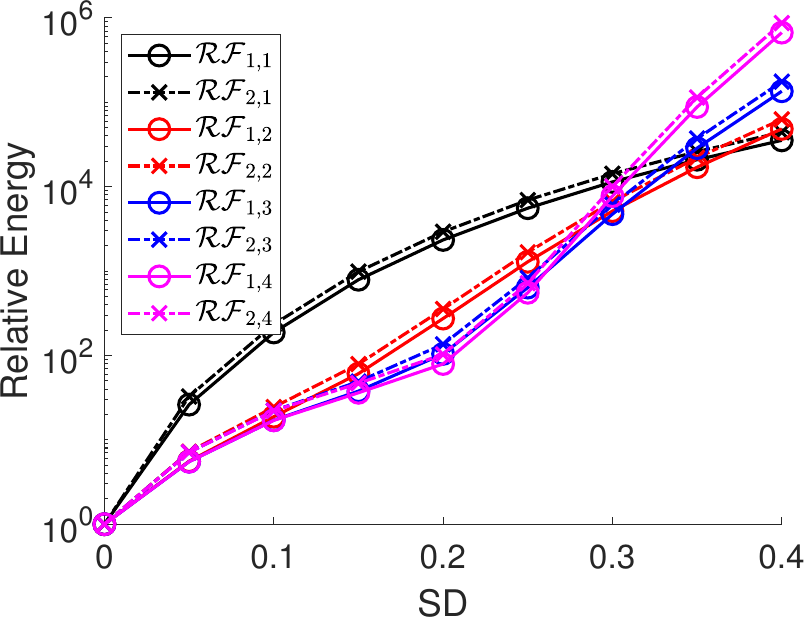} &
		\includegraphics[width=0.3\textwidth]{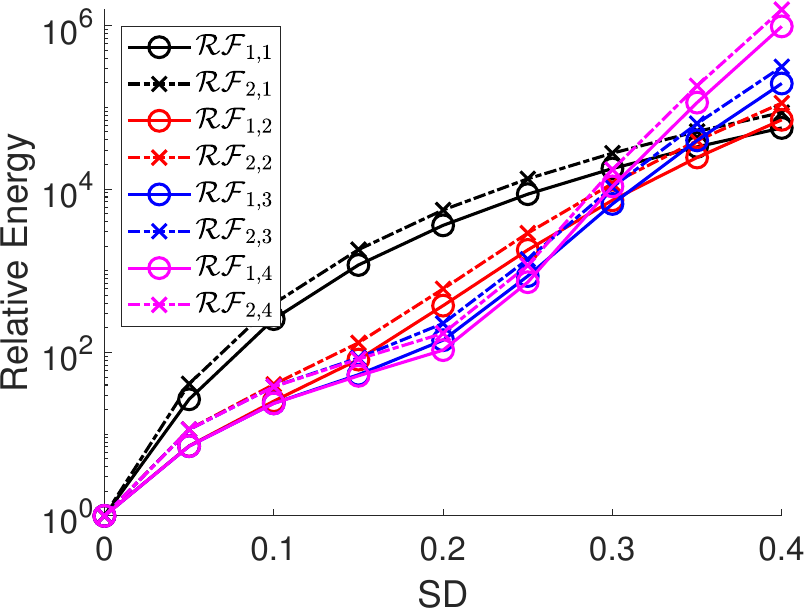} 
	\end{tabular}	
	\caption{Comparison of the relative energy of variants of model (\ref{eq.model1}) and (\ref{eq.model2}). 
		The noisy images are generated by adding Gaussian noise with SD varying from 0 to 0.4. (a)--(c) correspond to images `Pens', `Plant' and `Toy' in Figure \ref{fig.justify}.  
		Each relative energy is averaged over 10 experiments.}
	\label{fig.RelativeEnergyN}
\end{figure}

\section{Empirical justifications of the new models}
\label{sec.justification}
In this section, we empirically justify the new terms we introduced to the modified color elastica models. 
In model (\ref{eq.model1}) and (\ref{eq.model2}), three new terms are considered.
A new surface area measure  $\cA_1$ and two new color elastica terms $\cE_1, \cE_2$, which are defined by
\begin{eqnarray*}
	\cA_1(\bff)&=&\int_{\Omega} \sqrt{g-\alpha^2} d \bx, \cr
	\cE_1(\bff)&=&\int_{\Omega}\left(\sum_{k=1}^3 g|\Delta_g f_k|^2\right) \sqrt{g} d\bx, \cr \cE_2(\bff)&=&\int_{\Omega}\left(\sum_{k=1}^3 (g-\alpha^2)|\widetilde{\Delta}_g f_k|^2\right) \sqrt{g-\alpha^2} d\bx.
\end{eqnarray*}
The term $\cE_1$ is used in model (\ref{eq.model1}). 
The terms $\cA_1$ and $\cE_2$ are used in model (\ref{eq.model2}). 
We  denote the original surface area and elastica term in the color elastica model (\ref{eq.model.old}) by
\begin{align}
	\cA_0(f)\,=\,\int_{\Omega} \sqrt{g}d \bx,\,\,\,\,\,\,\,\, \cE_0(\bff)\,=\,\int_{\Omega}\left(\sum_{k=1}^3 |\Delta_g f_k|^2\right) \sqrt{g} d\bx.
\end{align}
Given a clean image $f_0$ and its noisy version $f$, we measure the effectiveness of these terms by computing the ratio between their values at $f$ and $f_0$. 
The quantity is considered to be more sensitive if the ratio increases more rapidly as the noise level increases. 
Specifically, we define the relative energy for surface areas and elastica terms as
\begin{eqnarray*}
	\cR\cA_i(f)&=&\frac{\cA_i(f)}{\cA_i(f_0)},\cr \cR\cE_j(f)&=&\frac{\cE_i(f)}{\cE_i(f_0)}
\end{eqnarray*}
for $i=0,1$ and $j=0,1,2$. 
For each term, we monitor the relation between its relative energy as a function of the noise level injected to the image.

In this set of experiments, we set $\alpha=10^{-3}$ and consider three images in Figure \ref{fig.justify}: `Pens', `Plant' and `Toy'. 
The noisy images are generated by adding Gaussian noise with standard deviation (variance) varying from 0 to 0.4. 
The comparison of the two relative energies of all measures are shown in Figure \ref{fig.RelativeEnergy}. 
Each relative energy is averaged over 10 experiments. 
The first row shows the comparison for variations on surface areas. 
The newly introduced term $\cA_1$ appears to be more sensitive than $\cA_0$ as its relative energy increases faster as a function of noise level for all three images. 
The second row shows comparisons for the elastica terms. 
The newly introduced measures $\cE_1$ and $\cE_2$ are similar. 
Both are more sensitive than $\cE_0$. 

We then compare the effectiveness of the regularizers in the proposed models and other models for color image regularization, including the color elastica (CE) model \cite{liu2021color}, the Polyakov action (PA) model \cite{kimmel1997high}, the color total variation (CTV) model \cite{blomgren1998color} and the vectorial total variation (VTV) model \cite{goldluecke2012natural}. 
We denote the regularizers in CE, Model (\ref{eq.model1}), Model (\ref{eq.model2}), PA, CTV and VTV by
\begin{eqnarray*}
	\cF_0(\bbf)&=&\cA_0(\bbf)+\beta\cE_0(\bbf), \cr \cF_1(\bbf)&=&\cA_0(\bbf)+\beta\cE_1(\bbf), \cr \cF_2(\bbf)&=&\cA_1(\bbf)+\beta\cE_2(\bbf), \cr 
	\cF_{\rm PA}(\bbf)&=&\cA_0(\bbf), \cr
	\cF_{\rm CTV}(\bbf)&=&\int_{\Omega} \sqrt{\sum_{k=1}^3 |\nabla f_k|^2}d\bx\cr \cF_{\rm VTV}(\bbf)&=& \int_{\Omega} \sigma_1(\nabla \bbf)d\bx,
\end{eqnarray*}
respectively, where in $\cF_{\rm VTV}(\bbf)$, $\sigma_1(\nabla \bbf)$ denotes the largest singular value of the Jacobian matrix $\nabla \bbf$. 
For any energy $\cF$ defined above, we define its relative energy as
\begin{eqnarray*}
	\cR\cF(\bbf)&=&\frac{\cF(\bbf)}{\cF(\bbf_0)},
\end{eqnarray*}
where $\bbf_0$ is the clean image and $\bbf$ is its noisy version. 

Figure \ref{fig.RelativeEnergy1} displays the relative energies as a function of the noise levels for all regularizers. 
In this set of experiments, the three images presented in Figure \ref{fig.justify}, `Pens', `Plant' and `Toy', are used. 
The noisy images are generated by adding Gaussian noise with SD varying from $0$ to $0.4$. 
For $\cF_0,\cF_1,\cF_2$, we set $\alpha=10^{-3}$. 
We use $\beta=10^{-2}$ in $\cF_0$ and $\beta=30$ in $\cF_{1}$ and $\cF_2$. 
Each relative energy is averaged over 10 experiments. 
In Figure \ref{fig.RelativeEnergy1}, the relative energy of CTV and VTV are  close to each other. 
That is because both regularizers only depend on the Jacobian matrix $\nabla \bbf$. 
In this comparison, the relative energy of model (\ref{eq.model1}) and (\ref{eq.model2}), that is, $\cF_1$ and $\cF_2$, are the most effective as they increase fastest with the noise level. 
This comparison empirically justifies the fact that the new terms in model (\ref{eq.model1}) and (\ref{eq.model2}) are better suited for modeling and denoising natural images.

In model (\ref{eq.model1}) and (\ref{eq.model2}), the elastica term is weighted by $g$ or $g-\alpha^2$. 
In the next test, we compare the relative energy of variants of model (\ref{eq.model1}) and (\ref{eq.model2}) in which the elastica term is weighted by different powers of $g$ or $g-\alpha^2$. 
Specifically, denote
\begin{eqnarray*}
	\cF_{1,m}(\bbf)&=&\cA_0(\bbf)+\beta \int_{\Omega}\left(\sum_{k=1}^3 g^m|\Delta_g f_k|^2\right) \sqrt{g} d\bx,\cr
	\cF_{2,m}(\bbf)&=&\cA_1(\bbf)+\beta \int_{\Omega}\left(\sum_{k=1}^3 (g-\alpha^2)^m|\widetilde{\Delta}_g f_k|^2\right) \sqrt{g-\alpha^2} d\bx.
\end{eqnarray*}
We compared in Figure \ref{fig.RelativeEnergyN} the relative energies $\cR\cF_{1,m}$ and $\cR\cF_{2,m}$ for $m=1,2,3,4$. 
In this experiment, we set $\alpha=10^{-3}$ and $\beta=30$. 
We observe that model (\ref{eq.model1}) and (\ref{eq.model2}), corresponding to $m=1$, are most effective for all three test images when $\rm{SD}\leq 0.3$. While model (\ref{eq.model1}) and (\ref{eq.model2}) are not the best when $\rm{SD}>0.3$, they are still most effective for practical consideration, since  usually very large noise is rare in practice. 


\section{Numerical experiments}
\label{sec.experiments}
We demonstrate the effectiveness of the propose models and the efficiency of the proposed alogrithms in this section. All experiments are implemented by MATLAB (R2020b) on a Windows desktop with 16GB RAM and Intel(R) Core(TM) i7-10700 CPU: 2.90GHz. In our experiments, images with pixel values in $[0,1]$ are used. We set $h=1,\gamma_1=1, \gamma_2=3,\xi_1=10^{-5}$ and $\epsilon=10^{-3}$, where $\xi_1$ is the stopping criterion of the iterative method for $\bp^{n+1/3}$ and $\epsilon$ is the small constant in (\ref{eq.freezq2}). For both algorithms, we use the stopping criterion  
$$
\frac{\|u^{n+1}-u^n\|}{\|u^n\|}\,\leq\, \zeta,
$$
for some small $\zeta>0$, where $\|\cdot\|$ denotes the Frobenius norm. 

There are three model parameters: $\alpha,\beta$ and $\eta$. The parameter $\beta$ controls the weight of the elastica term and $\eta$ controls the weight of fidelity term. One can get a smoother result if larger $\beta$ and $\eta$ are used. The parameter $\alpha$ controls the weight of the spatial coordinates in the manifold parametrization, see (\ref{eq.G}). As information from chromatic coordinates (feature coordinates) is usually more important than that from spatial coordinates, a small $\alpha$ should be used. Without specification, we set $\alpha=5\times 10^{-4}$ for model (\ref{eq.model1}), $\alpha=3\times 10^{-2}$ for model (\ref{eq.model2}) and $\zeta=10^{-5}$ for both algorithms.

\begin{remark}
	There are three algorithm parameters: $\gamma_1,\gamma_2$ and $\tau$, which control the evolution speed of all auxiliary variables. The performance of our algorithms is not sensitive to these parameters. In our experiments, our algorithms converge as long as these parameters are small enough. Actually, we do not need to set these parameters too small. Setting $\gamma_1=1,\gamma_2=1$ and $\tau=0.05$ already provides good results.
\end{remark}


\begin{figure}[t!]
	\centering
	\begin{tabular}{cccc}
		(a) & (b) &(c) & (d)\\
		\includegraphics[width=0.22\textwidth]{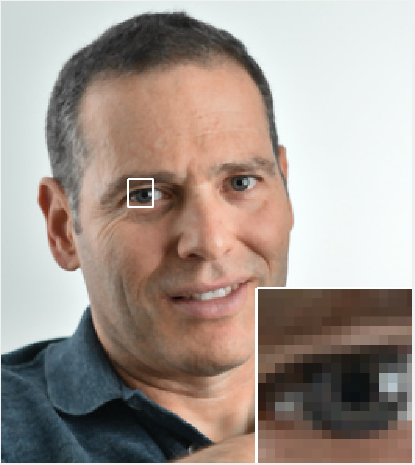}&
		\includegraphics[width=0.22\textwidth]{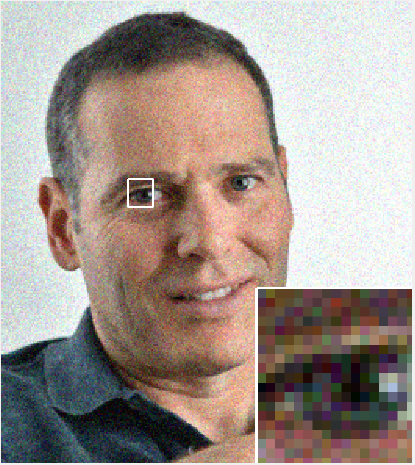}&
		\includegraphics[width=0.22\textwidth]{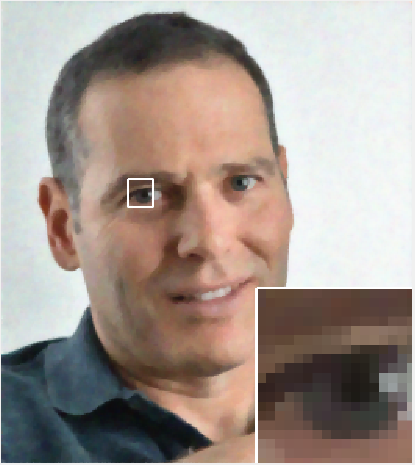}&
		\includegraphics[width=0.22\textwidth]{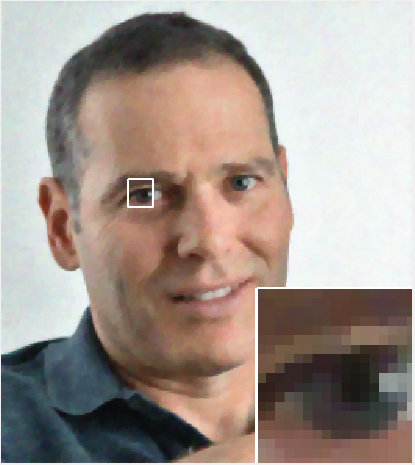}\\
		\includegraphics[width=0.22\textwidth]{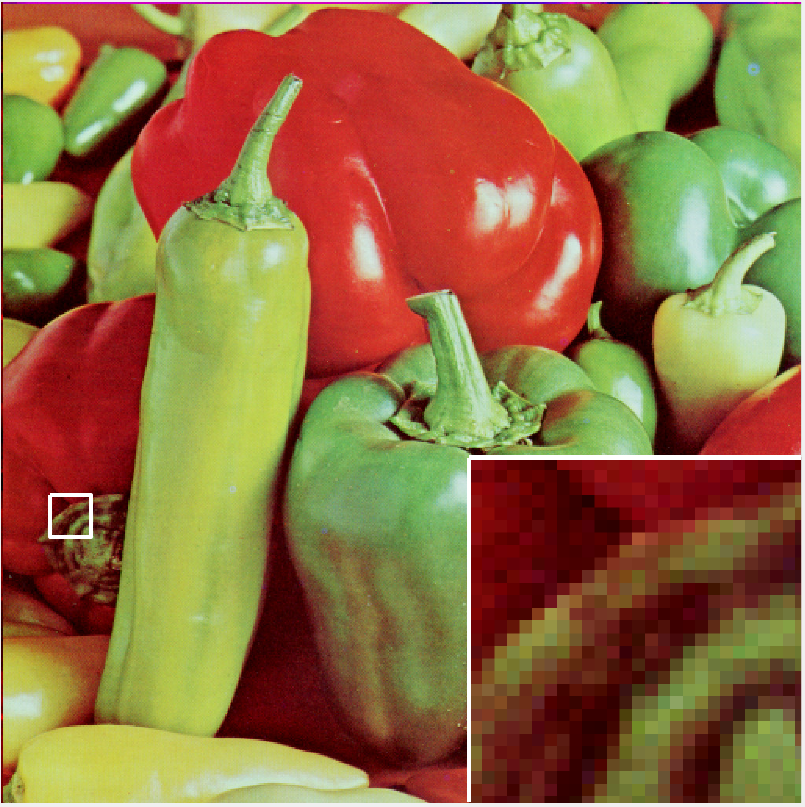}& 
		\includegraphics[width=0.22\textwidth]{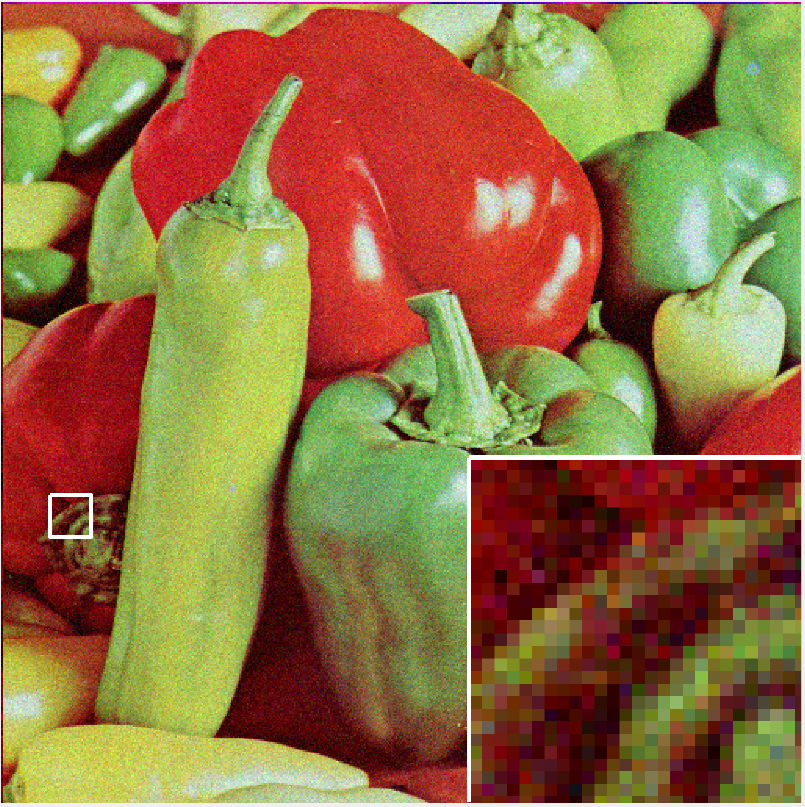}&
		\includegraphics[width=0.22\textwidth]{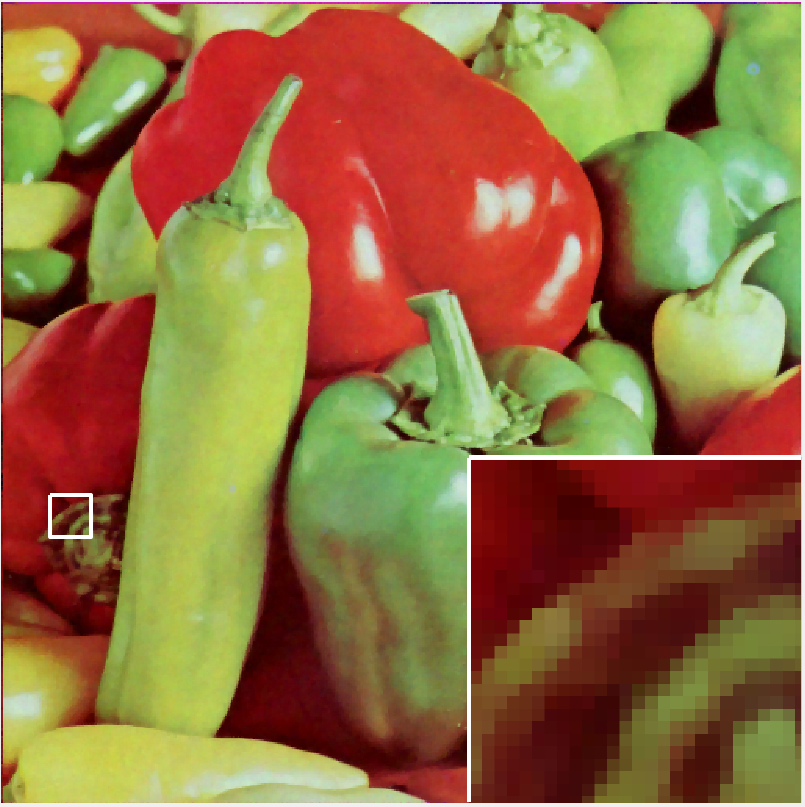}&
		\includegraphics[width=0.22\textwidth]{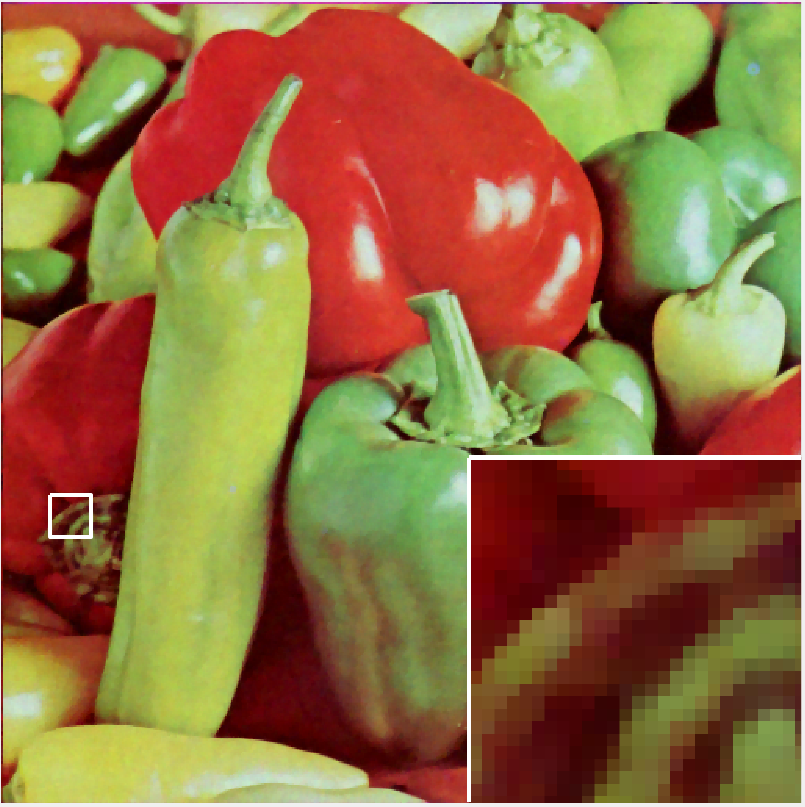}\\
		\includegraphics[width=0.22\textwidth]{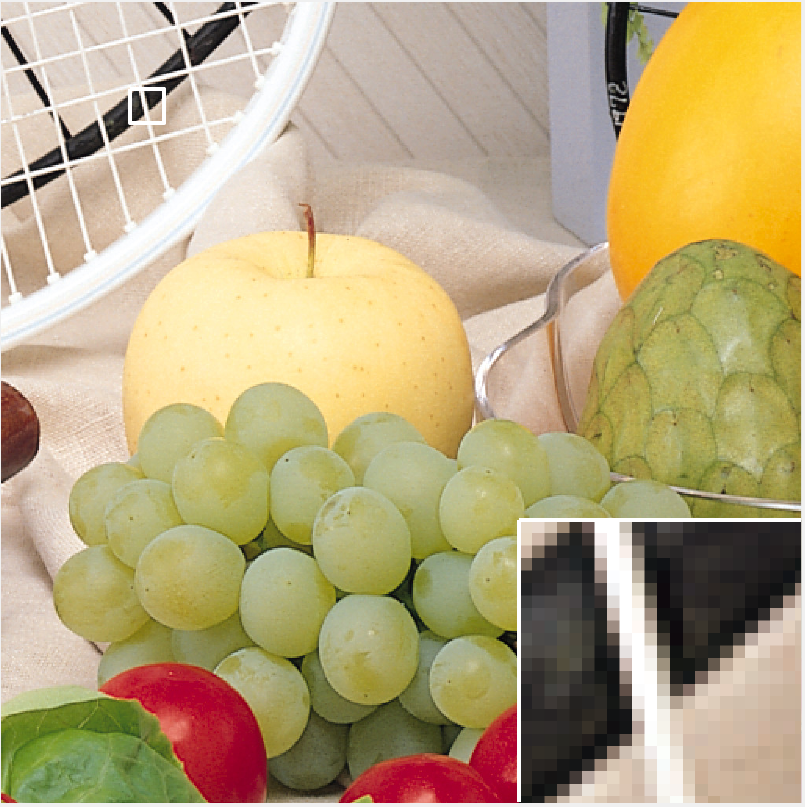} &
		\includegraphics[width=0.22\textwidth]{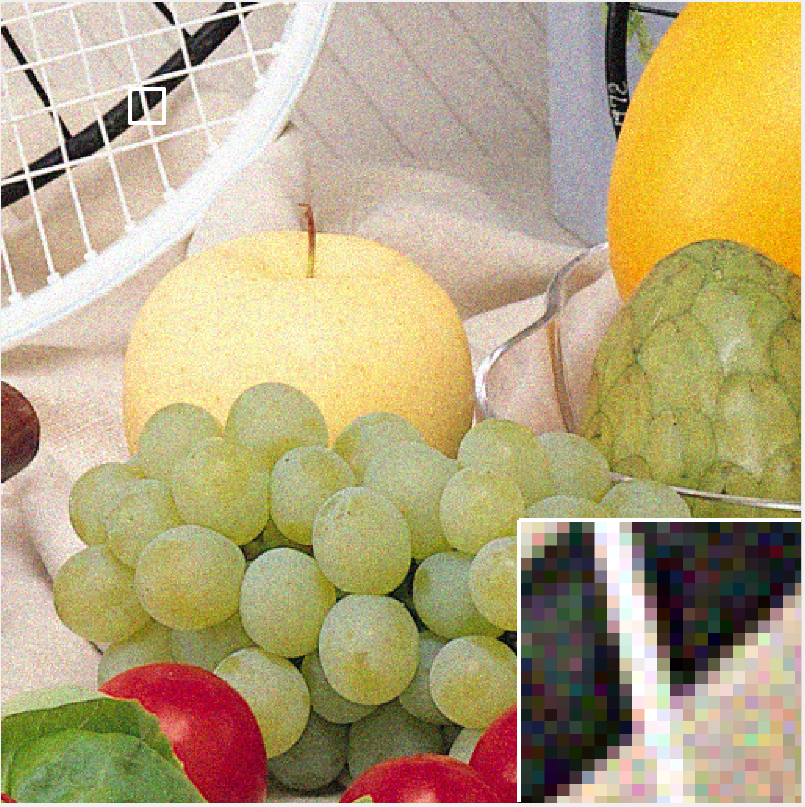}&
		\includegraphics[width=0.22\textwidth]{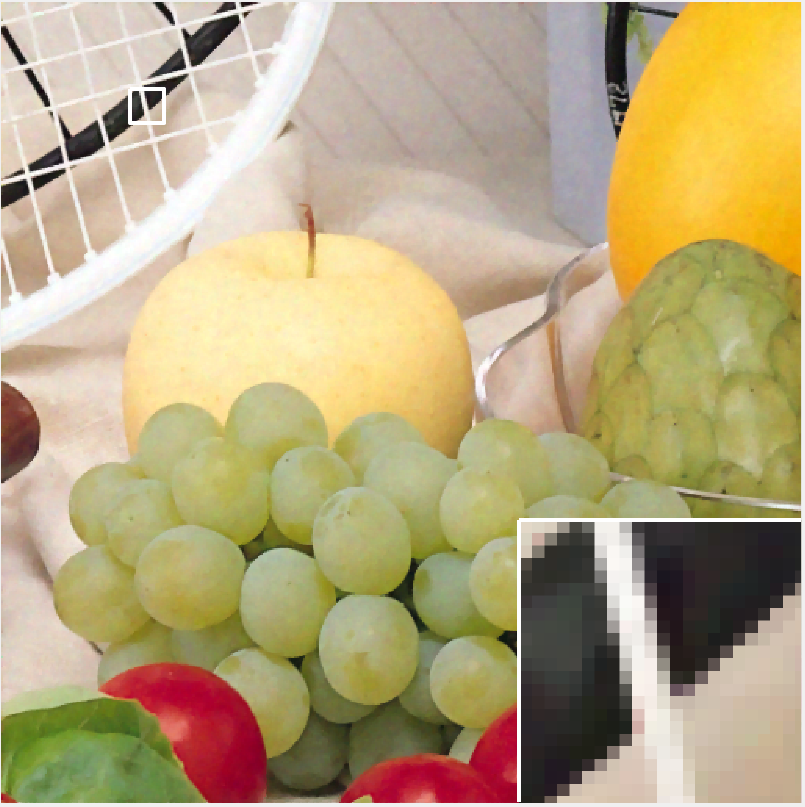}&
		\includegraphics[width=0.22\textwidth]{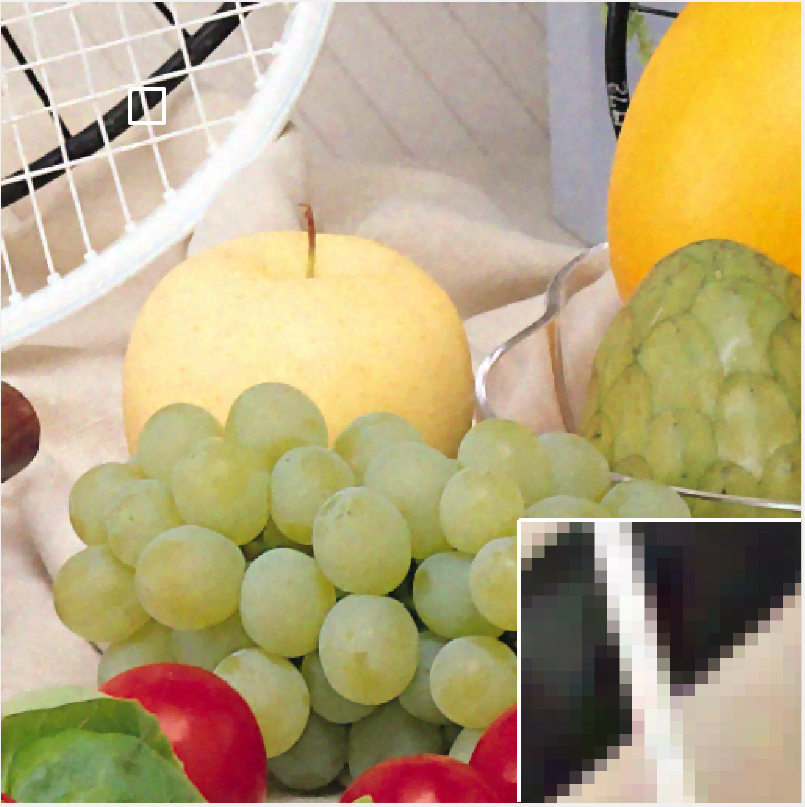}
	\end{tabular}
	\caption{Performance of the proposed models on general images with Gaussian noise and SD=0.06. (a) Clean images. (b) Noisy images. (c) Denoised images by model (\ref{eq.model1}) with $\alpha=5\times10^{-4}, \beta=50$ and $\eta=3$. (d) Denoised images by model (\ref{eq.model2}) with $\alpha=3\times10^{-2},\beta=30$ and $\eta=0.2$. From Row 1 to Row 3: 'Portrait', 'Peppers', 'Fruits'.
	}
	\label{fig.general}
\end{figure}

\begin{figure}[t!]
	\centering
	\begin{tabular}{cccc}
		(a) & (b) &(c) & (d)\\
		\includegraphics[trim={2.4cm 1.4cm 2.1cm 1.5cm},clip,width=0.22\textwidth]{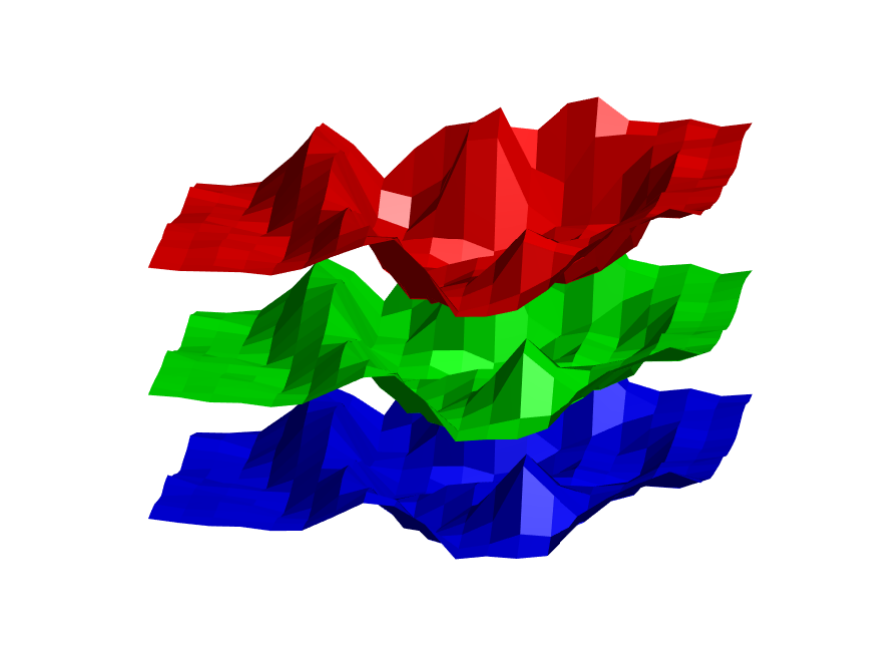}&
		\includegraphics[trim={2.4cm 1.4cm 2.1cm 1.5cm},clip,width=0.22\textwidth]{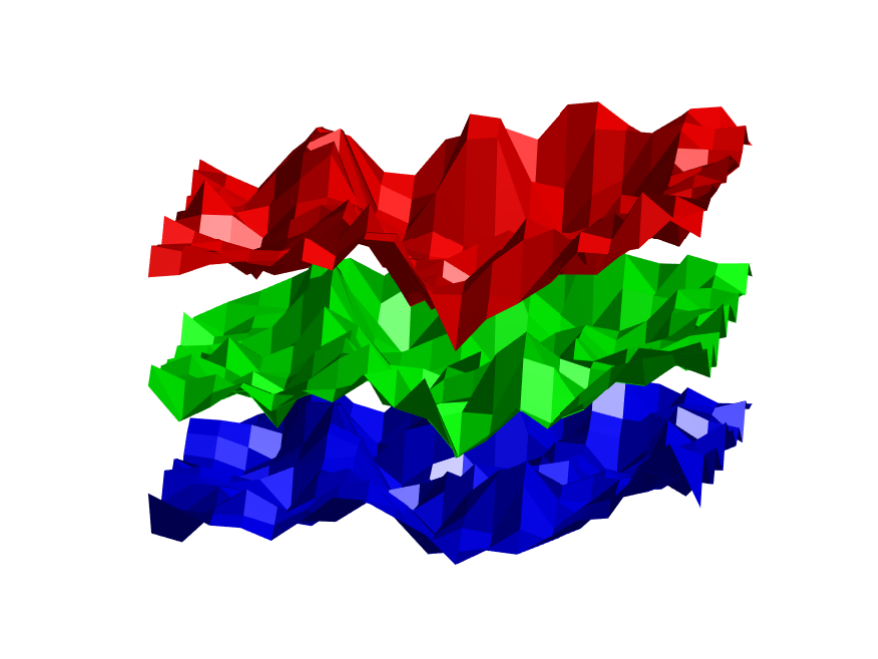}& 
		\includegraphics[trim={2.4cm 1.4cm 2.1cm 1.5cm},clip,width=0.22\textwidth]{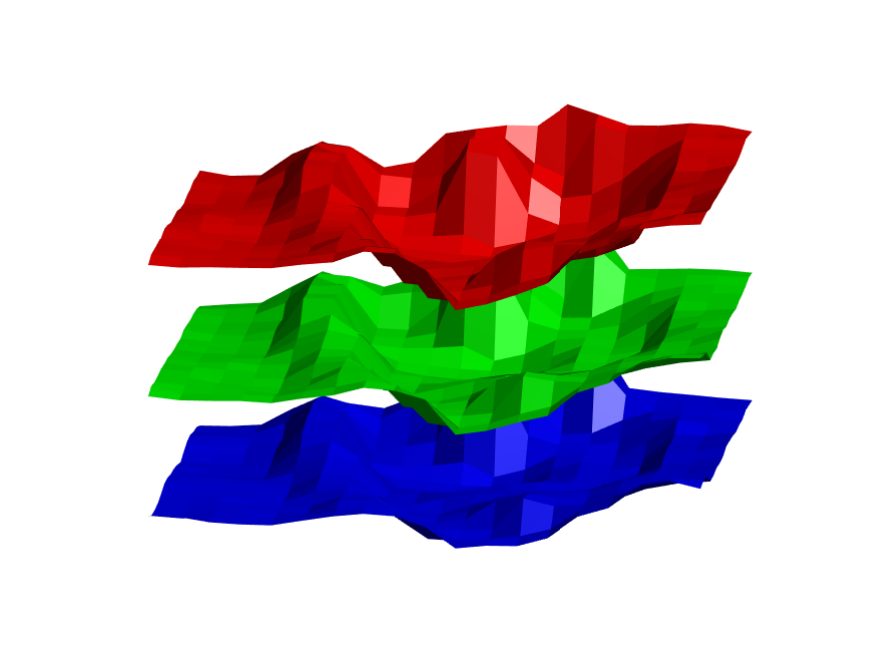} &
		\includegraphics[trim={2.4cm 1.4cm 2.1cm 1.5cm},clip,width=0.22\textwidth]{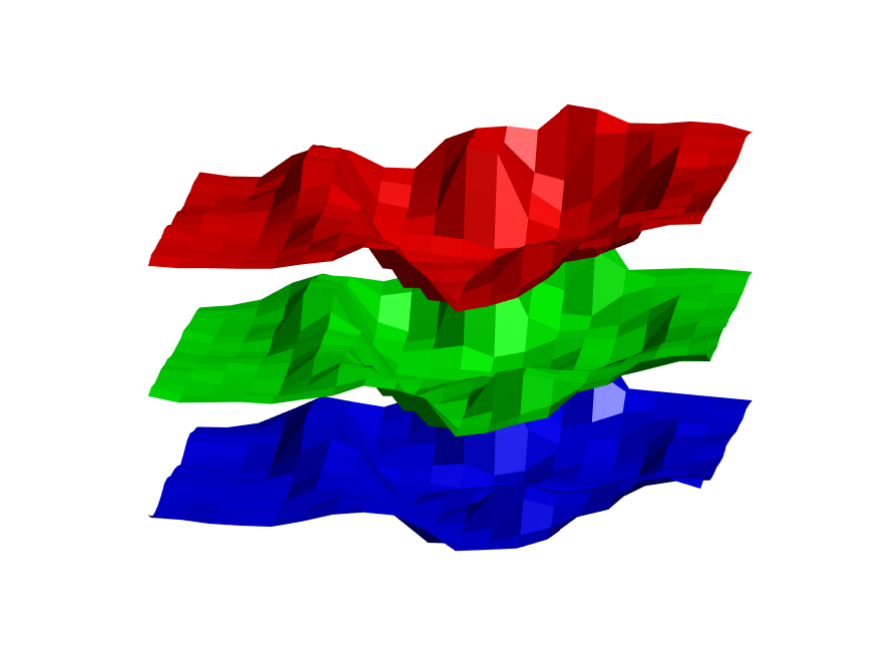}\\
		\includegraphics[trim={2.4cm 2cm 2.2cm 1.5cm},clip,width=0.22\textwidth]{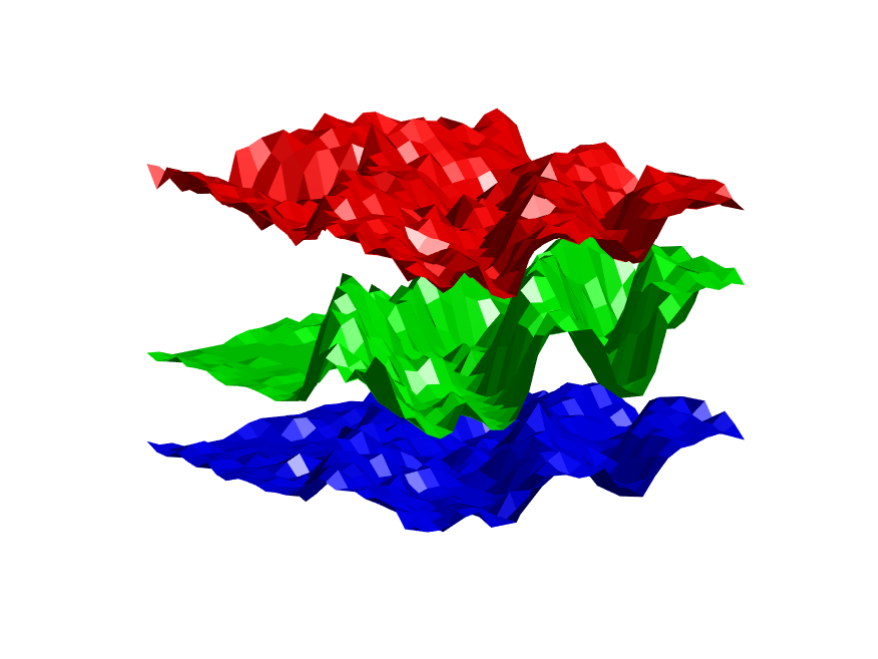}&
		\includegraphics[trim={2.4cm 2cm 2.2cm 1.5cm},clip,width=0.22\textwidth]{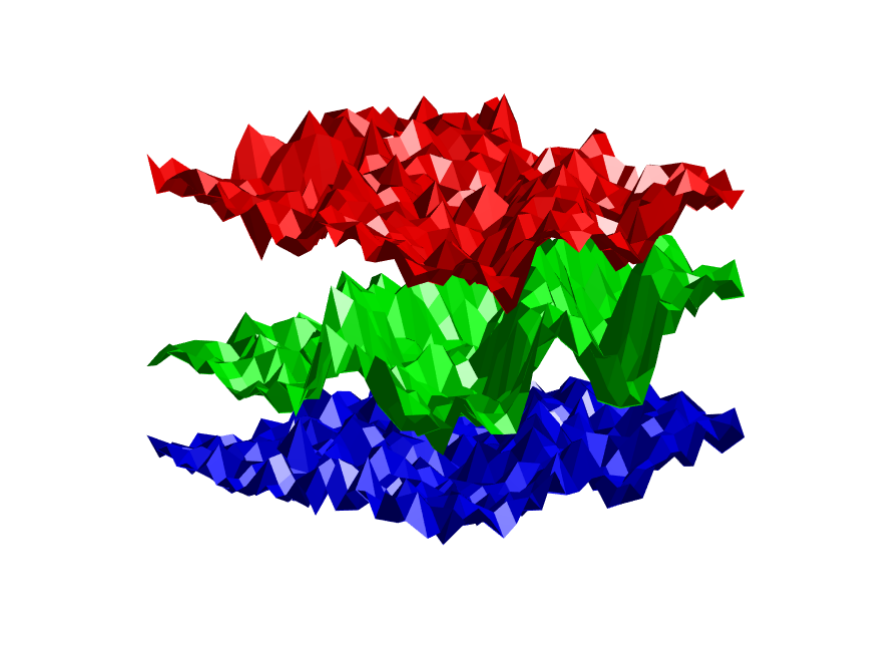}& 
		\includegraphics[trim={2.4cm 2cm 2.2cm 1.5cm},clip,width=0.22\textwidth]{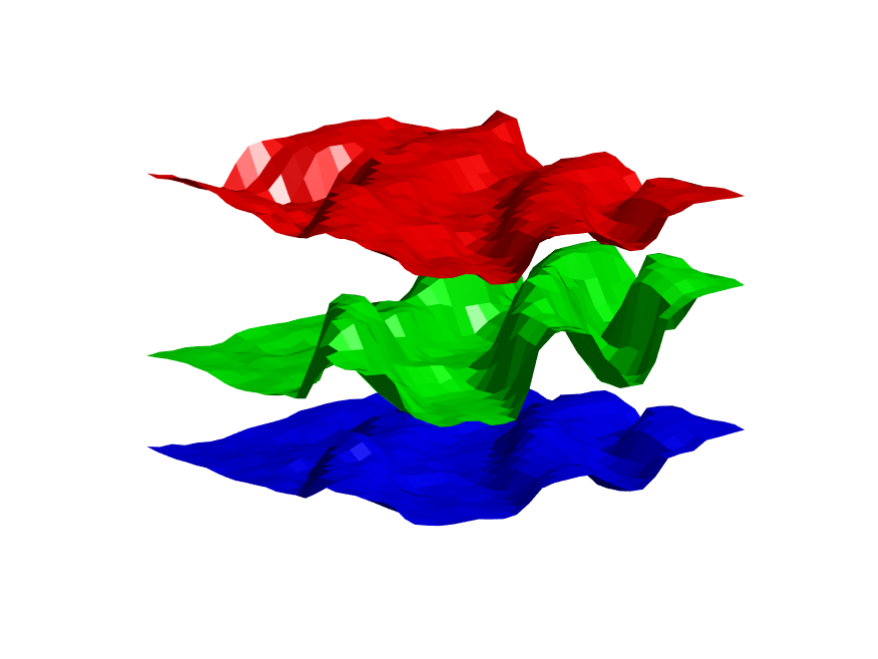} &
		\includegraphics[trim={2.4cm 2cm 2.2cm 1.5cm},clip,width=0.22\textwidth]{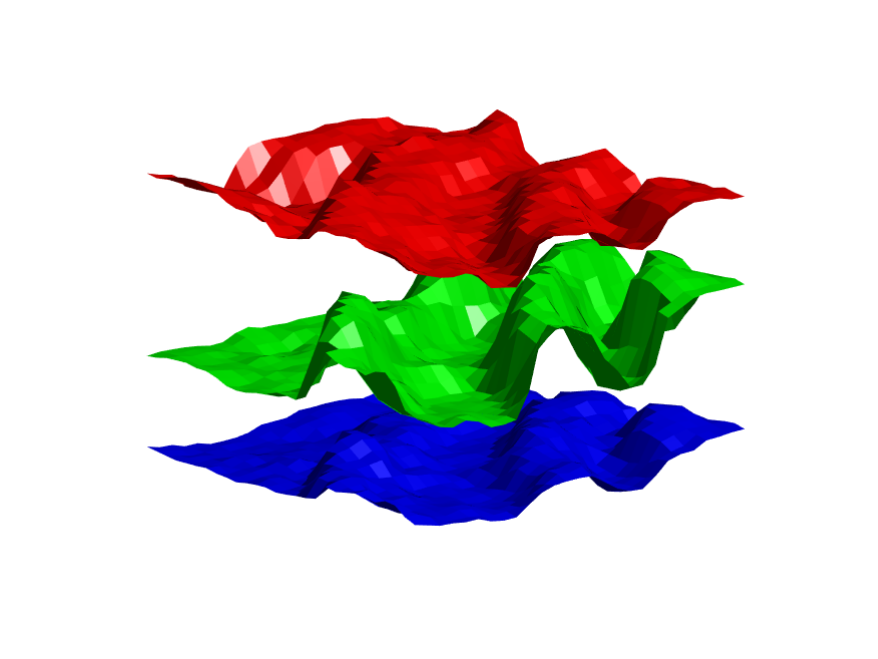}\\
		\includegraphics[trim={2.5cm 1.5cm 2.2cm 1.5cm},clip,width=0.22\textwidth]{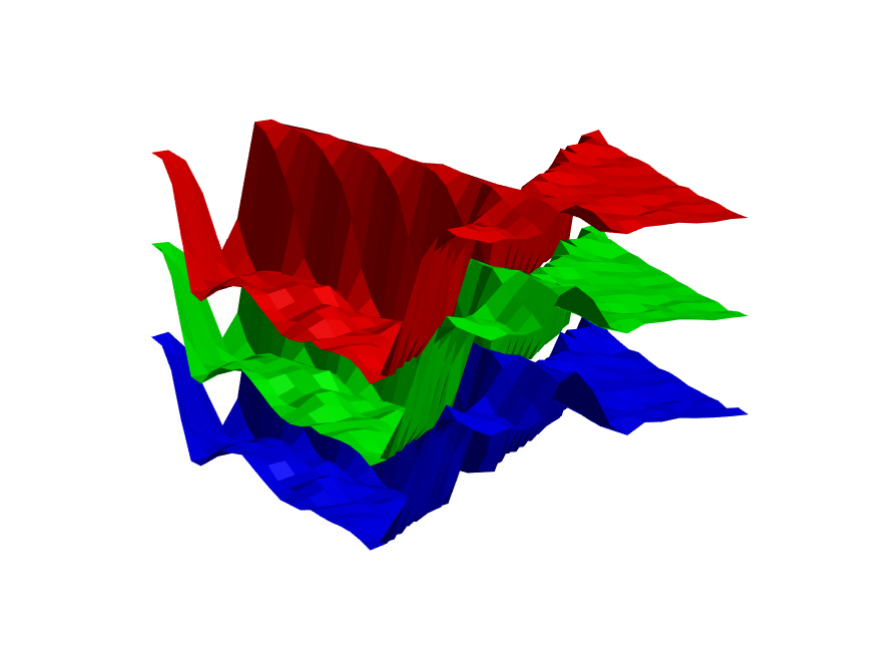}&
		\includegraphics[trim={2.5cm 1.5cm 2.2cm 1.5cm},clip,width=0.22\textwidth]{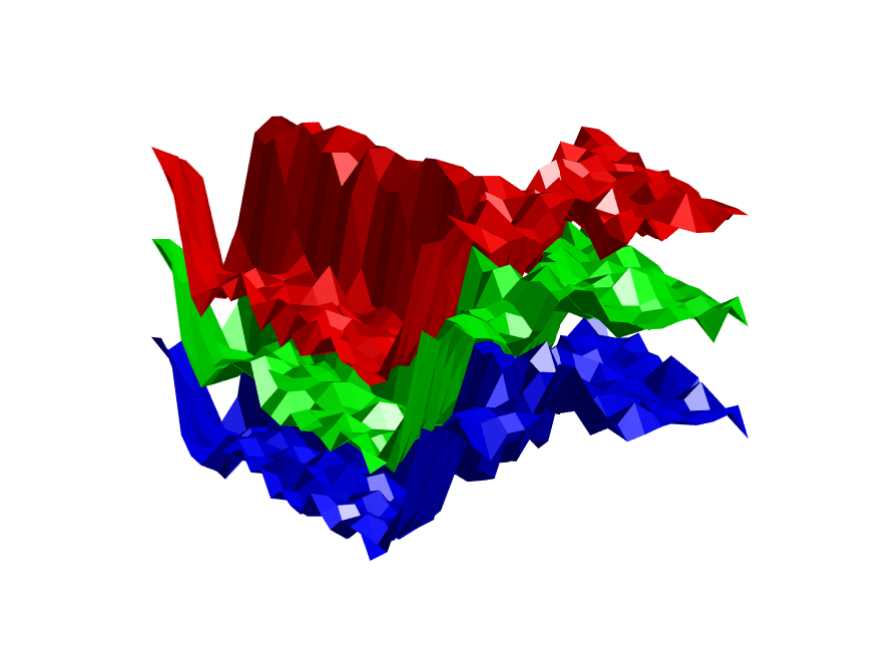}& 
		\includegraphics[trim={2.5cm 1.5cm 2.2cm 1.5cm},clip,width=0.22\textwidth]{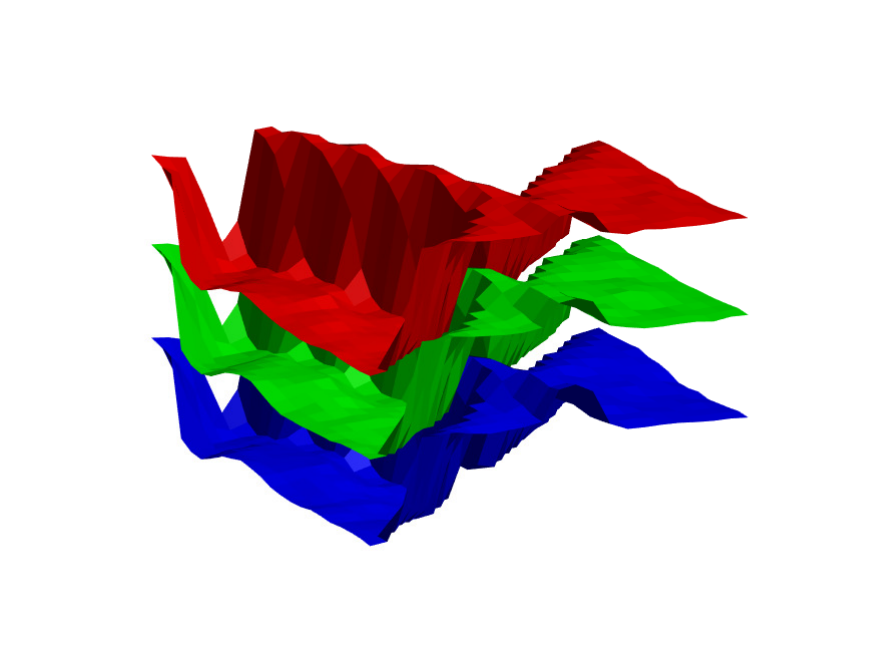} &
		\includegraphics[trim={2.5cm 1.5cm 2.2cm 1.5cm},clip,width=0.22\textwidth]{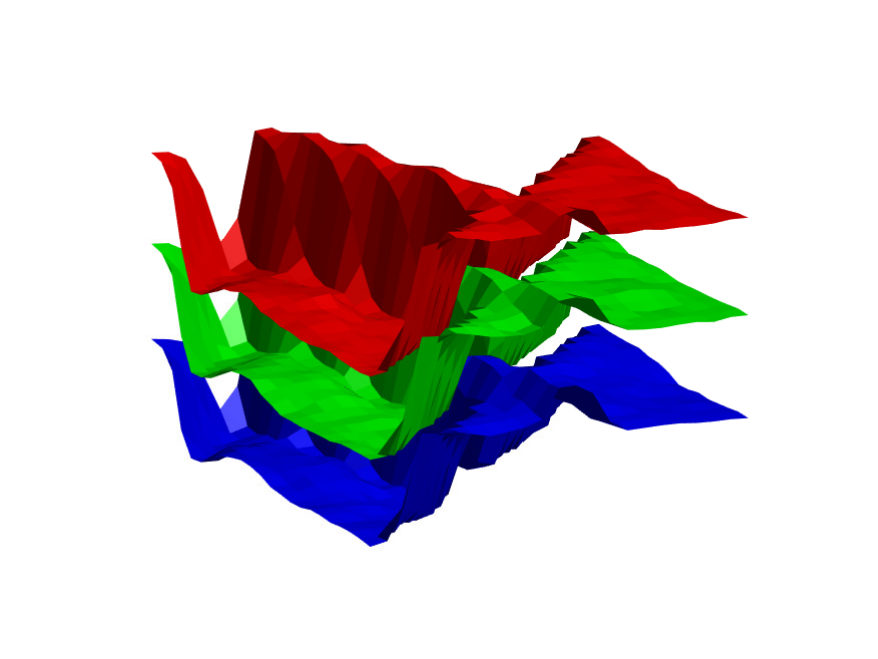}
	\end{tabular}
	\caption{Performance of the proposed models on general images. Surface plot of the zoomed regions of images in Figure \ref{fig.general}. (a) Clean images. (b) Noisy images with Gaussian noise and SD=0.06. (c) Denoised images by model (\ref{eq.model1}) with $\alpha=5\times10^{-4}, \beta=50$ and $\eta=3$. (d) Denoised images by model (\ref{eq.model2}) with $\alpha=3\times10^{-2},\beta=30$ and $\eta=0.2$. From Row 1 to Row 3: 'Portrait', 'Peppers', 'Fruits'.}
	\label{fig.general.surf}
\end{figure}

\begin{figure}[t!]
	\centering
	\begin{tabular}{ccc}
		(a) & (b) &(c) \\
		\includegraphics[trim={0.0cm 0cm 0.05cm 0cm},clip,width=0.3\textwidth]{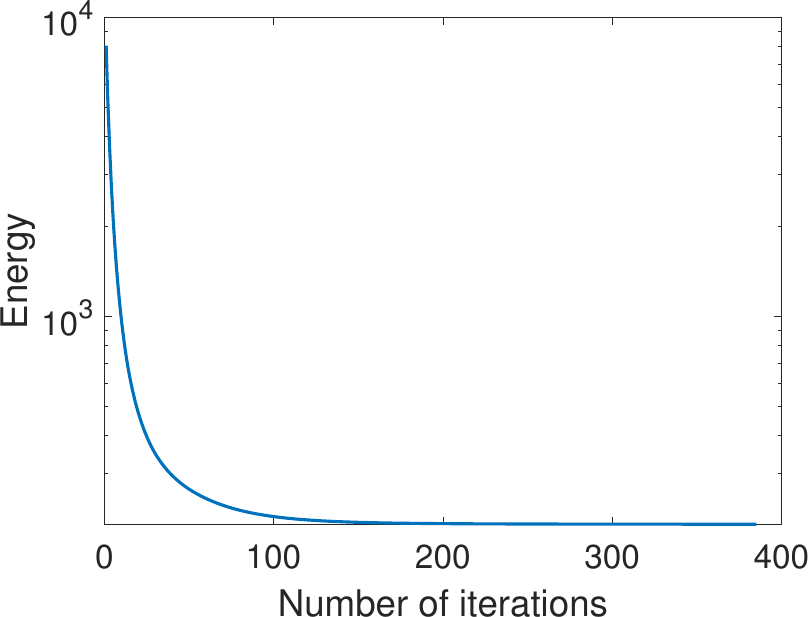}&
		\includegraphics[trim={0.0cm 0cm 0.05cm 0cm},clip,width=0.3\textwidth]{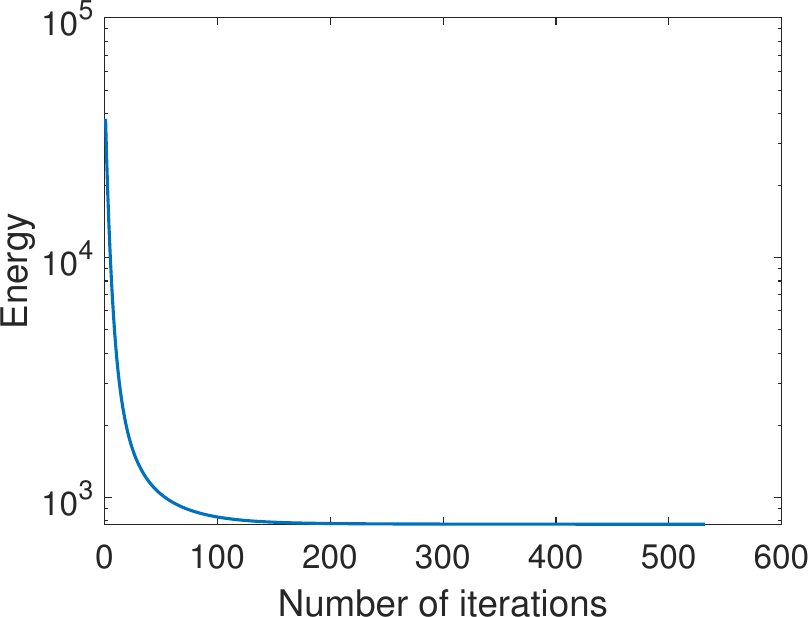}&
		\includegraphics[trim={0.0cm 0cm 0.05cm 0cm},clip,width=0.3\textwidth]{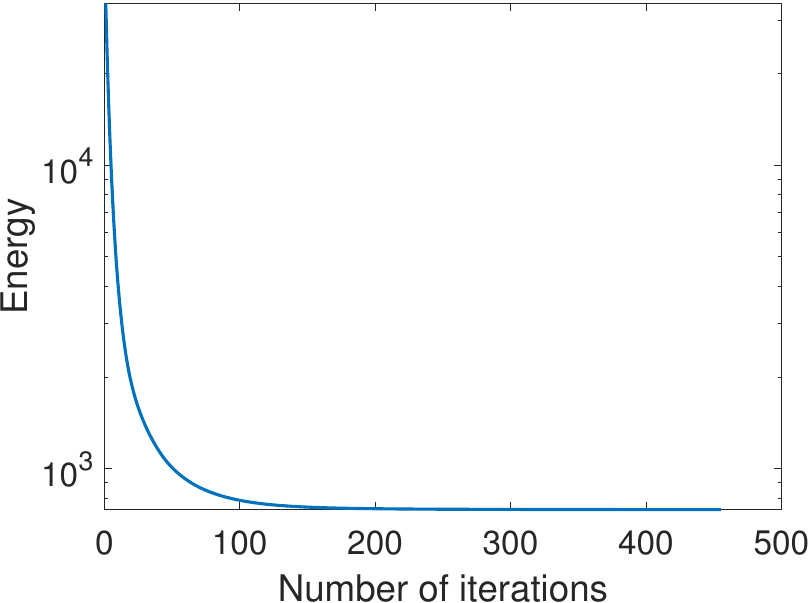}\\
		\includegraphics[trim={0.0cm 0cm 0.05cm 0cm},clip,width=0.3\textwidth]{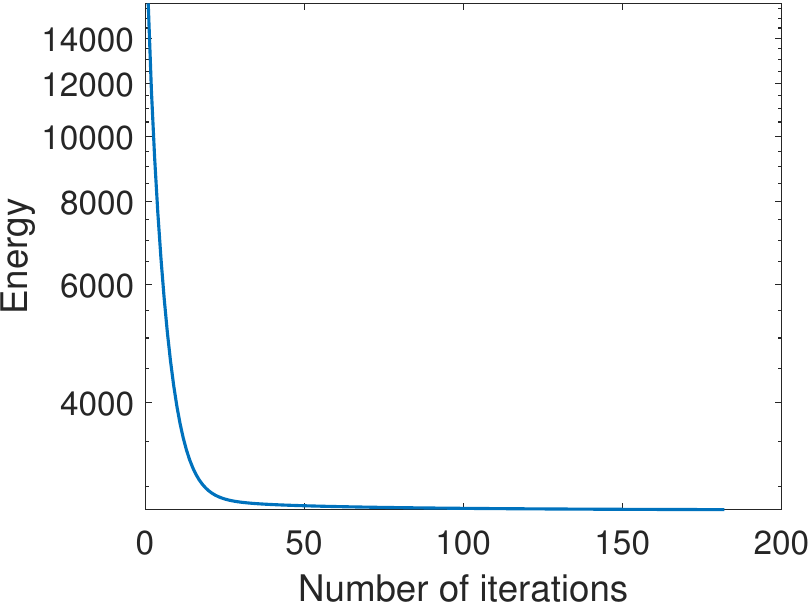} &
		\includegraphics[trim={0.0cm 0cm 0.05cm 0cm},clip,width=0.3\textwidth]{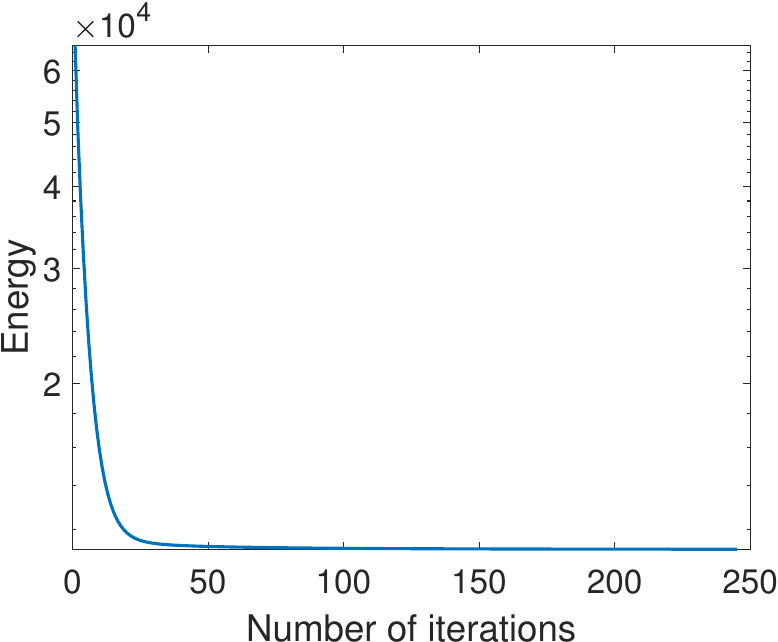} &
		\includegraphics[trim={0.0cm 0cm 0.05cm 0cm},clip,width=0.3\textwidth]{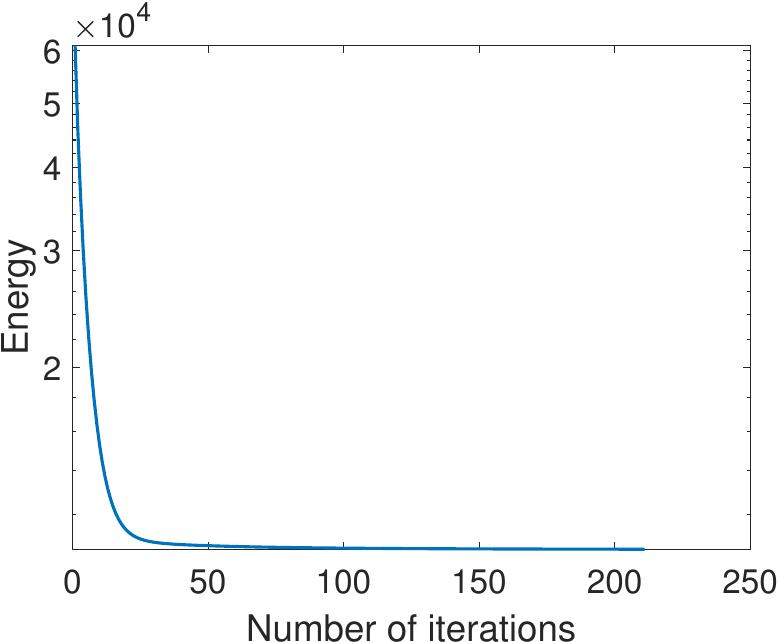} 
	\end{tabular}
	\caption{Performance of the proposed models. Histories of energies of results in Figure \ref{fig.general}. Row 1: Results by model (\ref{eq.model1}). Row 2: Results by model (\ref{eq.model2}). (a)-(c) correspond to 'Portrait', 'Peppers' and 'Fruits', respectively. }
	\label{fig.general.ener}
\end{figure}

\begin{figure}[t!]
	\centering
	\begin{tabular}{ccc}
		(a) & (b) &(c) \\
		\includegraphics[trim={0.0cm 0cm 0.05cm 0cm},clip,width=0.3\textwidth]{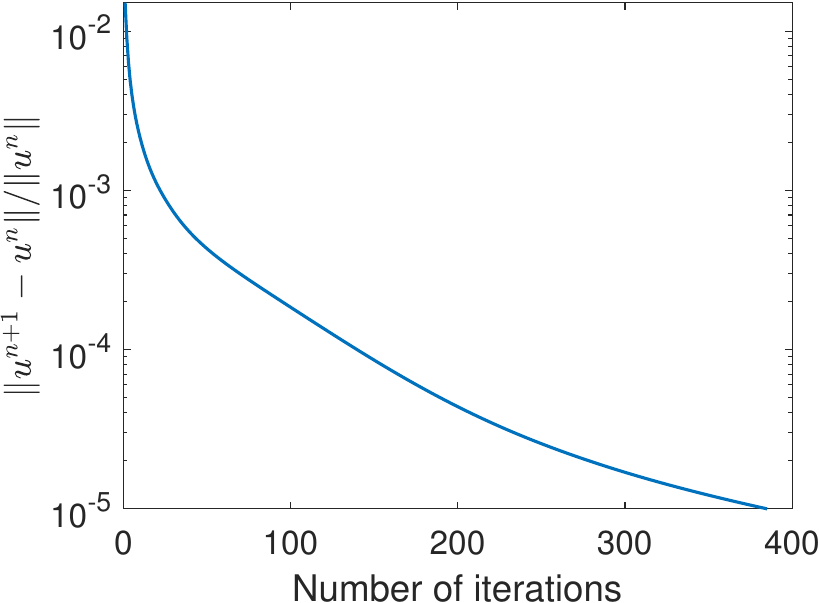}&
		\includegraphics[trim={0.0cm 0cm 0.05cm 0cm},clip,width=0.3\textwidth]{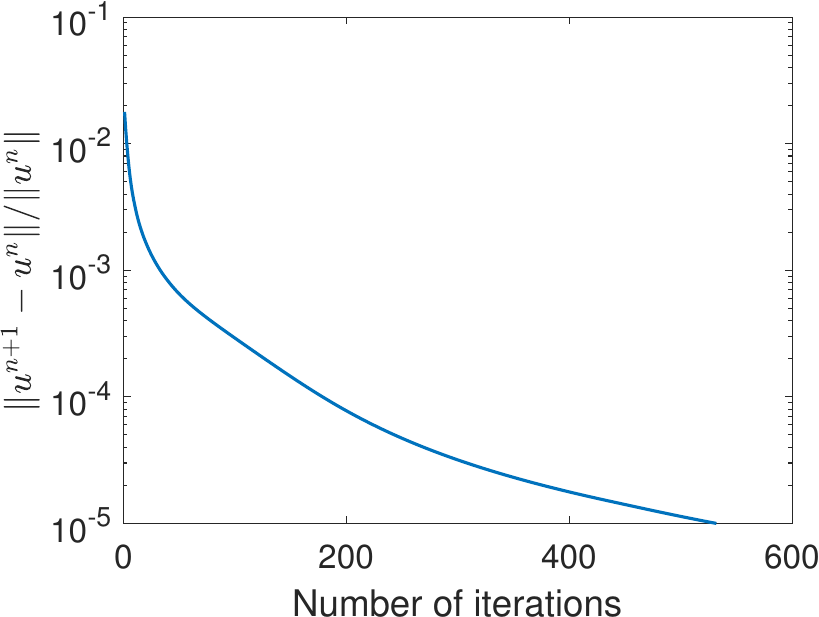}&
		\includegraphics[trim={0.0cm 0cm 0.05cm 0cm},clip,width=0.3\textwidth]{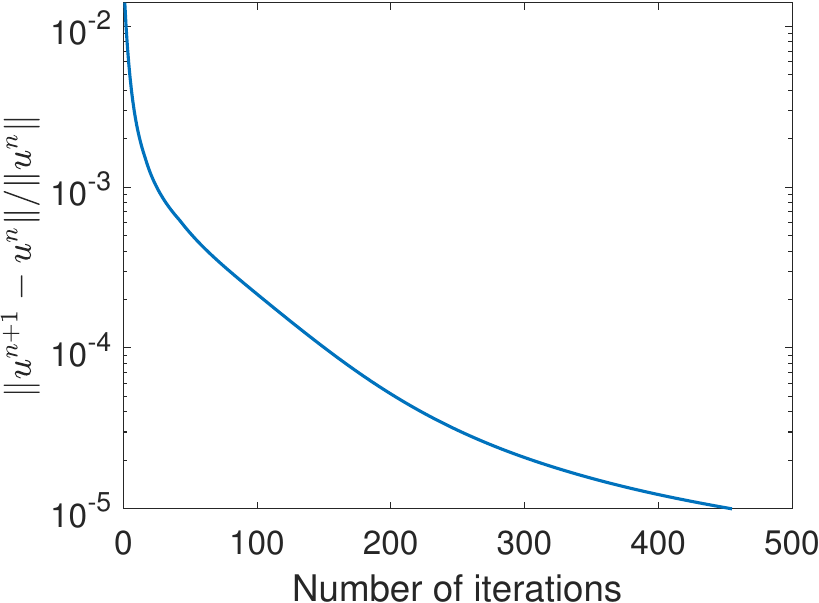}\\
		\includegraphics[trim={0.0cm 0cm 0.05cm 0cm},clip,width=0.3\textwidth]{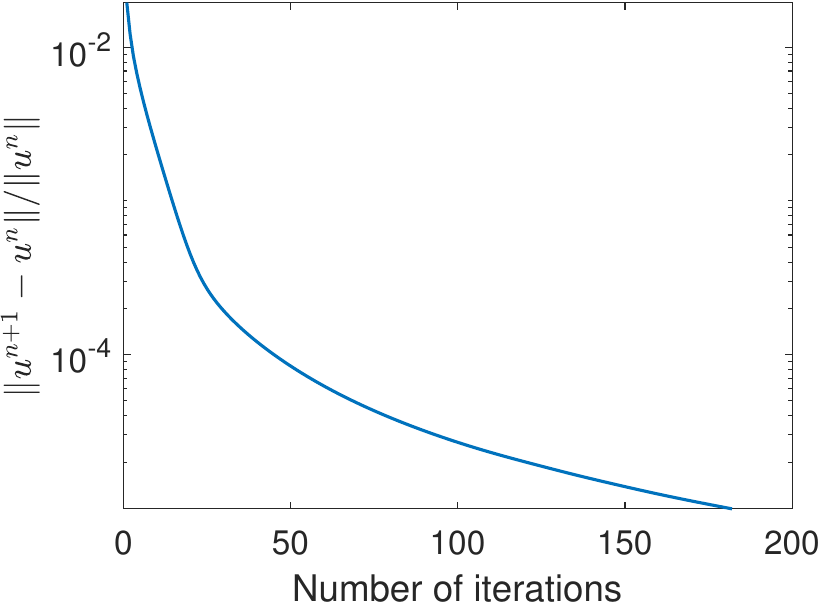} &
		\includegraphics[trim={0.0cm 0cm 0.05cm 0cm},clip,width=0.3\textwidth]{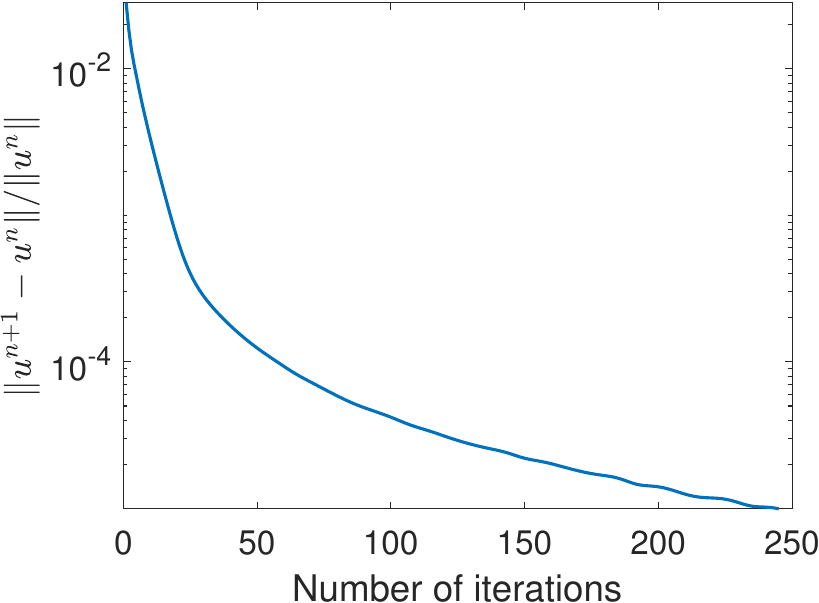} &
		\includegraphics[trim={0.0cm 0cm 0.05cm 0cm},clip,width=0.3\textwidth]{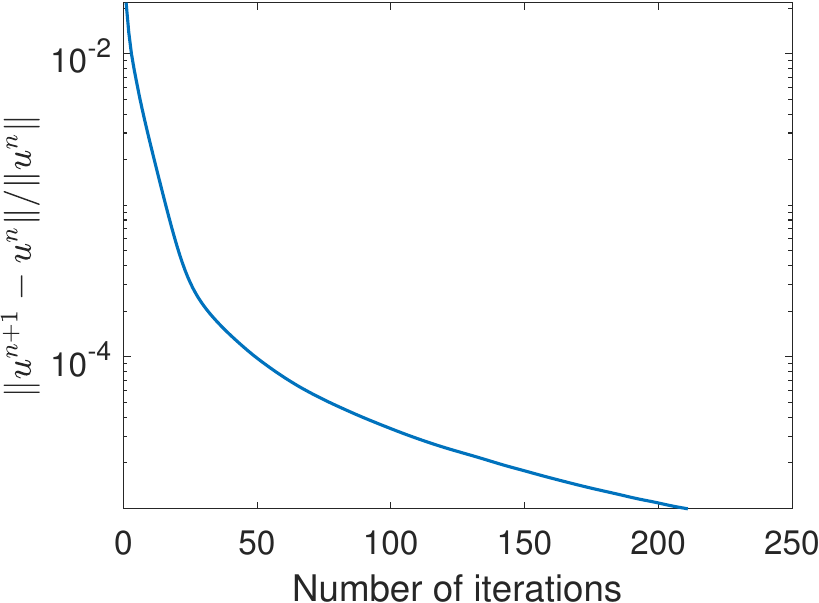} 
	\end{tabular}
	\caption{Performance of the proposed models. 
		Histories of relative errors of results in Figure \ref{fig.general}. Row 1: Results by model (\ref{eq.model1}). Row 2: Results by model (\ref{eq.model2}). (a)-(c) correspond to 'Portrait', 'Peppers' and 'Fruits', respectively. }
	\label{fig.general.err}
\end{figure}

\subsection{Performance on general examples}
We apply the propose algorithm on three images to test the performance and efficiency. The clean images are presented in the first row of Figure \ref{fig.general}. Then Gaussian noise with standard deviation (SD) 0.06 is added to generate the noisy images, as shown in the second row of Figure \ref{fig.general}. The denoised images by model (\ref{eq.model1}) and (\ref{eq.model2}) are shown in the third and fourth row of Figure \ref{fig.general}, respectively. Both models smooth the flat region of the images while keeping sharp edges. For the second column, the image 'Peppers', the clean image has some small oscillations. These oscillations are removed in the denoised images by both models. Meanwhile, the textures and shadows are kept. 

To better demonstrate the power of the proposed models, in Figure \ref{fig.general.surf}, we present the surface plot of the zoomed region of images in Figure \ref{fig.general}. Column 1--4 correspond to the clean images, noisy images, denoised images by (\ref{eq.model1}) and denoised images by model (\ref{eq.model2}), respectively. The three rows correspond to the three test images. The red, green and blue surfaces are surface plot of the RGB channels of the images. The denoised images by both models provide smooth image surfaces while keeping the contrast.

We then demonstrate the efficiency of the proposed algorithms in Figure \ref{fig.general.ener} and \ref{fig.general.err}, which present the histories of energy and relative error of the results in Figure \ref{fig.general} during iterations, respectively. Here the energy refers to the value of the functionals in (\ref{eq.model1}) and (\ref{eq.model2}). In both figures, the first row shows results by Algorithm \ref{alg.1}. The second row shows the results by Algorithm \ref{alg.2}. Column (a)--(c) correspond to 'Portrait', 'Peppers' and 'Fruits', respectively. In all experiments, Algorithm \ref{alg.1} needs about 100 iterations for the energy to achieve its minimum, while Algorithm \ref{alg.2} is more efficient and only uses 50 iterations. In terms of the relative error, sublinear converges is observed for both algorithms.

\begin{figure}[t!]
	\centering
	\begin{tabular}{cccc}
		(a) & (b) &(c) & (d)\\
		\includegraphics[width=0.22\textwidth]{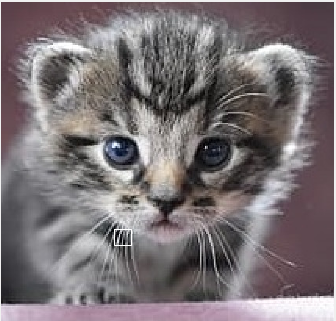}&
		\includegraphics[width=0.22\textwidth]{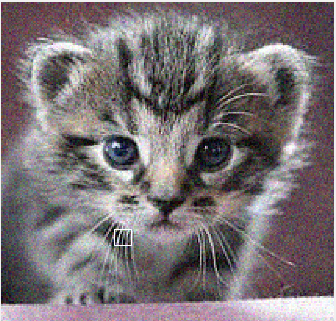}&
		\includegraphics[width=0.22\textwidth]{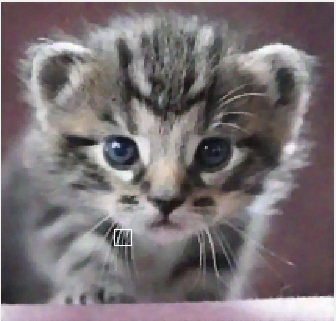}&
		\includegraphics[width=0.22\textwidth]{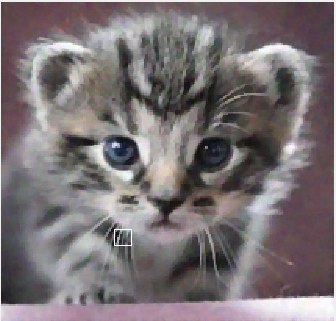}\\
		\includegraphics[trim={0.017cm 0 0.017cm 0.026cm},clip,width=0.22\textwidth]{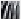}&
		\includegraphics[trim={0.017cm 0 0.017cm 0.026cm},clip,width=0.22\textwidth]{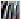}&
		\includegraphics[trim={0.017cm 0 0.017cm 0.026cm},clip,width=0.22\textwidth]{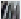}&
		\includegraphics[trim={0.017cm 0 0.017cm 0.026cm},clip,width=0.22\textwidth]{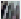}\\
		(e) &(f) &(g) & (h)\\
		\includegraphics[width=0.22\textwidth]{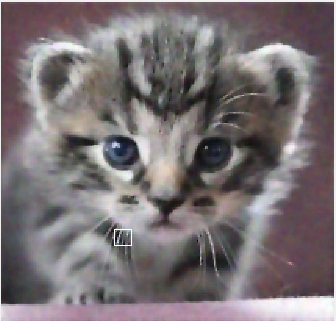}&
		\includegraphics[width=0.22\textwidth]{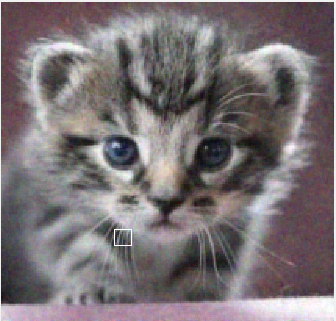}&
		\includegraphics[width=0.22\textwidth]{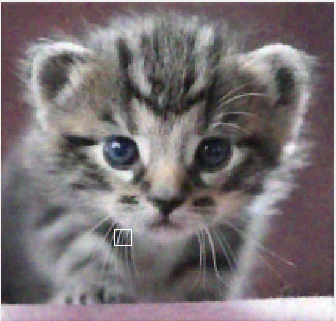}&
		\includegraphics[width=0.22\textwidth]{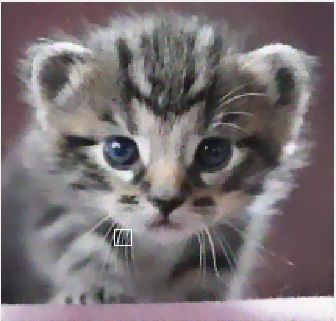}\\
		\includegraphics[trim={0.017cm 0 0.017cm 0.026cm},clip,width=0.22\textwidth]{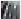}&
		\includegraphics[trim={0.017cm 0 0.017cm 0.026cm},clip,width=0.22\textwidth]{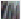}&
		\includegraphics[trim={0.017cm 0 0.017cm 0.026cm},clip,width=0.22\textwidth]{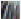}&
		\includegraphics[trim={0.017cm 0 0.017cm 0.026cm},clip,width=0.22\textwidth]{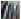}
	\end{tabular}
	\caption{Comparison of the proposed models and existing models on an image with Gaussian noise and SD=0.06. (a) Clean image. (b) Noisy image. (c) Denoised image by model (\ref{eq.model1}). (d) Denoised image by model (\ref{eq.model2}). (e) Denoised image by CE. (f) Denoised image by PA. (g) Denoised image by CTV. (h) Denoised image by VTV.}
	\label{fig.cat}
\end{figure}

\begin{table}[t!]
	\centering
	(a)\\
	\begin{tabular}{c|c|c|c|c|c|c|c}
		\hline\hline
		& Noisy & Model (\ref{eq.model1}) & Model (\ref{eq.model2}) & CE & PA & CTV & VTV\\
		\hline
		Portrait & 24.89& 33.91 & {\bf 34.00} &33.10 &30.56 & 31.74 & 32.94 \\
		\hline
		Peppers & 24.43 & {\bf 31.48} & 31.27 & 30.91 & 29.57 & 30.45 & 31.03\\
		\hline
		Fruits & 24.44 & 32.63 & {\bf 32.67} & 31.71 & 29.58 & 30.94 & 31.94\\
		\hline 
		Cat & 24.44 & {\bf 28.69} & 28.37 & 27.34 & 27.28 & 27.41 & 27.68\\
		\hline\hline
	\end{tabular}
	\vspace{0.3cm}
	
	(b)\\
	\begin{tabular}{c|c|c|c|c|c|c|c}
		\hline\hline
		& Noisy & Model (\ref{eq.model1}) & Model (\ref{eq.model2}) & CE & PA & CTV & VTV\\
		\hline
		Portrait & 0.6300 & 0.9514 & 0.9508 & 0.9379 &0.9132 &0.9022 &{\bf 0.9623}\\
		\hline
		Peppers & 0.9174 & {\bf 0.9822} & 0.9817 & 0.9802 & 0.9755 & 0.9781 & 0.9803\\
		\hline
		Fruits & 0.8149 & {\bf 0.9676} & 0.9673& 0.9611 & 0.9486 & 0.9518 & 0.9649\\
		\hline 
		Cat & 0.6919 & {\bf 0.8839} & 0.8778 & 0.8465 & 0.8435 & 0.8493 & 0.8576\\
		\hline\hline
	\end{tabular}
	\caption{\label{tab.compare} Comparison of the proposed models with existing models on images with Gaussian noise and SD=0.06. (a) PSNR values of the noisy image and denoised images by different methods. (b) SSIM values of the noisy image and denoised images by different methods. The largest value for each experiment (row) is marked in bold.}
\end{table}

\begin{table}[t!]
	\centering
	(a)\\
	\begin{tabular}{c|c|c|c|c|c|c}
		\hline\hline
		& Model (\ref{eq.model1}) & Model (\ref{eq.model2}) & CE & PA & CTV & VTV\\
		\hline
		Portrait $(295\times263)$& 385 & 182   &196 &118 & 100  & 492  \\
		\hline
		Peppers $(512\times512)$& 532  & 245 & 262 & 113 & 119 & 461\\
		\hline
		Fruits $(512\times512)$ & 455 & 211  & 244 & 114 & 113 & 427\\
		\hline 
		Cat $(204\times213)$ & 654  & 261 & 374 & 65 & 116 & 342\\
		\hline\hline
	\end{tabular}
	\vspace{0.3cm}
	
	(b)\\
	\begin{tabular}{c|c|c|c|c|c|c}
		\hline\hline
		& Model (\ref{eq.model1}) & Model (\ref{eq.model2}) & CE & PA & CTV & VTV\\
		\hline
		Portrait $(295\times263)$ & 61.92 & 25.63 & 64.44 & 2.07 &1.28 &8.86 \\
		\hline
		Peppers $(512\times512)$& 489.72  & 190.02 & 346.27 & 8.75 & 5.20 & 34.05\\
		\hline
		Fruits $(512\times512)$ & 475.37  & 164.47& 331.49 & 8.68 & 4.60 & 31.43\\
		\hline 
		Cat $(204\times213)$& 40.48  & 15.77 & 41.69 & 0.49 & 0.55 & 2.48\\
		\hline\hline
	\end{tabular}
	\caption{\label{tab.time} Comparison of the proposed models with existing models on images with Gaussian noise and SD=0.06. (a) Number of iterations used to satisfy stopping criterion. (b) CPU time (in seconds) used to satisfy stopping criterion.}
\end{table}

\begin{figure}[t!]
	\centering
	\begin{tabular}{cccc}
		\includegraphics[width=0.2\textwidth]{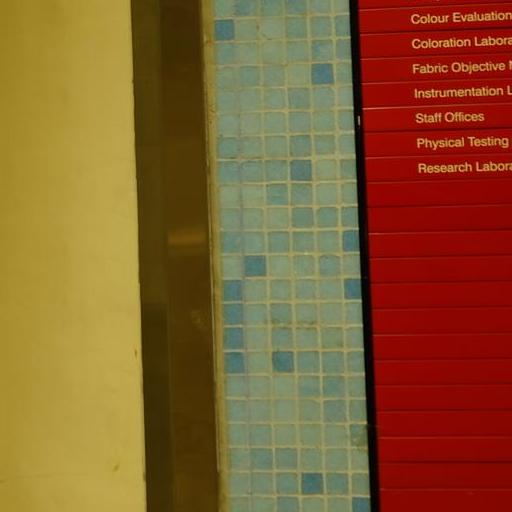}&
		\includegraphics[width=0.2\textwidth]{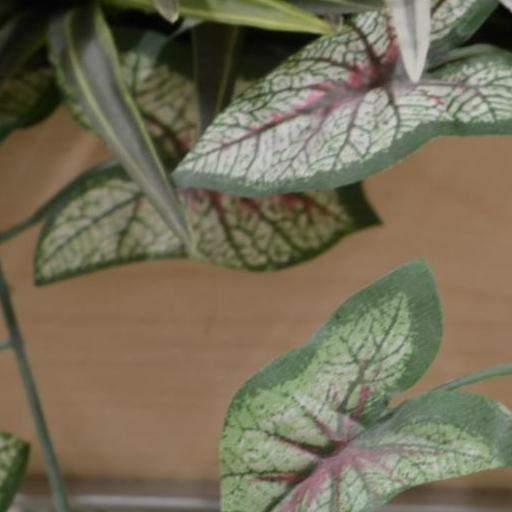}&
		\includegraphics[width=0.2\textwidth]{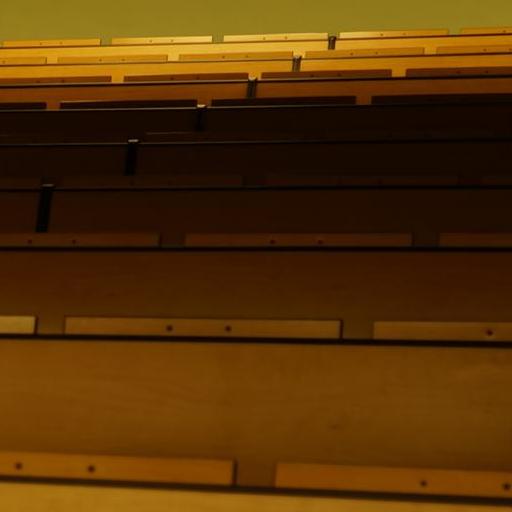}&
		\includegraphics[width=0.2\textwidth]{figure/BM_pen.jpg}\\
		\includegraphics[width=0.2\textwidth]{figure/BM_plant3.jpg}&
		\includegraphics[width=0.2\textwidth]{figure/BM_toy.jpg}&
		\includegraphics[width=0.2\textwidth]{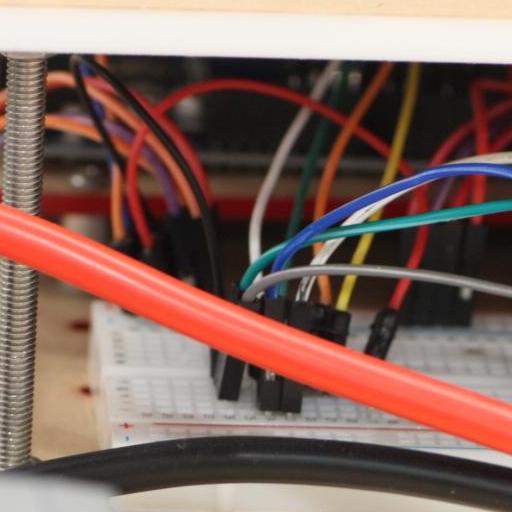}&
		\includegraphics[width=0.2\textwidth]{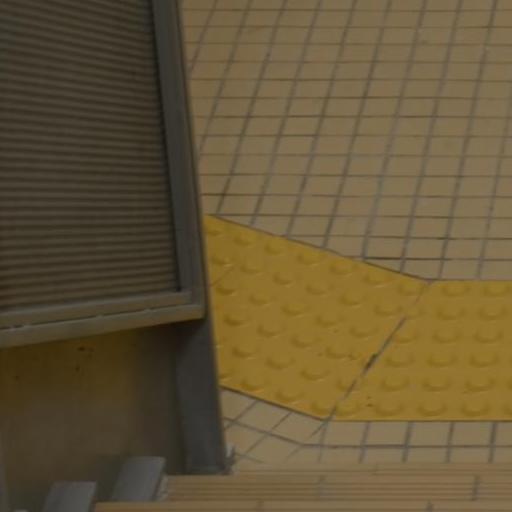}
	\end{tabular}
	\caption{Benchmark images from \cite{xu2018real}. Images in the first row: 'Wall', 'Flower', 'Classroom' and 'Pens'. Images in the second row: 'Plant', 'Toy', 'Wires' and 'Stairs'.}
	\label{fig.benchmark}
\end{figure}

\begin{table}[t!]
	\centering
	(a)\\
	\begin{tabular}{c|c|c|c|c|c|c|c}
		\hline\hline
		& Noisy & Model (\ref{eq.model1}) & Model (\ref{eq.model2}) & CE & PA & CTV & VTV\\
		\hline
		Wall& 24.91 & {\bf 33.91} & 33.84 & 33.02 & 31.63 & 32.01 & 33.57\\
		\hline
		Flower& 24.50 & {\bf 35.14} & 35.08 & 34.73 & 33.39 & 33.03 & 34.02\\
		\hline
		Classroom & 25.35 & 34.41 & {\bf 34.49} & 33.64 & 32.39 & 32.59 & 34.05\\
		\hline
		Pens & 24.74 & 33.38 & {\bf 33.50} & 32.41 & 29.78 & 31.05 & 31.96\\
		\hline
		Plant & 24.60 & 33.30 & {\bf 33.46} & 32.24 & 30.20 & 31.13 & 31.91\\
		\hline
		Toy & 24.53 & {\bf 35.56} & 35.40 & 34.92 & 33.64 & 33.15 & 35.11\\
		\hline
		Wires & 24.74 & 34.06 & {\bf 34.15} & 33.56 & 31.38 & 32.19 & 32.78\\
		\hline
		Stairs & 24.47 & {\bf 36.65} & 36.59 & 36.17 & 36.22 & 33.86 & 36.22\\
		\hline\hline
		Average & 24.73 & 34.55 & {\bf 34.56} & 33.84 & 32.33 & 32.38 &33.70\\
		\hline\hline
	\end{tabular}
	\vspace{0.3cm}
	
	(b)\\
	\begin{tabular}{c|c|c|c|c|c|c|c}
		\hline\hline
		& Noisy & Model (\ref{eq.model1}) & Model (\ref{eq.model2}) & CE & PA & CTV & VTV\\
		\hline
		Wall& 0.8724 & {\bf 0.9836} & {\bf 0.9836} & 0.9816 & 0.9766 & 0.9744 & 0.9817\\
		\hline
		Flower& 0.6866 & {\bf 0.9600} & 0.9583 & 0.9564 & 0.9469 & 0.9348 & 0.9506\\
		\hline
		Classroom & 0.8350 & 0.9651 & {\bf 0.9655} & 0.9624 & 0.9587 & 0.9540 & 0.9637\\
		\hline
		Pens & 0.7070 & {\bf 0.9589} &0.9581 & 0.9502 & 0.9230 & 0.9235 & 0.9585\\
		\hline
		Plant & 0.6326 & 0.9503 & 0.9491 & 0.9388 & 0.9118 & 0.9072 & {\bf 0.9506}\\
		\hline
		Toy & 0.7872 & {\bf 0.9767} & 0.9757 & 0.9740 & 0.9711 & 0.9627 & 0.9743\\
		\hline
		Wires & 0.6424 & 0.9511 & 0.9485 & 0.9434 & 0.9228 & 0.9069 & {\bf 0.9519}\\
		\hline
		Stairs & 0.6995 & {\bf 0.9731} & 0.9723 & 0.9702 & 0.9711 & 0.9461 & 0.9730\\
		\hline\hline
		Average & 0.7328 & {\bf 0.9648} & 0.9639 & 0.9596 & 0.9477 & 0.9387 & 0.9630\\
		\hline\hline
	\end{tabular}
	\caption{ Comparison of the proposed models with existing models on benchmark images in Figure \ref{fig.benchmark}. The noisy images are generated by adding Gaussian noise with SD=0.06. (a) PSNR values of the noisy image and denoised images by different methods. (b) SSIM values of the noisy image and denoised images by different methods. The largest value for each experiment (row) is marked in bold. \label{tab.benchmark}}
\end{table}

\begin{table}[t!]
	\centering
	(a)\\
	\begin{tabular}{c|c|c|c|c|c|c|c}
		\hline\hline
		& Noisy & Model (\ref{eq.model1}) & Model (\ref{eq.model2}) & CE & PA & CTV & VTV\\
		\hline
		Flower& 14.42 &  28.64 & {\bf 28.67} & 28.22 & 27.41 & 26.36 & 28.28\\
		\hline
		Pens & 15.06 & {\bf 25.37} &  25.28 & 24.32 & 23.83 & 24.15 & 24.67\\
		\hline
		Plant & 15.01 & {\bf 25.95} & 25.85 & 25.02 & 24.74 & 24.76 & 25.33\\
		\hline
		Wires & 15.20 & 26.01 & {\bf 26.03} & 25.54 & 24.96 & 24.96 & 25.47\\
		\hline\hline
		Average & 14.92 & {\bf 26.49} &  26.45 & 25.76 & 25.24 & 25.06 &25.94\\
		\hline\hline
	\end{tabular}
	\vspace{0.3cm}
	
	(b)\\
	\begin{tabular}{c|c|c|c|c|c|c|c}
		\hline\hline
		& Noisy & Model (\ref{eq.model1}) & Model (\ref{eq.model2}) & CE & PA & CTV & VTV\\
		\hline
		Flower& 0.2018 &  0.8746 & {\bf 0.8767} & 0.8643 & 0.8241 & 0.7797 & 0.8688\\
		\hline
		Pens & 0.3131 &  0.8521 & {\bf 0.8553} & 0.8246 & 0.7774 & 0.7610 & 0.8521\\
		\hline
		Plant & 0.2165 & 0.8292 & 0.8330 & 0.7972 & 0.7379 & 0.7085 & {\bf 0.8415}\\
		\hline
		Wires & 0.2900 & 0.8542 & {\bf 0.8619} & 0.8368 & 0.7832 & 0.7441 &  0.8605\\
		\hline\hline
		Average & 0.2554 & 0.8525 & {\bf 0.8567} & 0.8307 & 0.7807 & 0.7483 & 0.8557\\
		\hline\hline
	\end{tabular}
	\caption{ Comparison of the proposed models with existing models on benchmark images in Figure \ref{fig.benchmark} with large noise. The noisy images are generated by adding Gaussian noise with SD=0.2. (a) PSNR values of the noisy image and denoised images by different methods. (b) SSIM values of the noisy image and denoised images by different methods. The largest value for each experiment (row) is marked in bold. \label{tab.benchmark02}}
\end{table}

\begin{figure}[t!]
	\centering
	\begin{tabular}{ccc}
		(a) & (b) & \\
		\includegraphics[width=0.3\textwidth]{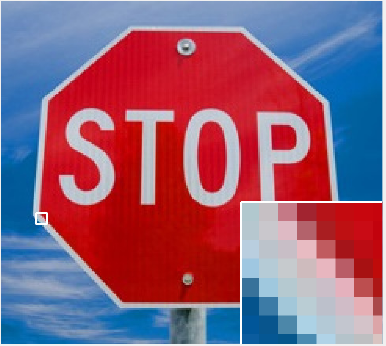}&
		\includegraphics[width=0.3\textwidth]{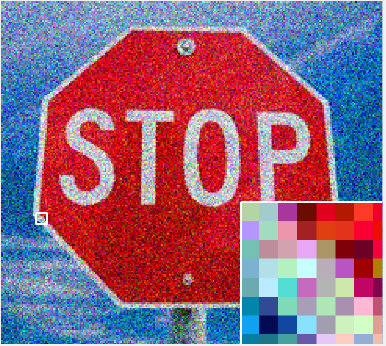}&\\
		(c) & (d) & (e) \\
		\includegraphics[width=0.3\textwidth]{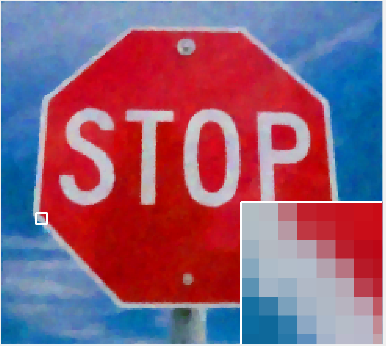}&
		\includegraphics[width=0.3\textwidth]{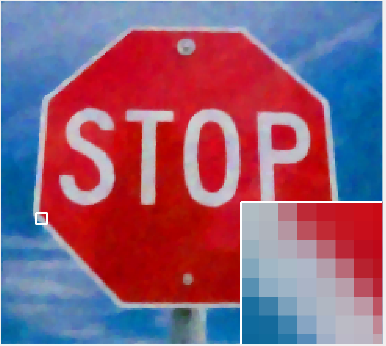}&
		\includegraphics[width=0.3\textwidth]{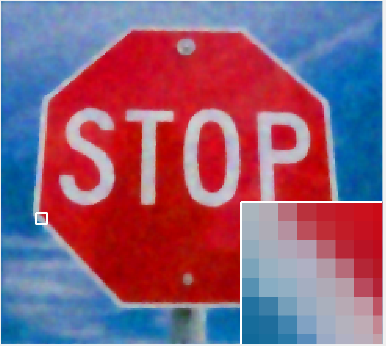}\\
		(f) & (g) & (h)\\
		\includegraphics[width=0.3\textwidth]{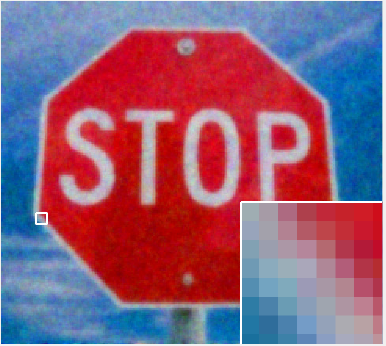}&
		\includegraphics[width=0.3\textwidth]{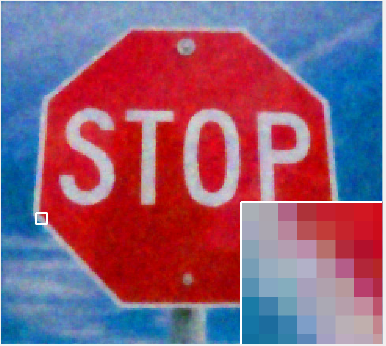}&
		\includegraphics[width=0.3\textwidth]{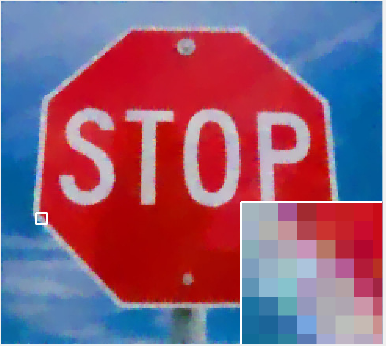}\\
		
	\end{tabular}
	\caption{ Comparison of the proposed models and existing models on an image with large noise. (a) Clean image. (b) Noisy image with SD=0.15. (c) Denoised image by model (\ref{eq.model1}). (d) Denoised image by model (\ref{eq.model2}). (e) Denoised image by CE. (f) Denoised image by PA. (g) Denoised image by CTV. (h) Denoised image by VTV.}
	\label{fig.stop.015}
\end{figure}

%

\begin{figure}[t!]
	\centering
	\begin{tabular}{ccc}
		(a) & (b) & \\
		\includegraphics[width=0.3\textwidth]{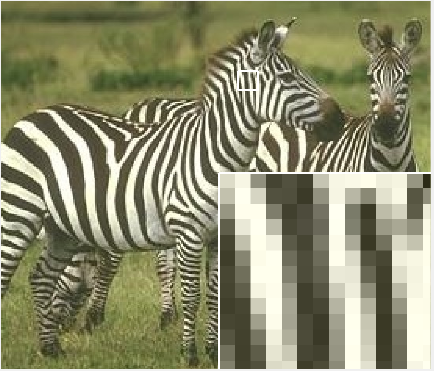}&
		\includegraphics[width=0.3\textwidth]{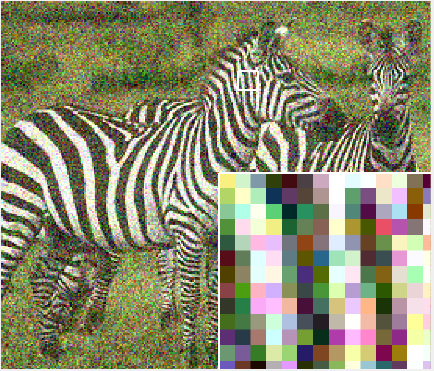}&\\
		(c) & (d) &(e) \\
		\includegraphics[width=0.3\textwidth]{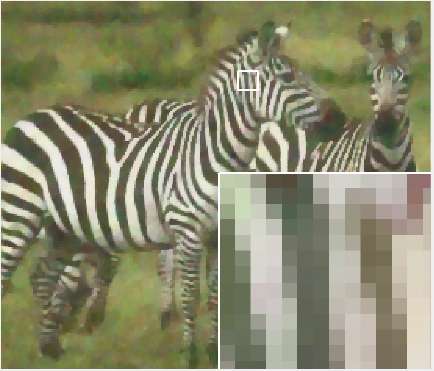}&
		\includegraphics[width=0.3\textwidth]{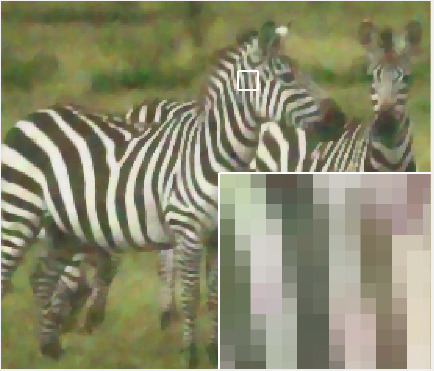}&
		\includegraphics[width=0.3\textwidth]{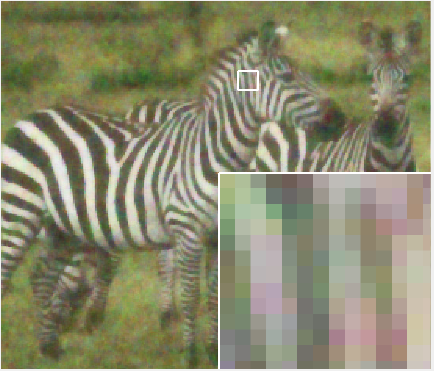}\\
		(f) & (g) & (h)\\
		\includegraphics[width=0.3\textwidth]{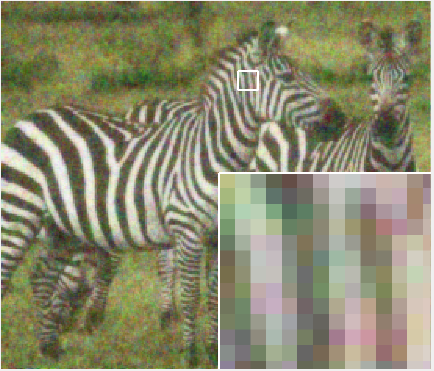}&
		\includegraphics[width=0.3\textwidth]{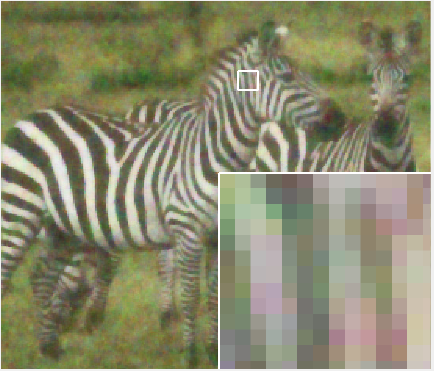}&
		\includegraphics[width=0.3\textwidth]{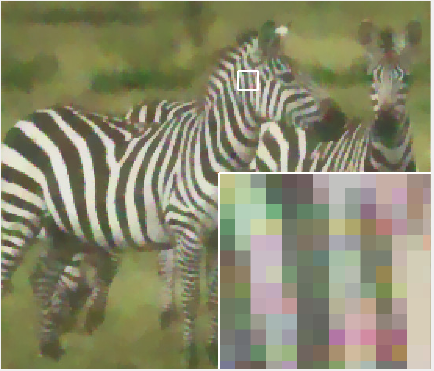}\\
		
	\end{tabular}
	\caption{Comparison of the proposed models and existing models on an image with large noise. (a) Clean image. (b) Noisy image with SD=0.15. (c) Denoised image by model (\ref{eq.model1}). (d) Denoised image by model (\ref{eq.model2}). (e) Denoised image by CE. (f) Denoised image by PA. (g) Denoised image by CTV. (h) Denoised image by VTV.}
	\label{fig.zebra.015}
\end{figure}

%

\begin{table}[t!]
	\centering
	(a)\\
	\begin{tabular}{c|c|c|c|c|c|c|c}
		\hline\hline
		& Noisy & Model (\ref{eq.model1}) & Model (\ref{eq.model2}) & CE & PA & CTV & VTV\\
		\hline 
		Stop Sign & 16.44 & {\bf 28.93} & 28.91 & 27.71 & 25.01 & 26.27 & 27.88\\
		\hline
		Zebra & 16.83 & {\bf 22.08} & 21.42 & 20.35 & 20.54 & 19.97 & 20.73\\
		\hline\hline
	\end{tabular}
	\vspace{0.3cm}
	
	(b)\\
	\begin{tabular}{c|c|c|c|c|c|c|c}
		\hline\hline
		& Noisy & Model (\ref{eq.model1}) & Model (\ref{eq.model2}) & CE & PA & CTV & VTV\\
		\hline 
		Stop Sign & 0.7785 & {\bf 0.9788} & 0.9787 & 0.9735 & 0.9547 & 0.9648 & 0.9738\\
		\hline
		Zebra & 0.6053 & {\bf 0.8486} & 0.8264 & 0.7870 & 0.7969 & 0.7844 & 0.8067\\
		\hline\hline
	\end{tabular}
	\caption{\label{tab.compare.015} Comparison of the proposed models with existing models on images with Gaussian noise and SD=0.15. (a) PSNR values of the noisy image and denoised images by different methods. (b) SSIM values of the noisy image and denoised images by different methods. The largest value for each experiment (row) is marked in bold.}
\end{table}

\begin{figure}[t!]
	\centering
	\begin{tabular}{ccc}
		(a) & (b) & \\
		\includegraphics[width=0.3\textwidth]{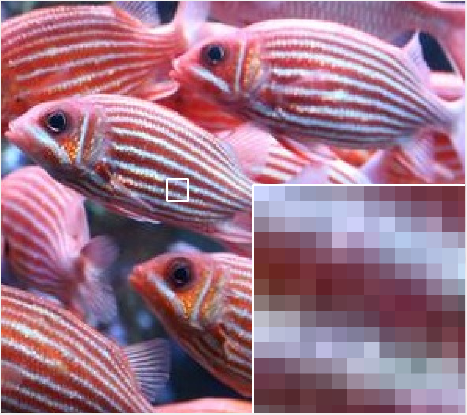}&
		\includegraphics[width=0.3\textwidth]{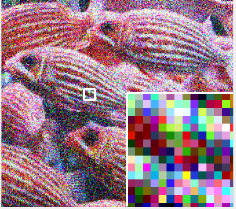}&\\
		(c) & (d) &(e) \\
		\includegraphics[width=0.3\textwidth]{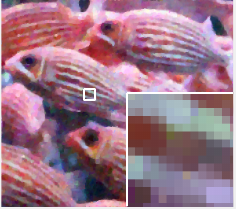}&
		\includegraphics[width=0.3\textwidth]{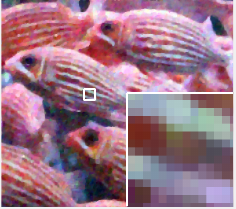}&
		\includegraphics[width=0.3\textwidth]{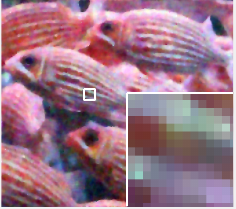}\\
		(f) & (g) & (h)\\
		\includegraphics[width=0.3\textwidth]{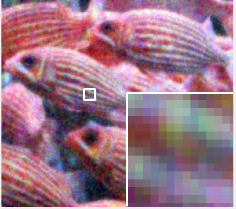}&
		\includegraphics[width=0.3\textwidth]{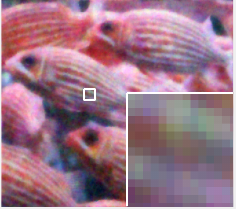}&
		\includegraphics[width=0.3\textwidth]{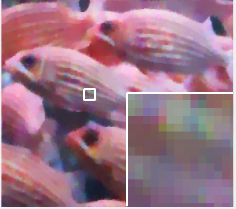}\\
		
	\end{tabular}
	\caption{Comparison of the proposed models and existing models on an image with large noise. (a) Clean image. (b) Noisy image with SD=0.3. (c) Denoised image by model (\ref{eq.model1}). (d) Denoised image by model (\ref{eq.model2}). (e) Denoised image by CE. (f) Denoised image by PA. (g) Denoised image by CTV. (h) Denoised image by VTV.}
	\label{fig.fish.03}
\end{figure}

%
\begin{center}
	\begin{figure}[t!]
		\centering
		\begin{tabular}{ccc}
			(a) & (b) & \\
			\includegraphics[width=0.3\textwidth]{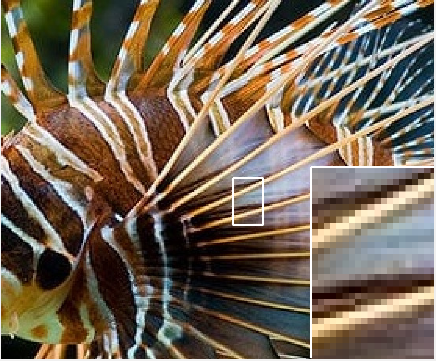}&
			\includegraphics[width=0.3\textwidth]{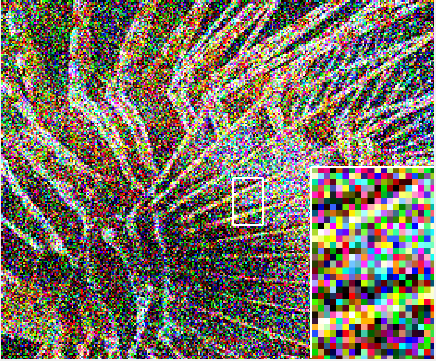}&\\
			(c) & (d) &(e) \\
			\includegraphics[width=0.3\textwidth]{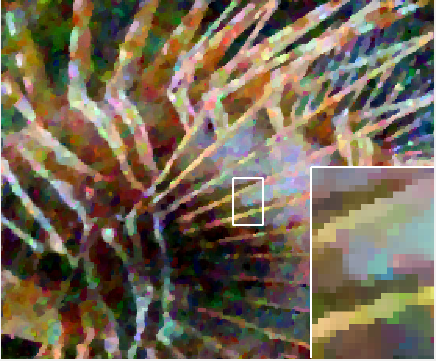}&
			\includegraphics[width=0.3\textwidth]{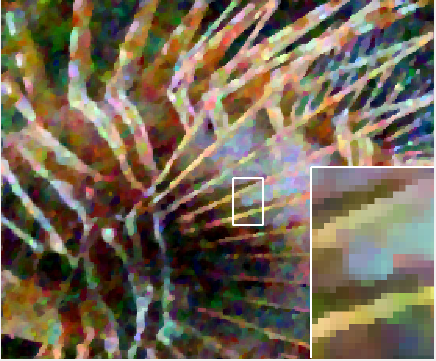}&
			\includegraphics[width=0.3\textwidth]{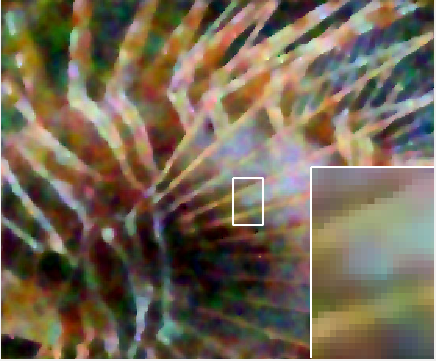}\\
			(f) & (g) & (h)\\
			\includegraphics[width=0.3\textwidth]{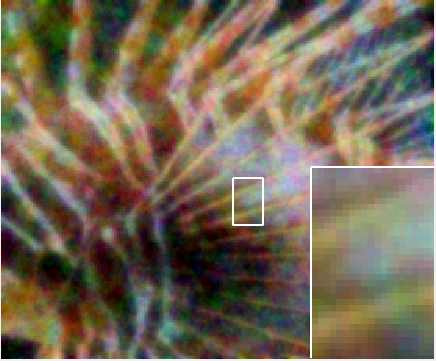}&
			\includegraphics[width=0.3\textwidth]{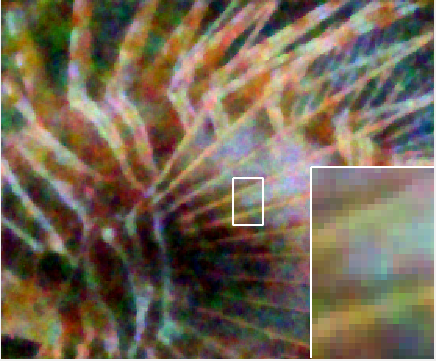}&
			\includegraphics[width=0.3\textwidth]{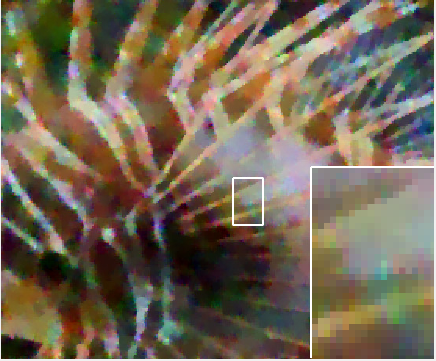}\\
			
		\end{tabular}
		\caption{Comparison of the proposed models and existing models on an image with large noise. (a) Clean image. (b) Noisy image with SD=0.5. (c) Denoised image by model (\ref{eq.model1}). (d) Denoised image by model (\ref{eq.model2}). (e) Denoised image by CE. (f) Denoised image by PA. (g) Denoised image by CTV. (h) Denoised image by VTV.}
		\label{fig.stripes.05}
	\end{figure}
\end{center}

\subsection{Comparison with other models}
We next demonstrate the advantages of the proposed models by comparing them with CE, PA, CTV and VTV.
In this experiment with noisy images and SD=0.06, we set $\alpha=5\times10^{-4}, \beta=50,\eta=3$ for model (\ref{eq.model1}), $\alpha=3\times10^{-2},\beta=30, \eta=0.2$ for model (\ref{eq.model2}). For other models, we use  $\alpha=1\times10^{-2},\beta=5\times10^{-3}, \eta=1$ for CE, $\lambda=6$ for CTV and $\lambda=0.1$ for VTV. The results on the natural images, 'Cat', is shown in Figure \ref{fig.cat}, in which (a) and (b) show the clean and noisy images, (c)--(h) show denoised images by model (\ref{eq.model1}), model (\ref{eq.model2}), CE, PA, CTV and VTV, respectively. To better demonstrate the advantage of the propose models, for each image, the zoomed image of some selected region is presented under it. 
The zoomed images in (c)--(d) have a uniform theme: the pixel colors of the same object change smoothly. In (f)--(h), many pixels have abrupt and artificial colors. In particular, the zoomed image in (h) looks like a color palette. In this comparison, the proposed models give more natural recoveries of images with less abrupt and artificial colors.

To quantify the improvements of the proposed models over others, we present the PSNR and SSIM \cite{wang2004image} values of the denoised images of all experiments and models in Table \ref{tab.compare}. For each experiment, the largest value is marked in bold. The proposed models provide the largest PSNR and SSIM values in almost all experiments. 

We present in Table \ref{tab.time} the number of iterations and CPU time (in seconds) used to satisfy the stopping criteria of each experiment in Table \ref{tab.compare}. Due to the model complexity, such as nonlinearities and high order derivatives, algorithms for the three color elastica related models (model (\ref{eq.model1}), model (\ref{eq.model2}) and CE) need more time than those for the other models. While among these three models, our algorithm for  model (\ref{eq.model2}) is the most efficient one. It only uses about half of the CPU time of that used by the other two algorithms to satisfy the stopping criterion.

We then conduct a comprehensive comparison of the proposed models with CE, PA, CTV and VTV. We consider eight benchmark images \cite{xu2018real} shown in Figure \ref{fig.benchmark}. The noisy images are generated by adding Gaussian noise with SD=0.06. The parameters of all models are the same as those mentioned at the beginning of the this subsection. The PSNR and SSIM values of all results are summarized in Table \ref{tab.benchmark}. On average, the results by model (\ref{eq.model1}) have an increment of 0.71 in PSNR and 0.0018 in SSIM over the best existing models. Compared to that, Model (\ref{eq.model2}) provides results with a larger PSNR but slightly smaller SSIM. 
The comparison of these models on images with larger noise, SD=0.2, is summarized in Table \ref{tab.benchmark02}. In this experiment, we use $\alpha=5\times10^{-4}, \beta=50,\eta=10$ for model (\ref{eq.model1}), $\alpha=5\times10^{-3},\beta=30, \eta=3.5$ for model (\ref{eq.model2}), $\alpha=1\times10^{-2},\beta=5\times10^{-3}, \eta=4$ for CE, $\lambda=2.5$ for CTV and $\lambda=0.3$ for VTV. In this comparison, both model (\ref{eq.model1}) and (\ref{eq.model2}) give results with larger PSNR than other models. The increment is about 0.5 on average. For the SSIM value, on average, model (\ref{eq.model2}) gives the best results.

We further compare all models on three images with large noise in Figure \ref{fig.stop.015}--\ref{fig.zebra.015}. The clean images and noisy images containing Gaussian noise with SD=0.15 are shown in (a) and (b), respectively. Denoised images by model (\ref{eq.model1}), model (\ref{eq.model2}), CE, PA, CTV and VTV are shown in (c)--(h), respectively. In this set of experiments, we set $\alpha=5\times10^{-4}, \beta=50,\eta=7$ for model (\ref{eq.model1}), $\alpha=5\times10^{-3},\beta=30, \eta=2.4$ for model (\ref{eq.model2}). For other models, we use  $\alpha=1\times10^{-2},\beta=5\times10^{-3}, \eta=2.5$ for CE, $\lambda=2.5$ for CTV and $\lambda=0.25$ for VTV. The proposed models provide the best results which recover the features with uniform color themes. For results by other models, either they have strong smoothing effects or contain pixels with abrupt colors. To quantify the differences, the PSNR and SSIM values of the denoised images are shown in Table \ref{tab.compare.015}. Again, results by the proposed models have the largest values.

In the last example, we compared all algorithms on images with very large noise, SD=0.3 in Figure \ref{fig.fish.03} and SD=0.5 in Figure \ref{fig.stripes.05}. For the experiments in Figure \ref{fig.fish.03}, we set $\alpha=5\times10^{-4}, \beta=50,\eta=12$ for model (\ref{eq.model1}), $\alpha=5\times10^{-3},\beta=30, \eta=3$ for model (\ref{eq.model2}). For other models, we use  $\alpha=1\times10^{-2},\beta=5\times10^{-3}, \eta=3.5$ for CE, $\lambda=1$ for CTV and $\lambda=0.6$ for VTV. Our results are presented in Figure \ref{fig.fish.03}. For the experiments in Figure \ref{fig.stripes.05}, we set $\alpha=5\times10^{-4}, \beta=50,\eta=18$ for model (\ref{eq.model1}), $\alpha=5\times10^{-3},\beta=30, \eta=5$ for model (\ref{eq.model2}). For other models, we use  $\alpha=1\times10^{-2},\beta=5\times10^{-3}, \eta=7$ for CE, $\lambda=0.8$ for CTV and $\lambda=0.8$ for VTV. Again, the proposed models (\ref{eq.model1}) and (\ref{eq.model2}) give the best results which recover the features best.

\section{Conclusion}
\label{sec.conclusion}
We propose in this article two modified color elastica models for vector-valued image regularization. 
Compared to the original color elastica model, model (\ref{eq.model1}) multiplies the Laplace-Beltrami term by the image metric $g$, and model (\ref{eq.model2}) utilizes the relation between the surface area regularizer and total variation regularizer. 
Both models reduces to Euler's elastica model for gray-scale images. For each proposed model, we introduced an operator-splitting method to find the minimizer. 
The nonlinearity is decoupled by introducing matrix- and vector-valued variables. 
Then, finding the minimizer is converted to solving an associated initial value problem, which is time-discretized by an operator-splitting method. 
Each subproblem after splitting either has a closed-form solution or can be solved efficiently. 
The advantages of the proposed models are demonstrated by systematic numerical experiments. 
Compared to existing models, the proposed models give more natural recoveries of images with less abrupt or artificial colors, and better PSNR and SSIM values.

\section*{Acknowledgment}

The authors would like to thank the anonymous reviewers of this article for most helpful comments and suggestions. 

\appendix
\section*{Appendix}
\section{Relation between the color elastica model (\ref{eq.model.old}) and Euler's elastica model} \label{sec.colorElastica.Euler}
We discuss the color elastica model (\ref{eq.model.old}) for one-channel images and its relation with Euler's elastica model. This discussion follows from \cite[Remark 3.2]{liu2021color}. 
For a one-channel image $v$, model (\ref{eq.model.old}) becomes 
\begin{align}
	\min_{v\in \cH^2(\Omega)} &\sqrt{\alpha}\int_{\Omega}\left(1+\beta \frac{1}{\alpha+ |\nabla v|^2} \left( \nabla \cdot \frac{\nabla v}{\sqrt{\alpha+|\nabla v|^2}}\right)^2\right) \sqrt{\alpha+|\nabla v|^2}d\bx \nonumber\\
	& + \frac{1}{2\eta}\int_{\Omega} |v-f|^2d\bx.
	\label{eq.color.euler.1}
\end{align}
Divide the first integral of (\ref{eq.color.euler.1}) by $\sqrt{\alpha}$ and let $\alpha\rightarrow0$, one gets
\begin{align}
	\min_{v\in \cH^2(\Omega)} \int_{\Omega} \left(1+\beta\frac{1}{ |\nabla v|^2} \left( \nabla \cdot \frac{\nabla v}{|\nabla v|}\right)^2\right)|\nabla v|d\bx + \frac{1}{2\eta}\int_{\Omega} |v-f|^2d\bx.
	\label{eq.color.euler.2}
\end{align}
The first term in (\ref{eq.color.euler.2}) is a variant of Euler's elastica model, except the term $\left( \nabla \cdot \frac{\nabla v}{|\nabla v|}\right)^2$ is weighted by $1/|\nabla v|^2$.


\bibliographystyle{abbrv}
\bibliography{ref}
\end{document}